\documentclass[11pt]{article} 

\usepackage{amsfonts}
\usepackage{xcolor}
\usepackage{amsthm}

\usepackage{amsmath}
\usepackage{amssymb}
\usepackage{enumitem}
\usepackage{mathrsfs}
\usepackage[numbers,sort&compress]{natbib}

\usepackage{multirow}
\usepackage{mathtools}
\usepackage{graphicx}
\usepackage{url}
\usepackage{algorithm2e}
\usepackage{soul}
\PassOptionsToPackage{algo2e,ruled}{algorithm2e}
\usepackage{amsthm}
\usepackage[margin=1in]{geometry}
\usepackage{hyperref}
\usepackage[capitalize]{cleveref}
\hypersetup{colorlinks=true,linkcolor=blue,citecolor=blue}

\usepackage[mathscr]{eucal}
\usepackage{thm-restate}

\newtheorem{theorem}{Theorem}
\newtheorem{lemma}[theorem]{Lemma}
  
    \newtheorem{definition}{Definition}
  \newtheorem{proposition}[theorem]{Proposition}
  \newtheorem{remark}[theorem]{Remark}
  \newtheorem{corollary}[theorem]{Corollary}

\newcommand{\comment}[1]{}

\title{Distributionally-Constrained Adversaries in Online Learning}

\author{
    \textbf{Mo\"ise Blanchard}\footnote{Corresponding author}\\
      Columbia University\\
      \small{\texttt{mb5414@columbia.edu}}
      \and 
      \textbf{Samory Kpotufe}\\
      Columbia University\\
      \small{\texttt{skk2175@columbia.edu}}
}

\date{}

\usepackage{thmtools}
\usepackage{thm-restate}
\usepackage[mathscr]{eucal}
\usepackage{bbm}

\newcommand{\nonl}
{\renewcommand{\nl}{\let\nl\oldnl}}

\DeclareMathOperator*{\argmax}{arg\,max}
\DeclareMathOperator*{\argmin}{arg\,min}

\usepackage{latexsym,tikz,graphicx}
\usetikzlibrary{calc}
\usetikzlibrary{decorations.pathreplacing}
\usetikzlibrary{patterns}

\renewenvironment{proof}[1][]{\par\noindent{\bf Proof #1\ }}{\hfill$\blacksquare$\\[2mm]}

\begin{document}

\newcommand{\trw}{\text{\small TRW}}
\newcommand{\maxcut}{\text{\small MAXCUT}}
\newcommand{\maxcsp}{\text{\small MAXCSP}}
\newcommand{\suol}{\text{SUOL}}
\newcommand{\wuol}{\text{WUOL}}
\newcommand{\crf}{\text{CRF}}
\newcommand{\sual}{\text{SUAL}}
\newcommand{\suil}{\text{SUIL}}
\newcommand{\fs}{\text{FS}}
\newcommand{\fmv}{{\text{FMV}}}
\newcommand{\smv}{{\text{SMV}}}
\newcommand{\wsmv}{{\text{WSMV}}}
\newcommand{\trwp}{\text{\small TRW}^\prime}
\newcommand{\alg}{\text{ALG}}
\newcommand{\rhos}{\rho^\star}
\newcommand{\brhos}{\brho^\star}
\newcommand{\bzero}{{\mathbf 0}}
\newcommand{\bs}{{\mathbf s}}
\newcommand{\bw}{{\mathbf w}}
\newcommand{\bws}{\bw^\star}
\newcommand{\ws}{w^\star}
\newcommand{\Prt}{{\mathsf {Part}}}
\newcommand{\Fs}{F^\star}

\newcommand{\Hs}{{\mathsf H} }

\newcommand{\hL}{\hat{L}}
\newcommand{\hU}{\hat{U}}
\newcommand{\hu}{\hat{u}}

\newcommand{\bu}{{\mathbf u}}
\newcommand{\ubf}{{\mathbf u}}
\newcommand{\hbu}{\hat{\bu}}

\newcommand{\primal}{\textbf{Primal}}
\newcommand{\dual}{\textbf{Dual}}

\newcommand{\Ptree}{{\sf P}^{\text{tree}}}
\newcommand{\bv}{{\mathbf v}}

\newcommand{\bq}{\boldsymbol q}

\newcommand{\rvM}{\text{M}}

\newcommand{\Acal}{\mathcal{A}}
\newcommand{\Bcal}{\mathcal{B}}
\newcommand{\Ccal}{\mathcal{C}}
\newcommand{\Dcal}{\mathcal{D}}
\newcommand{\Ecal}{\mathcal{E}}
\newcommand{\Fcal}{\mathcal{F}}
\newcommand{\Gcal}{\mathcal{G}}
\newcommand{\Hcal}{\mathcal{H}}
\newcommand{\Ical}{\mathcal{I}}
\newcommand{\Kcal}{\mathcal{K}}
\newcommand{\Lcal}{\mathcal{L}}
\newcommand{\Mcal}{\mathcal{M}}
\newcommand{\Ncal}{\mathcal{N}}
\newcommand{\Pcal}{\mathcal{P}}
\newcommand{\Scal}{\mathcal{S}}
\newcommand{\Tcal}{\mathcal{T}}
\newcommand{\Ucal}{\mathcal{U}}
\newcommand{\Vcal}{\mathcal{V}}
\newcommand{\Wcal}{\mathcal{W}}
\newcommand{\Xcal}{\mathcal{X}}
\newcommand{\Ycal}{\mathcal{Y}}
\newcommand{\Ocal}{\mathcal{O}}
\newcommand{\Qcal}{\mathcal{Q}}
\newcommand{\Rcal}{\mathcal{R}}

\newcommand{\brho}{\boldsymbol{\rho}}

\newcommand{\Cbb}{\mathbb{C}}
\newcommand{\Ebb}{\mathbb{E}}
\newcommand{\Nbb}{\mathbb{N}}
\newcommand{\Pbb}{\mathbb{P}}
\newcommand{\Qbb}{\mathbb{Q}}
\newcommand{\Rbb}{\mathbb{R}}
\newcommand{\Sbb}{\mathbb{S}}
\newcommand{\Vbb}{\mathbb{V}}
\newcommand{\Wbb}{\mathbb{W}}
\newcommand{\Xbb}{\mathbb{X}}
\newcommand{\Ybb}{\mathbb{Y}}
\newcommand{\Zbb}{\mathbb{Z}}

\newcommand{\Rbbp}{\Rbb_+}

\newcommand{\bX}{{\mathbf X}}
\newcommand{\bx}{{\boldsymbol x}}

\newcommand{\btheta}{\boldsymbol{\theta}}

\newcommand{\Pb}{\mathbb{P}}

\newcommand{\hPhi}{\widehat{\Phi}}

\newcommand{\Sigmah}{\widehat{\Sigma}}
\newcommand{\thetah}{\widehat{\theta}}

\newcommand{\indep}{\perp \!\!\! \perp}
\newcommand{\notindep}{\not\!\perp\!\!\!\perp}

\newcommand{\one}{\mathbbm{1}}
\newcommand{\1}{\mathbbm{1}}
\newcommand{\aprx}{\alpha}

\newcommand{\ST}{\Tcal(\Gcal)}
\newcommand{\x}{\mathsf{x}}
\newcommand{\y}{\mathsf{y}}
\newcommand{\Ybf}{\textbf{Y}}
\newcommand{\smiddle}[1]{\;\middle#1\;}

\definecolor{dark_red}{rgb}{0.2,0,0}
\newcommand{\detail}[1]{\textcolor{dark_red}{#1}}

\newcommand{\ds}[1]{{\color{red} #1}}
\newcommand{\rc}[1]{{\color{green} #1}}

\newcommand{\mb}[1]{\ensuremath{\boldsymbol{#1}}}

\newcommand{\metric}{\rho}
\newcommand{\proj}{\text{Proj}}

\newcommand{\paren}[1]{\left( #1 \right)}
\newcommand{\sqb}[1]{\left[ #1 \right]}
\newcommand{\set}[1]{\left\{ #1 \right\}}
\newcommand{\floor}[1]{\left\lfloor #1 \right\rfloor}
\newcommand{\ceil}[1]{\left\lceil #1 \right\rceil}
\newcommand{\abs}[1]{\left|#1\right|}
\newcommand{\norm}[1]{\left\|#1\right\|}

\newcommand{\todo}[1]{{\color{red} TODO: #1}}

\newcommand{\Reg}{\textnormal{Reg}}
\newcommand{\Ncov}{\textnormal{NCov}}
\newcommand{\TV}{\textnormal{TV}}
\newcommand{\Dkl}{\textnormal{D}_{\textnormal{KL}}}
\newcommand{\Binom}{\textnormal{Binom}}
\newcommand{\Beta}{\textnormal{Beta}}
\newcommand{\Unif}{\textnormal{Unif}}
\newcommand{\Ber}{\textnormal{Ber}}

\newcommand{\Span}{\textnormal{Span}}
\newcommand{\Proj}{\textnormal{Proj}}

\newcommand{\Treedim}{\textnormal{TreeDim}}
\newcommand{\TDim}{\textnormal{TDim}}
\newcommand{\Tdim}{\textnormal{Tdim}}
\newcommand{\dimu}{\underline{\dim}}
\newcommand{\dimo}{\overline{\dim}}
\newcommand{\tildedim}{\widetilde{\dim}}

\newcommand{\OblivReg}{\textnormal{Reg}^{\textnormal{ob}}}
\newcommand{\AdaptReg}{\textnormal{Reg}^{\textnormal{ad}}}
\newcommand{\OblivR}{\Rcal^{\textnormal{ob}}}
\newcommand{\AdaptR}{\Rcal^{\textnormal{ad}}}
\newcommand{\OblivRealizableR}{\Rcal^{\textnormal{ob},\star}}
\newcommand{\AdaptRealizableR}{\Rcal^{\textnormal{ad},\star}}

\maketitle

\begin{abstract}
    There has been much recent interest in understanding the continuum from adversarial to stochastic settings in online learning, with various frameworks including \emph{smoothed settings} proposed to bridge this gap. We consider the more general and flexible framework of \emph{distributionally constrained adversaries} in which instances are drawn from distributions chosen by an adversary within some constrained distribution class \cite{rakhlin2011online}. Compared to smoothed analysis, we consider general distributional classes which allows for a fine-grained understanding of learning settings between fully stochastic and fully adversarial for which a learner can achieve non-trivial regret.
    We give a characterization for which distribution classes are learnable in this context against both oblivious and adaptive adversaries, providing insights into the types of interplay between the function class and distributional constraints on adversaries that enable learnability. In particular, our results recover and generalize learnability for known smoothed settings. Further, we show that for several natural function classes including linear classifiers, learning can be achieved without any prior knowledge of the distribution class---in other words, a learner can simultaneously compete against any constrained adversary within learnable distribution classes.
\end{abstract}


\section{Introduction}

The continuum between fully adversarial and stochastic online learning has been of much interest over the last decade, with the aim to gain insights into challenging settings where the environment or \emph{distribution} of observations changes repeatedly overtime, either mildly (e.g.\ seasonally), or adversarially in reaction to a learner’s actions. The aim of this work is to gain further insights into sufficient and necessary conditions on the learning setting—i.e., model class and sequence of distributions—that admit learning, i.e., \emph{sublinear regret} over time. 

Formally, we consider online binary classification where the learner is competing with classifiers from a fixed hypothesis class $\Fcal \subset 2^{\cal X}$, while the environment iteratively samples each query $x_t\in\Xcal$ at time $t$ from a distribution $\mu_t$ chosen from a fixed class $\cal U$. As such $\cal U$ may be viewed as controlling the adversarial nature of the learning setting: the problem is fully adversarial if $\cal U$ admits all distributions on $\cal X$, and fully stochastic if $\cal U$ contains a single distribution. The question is then to understand learning settings, i.e., pairs $(\cal F, \cal U)$, for which learning is possible. 

For background, it is well known that outside of certain classes with very low \emph{complexity}, i.e., classes with finite Littlestone dimension, learning is impossible in the fully adversarial setting \cite{littlestone1988learning,ben2009agnostic}. For instance, this precludes fully adversarial learning even for linear classifiers in dimension one (thresholds). At the other extreme, any function class with finite Vapnik–Chervonenkis (VC) dimension is learnable for the fully stochastic case \cite{vapnik1971uniform,vapnik1974theory,valiant1984theory}, which in particular includes linear classifiers in finite dimension. This substantial gap between fully adversarial and stochastic settings thereby motivated a line of work on understanding the continuum between the two settings: e.g., sublinear regrets are possible in (a) \emph{smoothed adversarial settings} where all $\mu \in {\cal U}$ have bounded density with respect to some fixed base measure $\mu_0$ \cite{block2022smoothed,haghtalab2022oracle,haghtalab2024smoothed,block2024performance,blanchard2024agnostic}, and in recent relaxations thereof \cite{block2023sample}, or (b) whenever $(\cal F, \cal U)$ admits nontrivial \emph{sequential Rademacher complexity} \cite{rakhlin2011online}. In particular, the sufficient conditions in (b) allow for more learnable settings than in (a) by actually considering the interaction between $\cal F$ and $\cal U$; however, understanding situations with nontrivial sequential Rademacher complexity is a subject of ongoing work \cite{rakhlin2015online,block2021majorizing}. 

Our aim in this work is to further elucidate sufficient and necessary conditions for learnability by considering alternative views on the interactions between $\cal F$ and $\cal U$. We consider two classical formalisms on the adversary: first, \emph{oblivious adversaries} which pick a sequence of distributions $\mu_1, \ldots, \mu_t \in {\cal U}$ for all times $t$ {ahead of time}, i.e., oblivious to the learner’s actions, and second, more challenging \emph{adaptive adversaries} which may pick $\mu_t \in {\cal U}$ at time $t$ according to the entire history of the learner’s actions up to time $t$. We derive sufficient and necessary conditions on $(\cal F, U)$ w.r.t. either oblivious or adaptive adversaries, in each case described in terms of the interaction between $\cal F$ and $\cal U$: of consequence is how much mass measures $\mu \in \cal U$ may put on certain critical regions of space $\cal X$ defined by $\cal F$---these are regions somewhat rooted in the intuitive concept of \emph{disagreement regions} $\{f \neq f’\}$. 
The ensuing characterizations of learnability immediately recover known learnable settings such as \emph{smoothed} settings and their recent relaxations, and reveal more general settings where learning is possible in both oblivious and adaptive adversarial regimes. 
We illustrate these results with various instantiations, including further relaxations on smoothed settings which we term \emph{generalized smoothed settings}, and more transparent linear and polynomial classification settings. 

Our characterizations allow for useful comparisons between learning against oblivious and adaptive adversaries.
In particular, these reveal the interesting new fact that oblivious and adversarial regimes can sometimes be equivalent for relevant function classes (see \cref{remark:useful_characterization_obliv_vs_adaptive}): e.g., for linear or polynomial classes, any $\cal U$ is learnable with oblivious adversaries iff it is learnable with adaptive adversaries.

Furthermore, on the algorithmic front, our analysis reveals that, for such linear and polynomial function classes $\cal F$, there exists so-called \emph{optimistic} learners which can successfully learn any learnable $\cal U$ without prior knowledge of $\cal U$. This stands in contrast with the classical literature on smoothed settings, where until recently (see \cite{block2024performance,blanchard2024agnostic}) any successful algorithm required some knowledge of $\cal U$. The existence of such optimistic learners in general remains open.

We note that throughout the paper, our focus is mostly on understanding when learning is possible depending on the setting $(\Fcal,\Ucal)$ rather than achieving optimal rates. We leave open whether other equivalent characterizations can lead to tighter minimax regret bounds. Nevertheless, to prove our characterization of learnable settings, we provide generic procedures with convergence rates depending on  combinatorial complexity measures for the pair $(\Fcal, \Ucal)$ inherent in the characterizations of learnability. Notably, these rates recover (up to log terms) known rates in the literature in various settings of interest, including on smoothed online learning. 

However, these procedures are in general not computationally efficient, as they mainly serve to establish learnability, but may be made efficient when critical regions are easy to obtain---this is the case for some important function classes including linear separators.

We next present a more formal outline of the results. 

\subsection{Formal Results Outline}

As stated earlier, we consider fixed classes $\cal F$ of binary classifiers as baseline; the adversary $\cal A$ denotes a sequence of joint distributions on $x_t, y_t$, where the marginal $\mu_t$ on $x_t$ is constrained to a class $\cal U$. The learner $\cal L$ at each time $t$ picks a label ${\cal L}(x_t)$ (not necessarily realized in $\cal F$) before observing $y_t$, and incurs a loss $\1[\Lcal(x_t) \neq y_t].$

\paragraph{Oblivious Adversaries.} 
We consider the oblivious regret 
\begin{equation*}
    \OblivReg_T({\cal L}, {\cal A}) = \mathbb{E} \sqb{\sum_{t\in [T]} \1[\Lcal(x_t) \neq y_t] } - \inf_{f \in \cal F} \mathbb{E}\sqb{\sum_{t\in [T]} \1[f(x_t)\neq y_t]},
\end{equation*}
w.r.t. an oblivious adversary ${\cal A}$ that chooses distributions on $x_t, y_t$ ahead of time, with the restriction that $x_t \sim \mu_t \in {\cal U}$. 
We aim for necessary and sufficient conditions on $(\Fcal,\Ucal)$ for achieving sublinear regret $\lim_{T\to\infty}\frac{1}{T}\OblivReg_T(\cdot)  =0$.
We first show that the next two conditions are necessary: (1) $\Fcal$ must admit uniform finite covering numbers with respect to distances induced by any mixture measure from $\Ucal$, and (2) for any tolerance $\epsilon>0$, a so-called $\epsilon$-dimension of $(\Fcal,\Ucal)$ must be finite. This dimension---similarly to the Littlestone dimension---measures the maximum depth of a binary tree whose nodes correspond to \emph{disagreement regions} $\{f \neq f'\}$ over $\cal F$ that have mass at least $\epsilon$ under some $\mu \in \Ucal$. Intuitively, it captures the number of rounds over which an oblivious adversary can induce $\epsilon$ average-regret.
Upper-bounds are established via general procedures whose convergence rates depend on tree-depth. 

We then show that these conditions are also sufficient for general VC classes $\cal F$ (which we note immediately satisfy condition (1) above). The case of infinite VC dimension remains open.

\paragraph{Adaptive Adversaries.} We consider the adaptive regret
\begin{equation*}
    \AdaptReg_T({\cal L}, {\cal A}) = \mathbb{E} \sqb{\sum_{t\in [T]} \1[{\cal L}(x_t)\neq y_t] - \inf_{f \in \cal F} \sum_{t\in [T]} \1[f(x_t)\neq y_t]},
\end{equation*}
w.r.t. an adaptive adversary $\Acal$ that chooses the distribution $\mu_t\in\Ucal$ on $x_t$ as well as the label $y_t$ according to the learner's actions up to time $t-1$. Again, we aim for necessary and sufficient conditions on $(\cal F, U)$ to admit $\lim_{T\to\infty}\frac{1}{T} \AdaptReg_T(\cdot) =0$. The derived conditions in this case hold for any $\cal F$ irrespective of VC dimension. However, these conditions are more complex than in the oblivious case, and in particular rely on more refined notions of critical regions: it is no longer enough to consider simple disagreement regions $\{f \neq f'\}$, but subsets thereof, defined over a hierarchy of regions that distributions $\mu\in\cal U$ may put mass on. In particular such a hierarchy highlights the essential ingredients of difficult games for the learner: more precisely, critical regions at depth $k\geq 0$ and tolerance $\epsilon>0$ correspond to subsets of disagreement regions for which an adversary can enforce $k$ more rounds of a game that can induce at least $\epsilon$ average-regret. 
We show that a setting $(\Fcal,\Ucal)$ is learnable iff for any tolerance $\epsilon>0$ critical regions of arbitrary depth do not exist. In turn, learnable settings admit general algorithms with convergence rates depending on the maximum depth of critical regions.

\paragraph{Instantiations and Implications.} We consider the following three instantiations:

\begin{itemize}[wide, labelwidth=!, labelindent=0pt]
    \item \emph{Generalized Smoothed Settings.} We show that the above characterizations of learnability naturally recover and generalize usual smoothed settings of \citep{block2022smoothed,haghtalab2022oracle,haghtalab2024smoothed,block2024performance} (and their recent extensions \cite{block2023sample}): namely, letting $B$ in some $\cal B$ denote specific classes of measurable sets, we consider general distribution classes of the form 
    $\Ucal \doteq \{\mu: \mu(B)\leq \rho(\mu_0(B)), \ \forall B \in \cal B\} $, defined w.r.t. to a base measure $\mu_0$ and fixed function $\rho(\cdot)$ with $\lim_{\epsilon\to 0}\rho(\epsilon)=0$. As we will see, it follows easily from our general characterizations of learnability that for such $\Ucal$, the setting $(\Fcal, \Ucal)$ is learnable: (1) for $\cal B$ denoting all measurable sets, we recover known results that smoothed settings are learnable for adaptive adversaries (\cref{prop:extended_smooth_learnable}); (2) for even more general $\cal U$, where $\Bcal$ is restricted to just  disagreement regions $\{f\neq g\}$ for $f,g\in\Fcal$, $(\Fcal,\Ucal)$ remains learnable for oblivious adversaries (\cref{prop:pairwise_oblivious}).

    Finally, as noted above, the complexity terms inherent in our characterizations of learnability are sufficiently tight to recover known rates for specific smoothed settings studied in previous works.
    

    \item \emph{VC-1 classes ${\cal F}$.} Such simple classes are relevant in our analysis in two respects. 
    
    First, they 
    yield immediate clarity on characterizations of learnability since disagreement regions are then less opaque: this is thanks to simple representations of VC 1 classes as initial segments of an ordered tree  
    \cite{ben20152}, which then yield interpretable characterizations (\cref{prop:VC_1_oblivious_statement,thm:characterization_VC1}). In particular, these yield concrete examples of situations where adaptive learnability is clearly harder than oblivious learnability. 

    Second, they can help characterize learnability for some more general classes such as the linear case discussed below.

    
    \item \emph{\bf Linear Classes.} Last, we consider linear classifiers due to their practical importance in machine learning: our results apply generally to any class $\cal F$ of thresholds $f = \1[g \geq 0]$ over functions $g$ that are linear over some finite-dimensional basis (e.g., polynomials of finite degree, last layer of a neural network, random Fourier Features, etc.). 

    For such classes, we show that the 
    the general characterizations simplify into interpretable geometric conditions. Namely, a distribution class $\Ucal$ is learnable for linear classifiers against either oblivious or adaptive adversaries, \emph{iff any projection of this class onto a $1$-dimensional direction---resulting in threshold classifiers---is learnable} (\cref{prop:linear_classifiers}). Furthermore, this gives an example of a general function class for which learnability against oblivious and adaptive adversaries are equivalent (although convergence rates may be different). 
\end{itemize}

Surprisingly, in all the above instantiations and for both oblivious and adaptive adversaries, we demonstrate a stronger form of learnability: there exist so-called \emph{optimistic} algorithms that achieve sublinear regret under all distribution classes $\Ucal$ for which the setting $(\Fcal,\Ucal)$ is learnable (\cref{thm:VC_1_optimistic_learning,prop:lin_classifiers_optimistic}), which is the minimal condition for learning. As a result, for some relevant function classes including linear or polynomial classes, no prior information on $\Ucal$ is necessary to successfully learn. While this may not hold in general against adaptive adversaries, it could be true in general for oblivious adversaries.

\subsection{Related Works}

\paragraph{Sequential Learning.} Online learning is perhaps one of the most common formulations of sequential statistical learning, in which a learner is iteratively tested on an instance, makes a prediction, then observes the true value to update its model \cite{cesa2006prediction}. For fully adversarial data, learnable function classes are exactly those with finite Littlestone dimension \cite{littlestone1988learning,ben2009agnostic}. This marks a significant gap compared to the classical statistical learning setting when data is i.i.d.; in that case, learning is possible whenever the function class has finite VC dimension \cite{vapnik1971uniform,vapnik1974theory,valiant1984theory}. Beyond the fully adversarial setting, there has been a wide interest in understanding learnability including in the distributionally-constrained setting considered in this paper \cite{rakhlin2011online}. To this end, \cite{rakhlin2010online,rakhlin2011online} introduced sequential counterparts to complexity measures in classical statistical learning, and in particular, the sequential Rademacher complexity which can be used to upper bound the minimax adaptive regret \cite{rakhlin2015online,rakhlin2015sequential,block2021majorizing}. In comparison, our work gives characterizations of learnability in this distributionally-constrained setting for both adaptive and oblivious regret, and provides new insights in terms of critical regions that the adversary aims to place mass on.

\paragraph{Smoothed Online Learning.} As a specific form of sequential learning, several recent works have focused on the smoothed online learning setting in which instances are sampled from smooth distributions with respect to some fixed base measure \cite{rakhlin2011online}, and derived statistical rates as well as oracle-efficient algorithms  \cite{ haghtalab2022oracle,block2022smoothed,block2023sample,haghtalab2024smoothed,block2024performance,blanchard2024agnostic}. Beyond standard online learning, the smoothed learning framework was adopted in a wide variety of settings including sequential probability assignments \cite{bhatt2024smoothed}, learning in auctions \cite{durvasula2023smoothed,cesa2023repeated}, robotics \cite{block2022efficient,block2023oracle}, and differential privacy \cite{haghtalab2020smoothed}.
To illustrate the utility of our general characterizations, we show how they recover and extend known learnability results and regret bounds in smoothed online learning---which corresponds to a specific choice of distributional constraints on the adversary.

\paragraph{Contextual Bandits and Reinforcement Learning.} Beyond the binary classification setting considered in this paper, contextual bandits or reinforcement learning are natural settings in which distributions of instances (of states in this context) are constrained (e.g., through transition maps). Accordingly, sequential complexities or smoothed classes have also been considered in these context \cite{foster2020beyond, xie2022role}. Formulations of learnability in terms of critical regions---similar to those introduced in this work---may extend to these more general learning models, a direction we leave for future work.

\paragraph{Optimistic Learning.} We follow the optimistic learning framework introduced by \cite{hanneke2021learning} when studying learning without prior knowledge on the distribution class. At the high-level, the main question in this line of work is whether learning can be achieved with a unifying algorithm for all settings that are learnable---in our setting this corresponds to having a single algorithm that achieves sublinear regret under any learnable distribution class. This minimal assumption that learning is possible is known as the ``optimist's assumption''. Optimistic learning was originally introduced to understand which processes generating the data are learnable in online learning \cite{hanneke2021learning,hanneke2021bayes,blanchard2021reduction,blanchard2022universal,hanneke2024theory}, general regression \cite{blanchard2023universal} and contextual bandits \cite{blanchard2023contextual,blanchard2023adversarial}.

\section{Formal Setup and Preliminaries}

\subsection{Classical Complexity Notions for Function Classes}

Before defining our formal setup, we recall some standard combinatorial dimensions that will be useful in discussions. 

In the stochastic setup (PAC learning), where instances are sampled i.i.d.,\ it is known that the VC dimension characterizes learnability \cite{vapnik1971uniform,vapnik1974theory,valiant1984theory}.

\begin{definition}[VC dimension]\label{def:VC_dimension}
    Let $\Fcal:\Xcal\to\{0, 1\}$ be a function class. We say that $\Fcal$ shatters a set of points $\{x_1,\ldots,x_d\}\subset\Xcal$ if for any sequence of labels $y\in\{0, 1\}^d$, there exists $f_y\in\Fcal$ such that $f_y(x_t)=y_t$ for all $t\in[d]$. 
    
    The {\bf VC dimension} of $\Fcal$ is the size of the largest shattered set.
\end{definition}

In the full adversarial setting, i.e., $\cal U=\{\text{all distributions }\mu\text{ on }\Xcal\}$, both oblivious and adaptive learnability are characterized by finite Littlestone dimension \cite{littlestone1988learning,ben2009agnostic}.

\begin{definition}[Littlestone dimension]\label{def:littlestone_dimension}
    Let $\Fcal:\Xcal\to\{0, 1\}$ be a function class. We say that $\Fcal$ \emph{shatters} a tree of depth $d$ if there exist mappings $x_1,\ldots,x_d$ with $x_t:\{0,1\}^{t-1}\to \Xcal$ for $t\in[d]$, such that for any sequence of labels $y\in\{0,1\}^d$, there exists a function $f_y\in\Fcal$ satisfying $f_y(x_t(y_1,\ldots,y_{t-1}))=y_t$ for all $t\in[d]$.

    The {\bf Littlestone dimension} of $\Fcal$ is the maximum depth of a shattered tree.
\end{definition}

To give an example, the function class with at most $k$ non-zero values, $\Fcal=\{\1_S, |S|\leq k\}$ has Littlestone dimension $k$.
Function classes with finite Littlestone dimension are, however, somewhat contrived. For instance, even the class of thresholds on $[0,1]$, $\Fcal=\{\1[\cdot\leq x], x\in[0,1]\}$, has infinite Littlestone dimension.

\subsection{Distributionally-Constrained Online Learning}
\label{subsec:formal_setup}

\paragraph{General Online Setup.}
We focus on the classical full-feedback online binary classification problem with 0-1 loss. Precisely, let $\Xcal$ be an instance space equipped with a sigma-algebra (we denote by $\Sigma$ the set of measurable sets) and let $\Fcal$ be a function class, defined as a set of measurable functions $f:\mathcal X\to\set{0,1}$. A distribution class $\Ucal$ is a set of distributions on $\Xcal$. Given a time horizon $T\geq 1$, we consider the following sequential problem in which the adversarial test distributions are constrained within $\Ucal$. At each iteration $t\in[T]$:
\begin{enumerate}
    \item An adversary selects a test distribution $\mu_t\in\Ucal$ and a measurable value function $f_t:\mathcal X\to \set{0,1}$, depending on the full learning procedure history $\Hcal_t = \{(x_s,y_s,\hat y_s)_{s<t}\}$.
    \item The environment samples a test instance $x_t\sim\mu_t$.
    \item The learner observes the test instance $x_t$ and makes a prediction $\hat y_t \in\set{0,1}$, using $x_t$ and the history $\Hcal_t$ only.
    \item The learner observes the realized value $y_t = f_t(x_t)$ and incurs a loss $\1[\hat y_t\neq y_t]$. \label{item:feedback_learner}
\end{enumerate}
The goal of the learner is to minimize its expected total loss. Precisely, we consider two setups: oblivious and fully adaptive adversaries.

\paragraph{Oblivious Adversaries.} In the oblivious setup, the test distributions $\mu_1,\ldots,\mu_T$ as well as value functions $f_1,\ldots,f_T$ chosen by the adversary are fixed a priori of the learning procedure. In particular, they are independent from the test samples $x_1,\ldots,x_T$. Accordingly, we define the oblivious regret of the learner as being the expected excess loss of the learner with respect to the best \emph{a priori} function in class:
\begin{equation*}
    \OblivReg_T(\Lcal,\Acal):= \Ebb_{\Lcal,\Acal} \sqb{\sum_{t=1}^T \1[\hat y_t\neq y_t]}  - \inf_{f\in\Fcal} \Ebb_{\Lcal,\Acal} \sqb{\sum_{t=1}^T \1[f(x_t)\neq y_t]},
\end{equation*}
where we denote the learner and adversary by $\Lcal$ and $\Acal$ respectively. The expectation is taken over the randomness of the learner, the adversary, and the environment samples.

\paragraph{Adaptive Adversaries.} In the adaptive setup, the decisions of the adversary may depend on the complete history, as described above. We then aim to minimize the adaptive regret of the learner which corresponds to the expected excess loss of the learner with respect to the best \emph{a posteriori} function in class:
\begin{equation*}
    \AdaptReg_T(\Lcal,\Acal) := \Ebb_{\Lcal,\Acal}\sqb{\sum_{t=1}^T \1[\hat y_t\neq y_t] - \inf_{f\in\Fcal} \sum_{t=1}^T \1[f(x_t)\neq y_t]}.
\end{equation*}

Formally, learners and adversaries are defined as follows.

\begin{definition}[Learners and adversaries]
\label{def:learner_policies}
    Fix $T\geq 1$. A {\bf learner} is a sequence $(\Lcal_t)_{t\in[T]}$ of measurable functions $\Lcal_t:\Xcal^{t-1}\times \set{0,1}^{t-1}\times \Xcal\mapsto \set{0,1}$. The learner prediction at $t$ is $\hat y_t = \Lcal_t((x_s)_{s<t},(y_s)_{s<t},x_t)$.

    An {\bf oblivious adversary} is a sequence of distributions $\mu_1,\ldots,\mu_T\in\Ucal$ and a sequence of measurable value functions $f_1,\ldots,f_T:\Xcal\to\{0,1\}$.

    An {\bf adaptive adversary} is a sequence $(\Mcal_t,F_t)_{t\in[T]}$ of measurable mappings $M_t: \Xcal^{t-1}\times \set{0,1}^{t-1} \to \Ucal$ (with respect to the weak topology of distributions on $\Xcal$) and $F_t: \Xcal^{t-1}\times \set{0,1}^{t-1} \times\Xcal\to\{0,1\}$. The distribution and value function chosen by the adversary at iteration $t$ are $\mu_t=M_t((x_s)_{s<t},(\hat y_s)_{s<t})$ and $f_t=F_t((x_s)_{s<t},(\hat y_s)_{s<t} ,\cdot)$.
\end{definition}

As discussed in the introduction, this distributionally-constrained adversary setup has received attention in the literature \cite{rakhlin2011online} because it naturally generalizes classical frameworks in the statistical learning literature.

\paragraph{Realizable and Agnostic Settings.} In this paper we will consider two different scenarios for the distributionally-constrained model. First, we will consider the \emph{realizable} or \emph{noiseless} case in which a realizable adversary is constrained to output labels that are consistent with some function in class. That is, in the oblivious setting we assume that there exists $f\in\Fcal$ such that $f(x_t)=y_t$ for $t\in[T]$ almost surely, and for the adaptive setting, we assume that $\inf_{f\in\Fcal}\sum_{t=1}^T \1[f(x_t)\neq y_t]=0$ almost surely.

Second, we also consider the general case that has been discussed so far---also known as the \emph{agnostic} case---in which the values chosen by the adversary $y_t\in\{0,1\}$ are arbitrary.

\paragraph{Learnability of Function Class / Distribution Class Pairs.}

A natural question is to understand for which distribution classes $\Ucal$ can the learner ensure sublinear regret. As discussed in the introduction, this depends on the choice of function class $\Fcal$, which gives rise to the following notion of learnability for a pair $(\Fcal,\Ucal)$.

\begin{definition}\label{def:learnable_pairs}
    Fix a function class $\Fcal$ and a distribution class $\Ucal$ on $\Xcal$. We say that $(\Fcal,\Ucal)$ is {\bf learnable} for adaptive adversaries (in the agnostic setting) if
    \begin{equation*}
        \frac{1}{T}\AdaptR_T(\Fcal,\Ucal):= \inf_{\Lcal} \sup_{\Acal} \frac{1}{T}\AdaptReg_T(\Lcal,\Acal)  \underset{T\to\infty}{\longrightarrow} 0,
    \end{equation*}
    where the infimum is taken over all learners, and the supremum is taken over all adaptive adversaries $\Acal$ for the distribution class $\Ucal$. Learnability in the realizable setting is defined similarly by taking the supremum over adaptive and realizable adversaries; we denote by $\AdaptRealizableR_T$ the corresponding minimax regret.

    We define the notion of learnability and minimax regret $\OblivR_T,\OblivRealizableR_T$ for oblivious adversaries similarly, replacing $\AdaptReg_T$ with the oblivious regret $\OblivReg_T$. 
\end{definition}

As an important remark, we consider learnability with respect to a \emph{any fixed} pair $(\cal F, U)$ of function class and distribution class. Hence, a characterization of learnability should capture interactions between both classes, which allows for fine-grained results on the strength of adversaries.

Note that within \cref{def:learnable_pairs}, the algorithm $\Lcal$ is allowed to know the horizon $T$ in advance. One could alternatively consider a definition in which the learner strategy is independent of the horizon. These can be shown to be equivalent via the standard doubling trick. 
For intuition, give a brief translation of known results in the online learning literature in terms of learnable pairs $(\Fcal,\Ucal)$.
In the fully adversarial setting $\Ucal=\set{\text{all distributions}}$, a learner can achieve sublinear regret if and only if only if $\Fcal$ has finite Littlestone dimension \cite{littlestone1988learning,ben2009agnostic}. On the other hand, one can achieve sublinear regret if the distribution class is a singleton $\Ucal=\{\mu_0\}$ whenever $\Fcal$ has finite VC dimension \cite{vapnik1971uniform,vapnik1974theory,valiant1984theory}. The latter result was later generalized to smooth distribution classes $\Ucal_{\mu_0,\sigma}$ in smoothed online learning literature \cite{block2022smoothed,haghtalab2024smoothed}.

\paragraph{Optimistic Procedures: Learning with Unknown $\cal U$.}
\label{subsec:unknown_distr_class}
By default, the learnability notion from \cref{def:learnable_pairs} assumes that the learner knows the distribution class $\Ucal$ beforehand (and the function class $\Fcal$). 
This is a standard assumption within the smoothed online literature, in which most works assume that the base measure $\mu_0$ is known to the learner. However, this differs from other learning frameworks, including PAC learning where the data distribution $\mu_0$ is typically unknown. It turns out that for some useful function classes, no prior knowledge of the distribution class $\Ucal$ is needed. To formalize this, we recall that for a given function class $\Fcal$, in general not all distribution classes are learnable. The best one can hope for is therefore to design so-called \emph{optimistic learners},\footnote{This term refers to the idea that such algorithms learn under the minimal ``optimistic'' assumption that the setting at hand is learnable.} that successfully learn all distribution classes $\Ucal$ for which $(\Fcal,\Ucal)$ is learnable. We borrow this terminology which we borrow from works on universal learning \cite{hanneke2021learning}.

\begin{definition}\label{def:optimistic_learner}
    Fix a function class $\Fcal$ on $\Xcal$.
    We say that a learner $\Lcal$ is an {\bf optimistic learner} if for any distribution class $\Ucal$ such that $(\Fcal,\Ucal)$ is learnable, $\sup_{\Acal_T} \frac{1}{T}\AdaptReg_T(\Lcal,\Acal_T) \underset{T\to\infty}{\longrightarrow} 0$,
    where the supremum is taken over corresponding adversaries for horizon $T$ and distribution class $\Ucal$. We can similarly define optimistic learners for all other considered settings (adaptive/oblivious, agnostic/realizable).

    We say that the function class $\Fcal$ is {\bf optimistically learnable} if it admits an optimistic learner.
\end{definition}

As a note, learnable distribution classes $\Ucal$ may have very different structure, making optimistic learnability a very strong form of learnability. 
Surprisingly, we will show that for some useful function classes, we can indeed construct optimistic learners (see \cref{sec:examples}). In fact, we leave open whether optimistic learning is always achievable.

\section{Learnability with Oblivious Adversaries}
\label{sec:oblivious_adversaries}

In this section, we focus on learning with oblivious adversaries and first introduce an analogue notion of the Littlestone dimension for distributionally-constrained online learning. For convenience, we write $\bar\Ucal$ for the set of mixture distributions from $\Ucal$. We first define the notion of \emph{interaction tree} for a pair of function class and distribution class $(\Fcal,\Ucal)$. 

\begin{definition}
    Fix a pair of function and distribution classes $(\Fcal,\Ucal)$.
    An \textbf{interaction tree} of depth $d$ for $(\Fcal,\Ucal)$ is a full binary tree $\Tcal_d$ of depth $d$ such that each non-leaf node $v\in\Tcal_d$ is labeled with a distribution $\mu_v\in\bar\Ucal$ and each edge $e$ of the tree is labeled by a function $f_e\in\Fcal$. By convention, edges are directed from parents to their children.
\end{definition}

\cref{fig:oblivious_dimension} gives an illustration of an interaction tree of depth $2$.
We next introduce the following notion of dimension for $(\Fcal,\Ucal)$ which measures the largest depth of an interaction tree in which the distribution $\mu_v$ for each non-leaf node $v$ puts $\epsilon>0$ mass on the disagreement region $\{f_{(v,c_1)}\neq f_{(v,c_2)}\}$ where $c_1,c_2$ are the two children of $v$.

\begin{definition}
\label{def:dimension_oblivious}
    Fix $\epsilon>0$. We say that a pair of function and distribution classes $(\Fcal,\Ucal)$ {\bf $\epsilon$-shatters a tree of depth $d$} if there exist an interaction tree of depth $d$ that satisfies the following.
    \begin{enumerate}
        \item For any non-leaf node $v\in\Tcal_d$, let $f^l,f^r\in\Fcal$ be the functions associated with edges incident to $v$ leading to its left and  right child respectively. Then, $\mu_v(f^l\neq f^r)\geq \epsilon$.
        \item For any edge $e:=(v_1,v_2)$, consider any \emph{ancestor edge} $e':=(v_1',v_2')$, i.e., $v_1'$ and $v_2'$ are ancestors of $v_1$. Then, $\mu_{v_1'}(f_e\neq f_{e'})\leq \frac{\epsilon}{3}$.
    \end{enumerate}

    The {\bf $\epsilon$-dimension} of $(\Fcal,\Ucal)$, denoted by $\dim(\epsilon;\Fcal,\Ucal)$, is defined as the maximum depth of a tree $\epsilon$-shattered by $(\Fcal,\Ucal)$.
\end{definition}

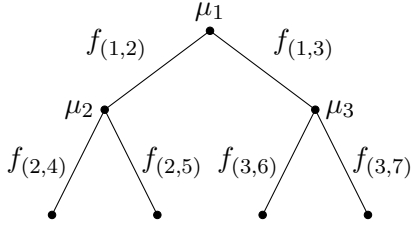
\begin{figure}
    \centering
    \begin{tikzpicture}[scale=0.7]

\filldraw (0,-0.5) circle (2pt);
\node[above] at (0,-0.5) {$\mu_1$};

\filldraw (-2,-2) circle (2pt);
\node[left] at (-2,-2) {$\mu_2$};
\draw (0,-0.5) -- (-2,-2) node[midway, above left] {$f_{(1,2)}$};

\filldraw (2,-2) circle (2pt);
\node[right] at (2,-2) {$\mu_3$};
\draw (0,-0.5) -- (2,-2) node[midway, above right] {$f_{(1,3)}$};

\filldraw (-3,-4) circle (2pt);
\draw (-2,-2) -- (-3,-4) node[midway, left] {$f_{(2,4)}$};

\filldraw (-1,-4) circle (2pt);
\draw (-2,-2) -- (-1,-4) node[midway, right] {$f_{(2,5)}$};

\filldraw (1,-4) circle (2pt);
\draw (2,-2) -- (1,-4) node[midway, left] {$f_{(3,6)}$};

\filldraw (3,-4) circle (2pt);
\draw (2,-2) -- (3,-4) node[midway, right] {$f_{(3,7)}$};

\end{tikzpicture}
    \caption{Illustration of an interaction tree of depth $d=2$}
    \label{fig:oblivious_dimension}
\end{figure}

We give an illustration of labeled trees from \cref{def:dimension_oblivious} in \cref{fig:oblivious_dimension}, which have a similar interpretation to Littlestone trees.
We can easily check that in the fully adversarial case when $\Ucal=\{\text{all distributions }\mu\text{ on }\Xcal\}$, the $\epsilon$-dimension of $(\Fcal,\Ucal)$ coincides with the Littlestone dimension (we can use Dirac distributions for all inner nodes). As we will see, having finite $\epsilon$-dimension is necessary for learning with oblivious adversaries.

Additionally, we show that having uniformly bounded covering numbers among mixtures of distributions in $\bar \Ucal$ is necessary. Formally, each distribution on $\Xcal$ induces a pseudo-metric on $\Fcal$ in terms of the mass it places on disagreement regions $\{f\neq g\}$. We then define covering numbers associated to a distribution on $\Xcal$ as follows.

\begin{definition}[Covering numbers]\label{def:covering_numbers}
    Let $\mu$ be a distribution on $\Xcal$. An $\epsilon$-covering set of $\Fcal$ with respect to $\mu$ is a subset $S\subseteq \Fcal$ such that for all $f\in\Fcal$, $=\inf_{f_0\in S} \mu(f\neq f_0) \leq \epsilon$.
    
    We denote by $\Ncov_\mu(\epsilon;\Fcal)$ the {\bf $\epsilon$-covering number} of $\Fcal$ with respect to $\mu$, defined as the minimum cardinality of an $\epsilon$-covering set for $\mu$.
\end{definition}

\paragraph{Necessary Conditions for General $\cal F$.} With these definitions at hand, we summarize below the necessary conditions for learning with oblivious adversaries for general function classes.

\begin{proposition}[Necessary Conditions for Learnability]
\label{prop:necessary_condition}
     Let $\Fcal,\Ucal$ be a function class and distribution class on $\Xcal$. Suppose that $(\Fcal,\Ucal)$ is learnable for oblivious adversaries (in either agnostic or realizable settings). Then, (1) for any $\epsilon>0$, $\sup_{\mu\in\bar\Ucal} \Ncov_\mu(\epsilon;\Fcal)<\infty$ and (2) for any $\epsilon>0$, $\dim(\epsilon;\Fcal,\Ucal)<\infty$.
\end{proposition}

This is proved in \cref{subsec:necessary_conditions_oblivious}.
In general, the conditions from \cref{prop:necessary_condition} may not be sufficient anymore; the following result gives such an example for a function class $\Fcal$ with infinite dimension. The proof is given in \cref{subsec:counterexample_VC_class}.

\begin{proposition}[Insufficiency beyond VC Classes]\label{prop:counter_example_infinite_dim}
    There exists a function class and distribution class $\Fcal,\Ucal$ on $\Xcal=\bigcup_{d\geq 1} S_{d-1}$ where $S_{d-1}=\{x\in\Rbb^d:\|x\|_2=1\}$, such that for all $\epsilon>0$, $\sup_{\mu\in\bar \Ucal} \Ncov_\mu(\epsilon;\Fcal)<\infty$ and $\dim(\epsilon;\Fcal,\Ucal)<\infty$, but $(\Fcal,\Ucal)$ is not learnable for oblivious adversaries (in either agnostic or realizable settings).
\end{proposition}

\paragraph{Characterization for $\Fcal$ with Finite VC Dimension.}
However, when $\Fcal$ has finite VC dimension, we show that the conditions from \cref{prop:necessary_condition} are also sufficient for learning with oblivious adversaries. Hence these conditions provide a complete characterization of learnable pairs $(\Fcal,\Ucal)$.
As a remark, the first condition from \cref{prop:necessary_condition} is automatically satisfied for function classes with finite VC dimension since any distribution $\mu$ on $\Xcal$ has $\Ncov_\mu(\epsilon;\Fcal) = \Ocal(\text{VC}(\Fcal)\log\frac{1}{\epsilon})$ where $\text{VC}(\Fcal)$ denotes the VC dimension of $\Fcal$ \cite{haussler1995sphere}.

The characterization of learnable settings $(\Fcal,\Ucal)$ for oblivious adversaries when $\Fcal$ has finite VC dimension is detailed below. The proof is given in \cref{subsec:oblivious_realizable,subsec:oblivious_adaptive}.

\begin{theorem}\label{thm:oblivious_adversaries}
    Let $\Fcal$ be a function class with finite VC dimension and let $\Ucal$ be a distribution class on $\Xcal$. The following are equivalent:
    \begin{itemize}
        \item $(\Fcal,\Ucal)$ is learnable for oblivious adversaries (in the agnostic setting),
        \item $(\Fcal,\Ucal)$ is learnable for oblivious adversaries in the realizable setting,
        \item For any $\epsilon>0$, $\dim(\epsilon;\Fcal,\Ucal)<\infty$.
    \end{itemize}
\end{theorem}

We give some brief intuitions as to why the necessary conditions from \cref{prop:necessary_condition} are also sufficient for finite VC dimension classes $\Fcal$. Such function classes $\Fcal$ enjoy the uniform convergence property for any choice of distribution $\mu\in\bar\Ucal$, which can be used for algorithms akin to empirical risk minimization. In contrast, the first condition from \cref{prop:necessary_condition} only ensures uniform convergence for $\epsilon$-covers for $\mu$---this is not always sufficient to recover the target function.

\begin{remark}
    We also provide non-asymptotic minimax oblivious regret bounds in terms of $\dim(\epsilon;\Fcal,\Ucal)$ and the VC dimension of $\Fcal$. These are deferred to the appendix (see \cref{thm:combined_oblivious_non_asymptotic}) for simplicity.
\end{remark}

\section{Learnability with Adaptive Adversaries}
\label{sec:adaptive_adversaries}

As in the oblivious case, we can think about the learning problem in terms of critical regions of $\Xcal$ on which the distribution class places sufficient mass. However such critical regions are now more involved. 

For intuition let's consider a simple realizable setting, i.e., where the adversary is constrained to produced an $f\in \cal F$ with no error on the observed sequence.
In this case, the learner can restrict its search to the current \emph{version space} $\Fcal_t$ containing all functions in the class $\Fcal$ that are consistent with the observed dataset $D_t:= \{(x_s,y_s), s<t\}$. The adversary then aims to place sufficient mass---i.e., pick $\mu_t$ with sufficient mass, say $\epsilon>0$---on the disagreement region $B_t$ induced by $\Fcal_t$ so as to encourage a mistake by the learner in the next round (that is, $B_t\subseteq \{x\in\Xcal: \exists f,g\in\Fcal_t, f(x)\neq g(x)\}$). 
If the new instance $x_t$ falls in this region, however, the next-round version space $\Fcal_{t+1}$ is accordingly reduced and so is the corresponding disagreement region $B_{t+1}$, thus making it harder for the adversary to continue inducing errors on the part of the learner. It is therefore incumbent on the adversary to instead aim to put mass on careful subsets of such disagreement regions $B_t$ \emph{so as to ensure that it can keep inducing mistakes in subsequent rounds.} This leads to the following recursive definition of such critical regions where an adversary ought to put mass. Interestingly, while defined with the above \emph{realizable} intuition, the same critical regions turn out to also characterize the agnostic setting where there is no restriction on the adversary to maintain $0$ error.

We denote the collection of all possible finite datasets by $\Dcal:=\set{ D\subset \Xcal\times\{0,1\}, |D|<\infty} $. Such datasets $D$ will serve to define version spaces, and hence corresponding critical regions $B_k(D,\epsilon)$ on which the adversary aims to put mass $\epsilon$, iterating on the horizon $k\geq 0$ as follows.

\begin{definition} Fix any $\epsilon> 0$.
 First, let $\Lcal_0(\epsilon) = \set{D\in \Dcal: \exists f\in\Fcal, \forall (x,y)\in D,f(x)=y }$ contain all (realizable) datasets whose corresponding version space is non-empty. For $k\geq 1$ and a dataset $D\in\Dcal$, we define the {\bf critical region}
\begin{equation*}
    B_k(D;\epsilon) := \set{x\in\Xcal: D\cup \{(x,0)\}, D\cup\{(x,1)\}\in\Lcal_{k-1}(\epsilon)},
\end{equation*}
recursively in terms of the sets 
\begin{equation*}
    \Lcal_k(\epsilon) := \set{D \in {\cal D}: \sup_{\mu\in\Ucal} \mu(B_k(D;\epsilon)) \geq \epsilon}.
\end{equation*}
\end{definition}

In simple terms, the critical region $B_k(D;\epsilon)$ is the target region of instances $x$ for which after adding any labeling of $x$ to the dataset, the adversary can still ensure $k-1$ future learner mistakes.
Note that by construction, the sets $(\Lcal_k(\epsilon))_{k\geq 0}$ are non-increasing (see \cref{lemma:decreasing_level_sets} in \cref{sec:proofs_adaptive}).
We can characterize learnable distribution classes against both realizable and agnostic adversaries in terms of the sets $\Lcal_k(\epsilon)$ as follows.

\begin{theorem}\label{thm:qualitative_charact}
    Let $\Fcal,\Ucal$ be a function class and distribution class on $\Xcal$. The following are equivalent:
    \begin{itemize}
        \item $(\Fcal,\Ucal)$ is learnable (in the agnostic setting),
        \item $(\Fcal,\Ucal)$ is noiseless-learnable,
        \item  $\forall \epsilon>0,\exists k: \Lcal_k(\epsilon) = \emptyset$.
    \end{itemize}
\end{theorem}

Next, we give more quantitative results on the regret convergence. We introduce the notation
\begin{equation*}
    k(\epsilon) := \sup\{k\geq 0:  \Lcal_k(\epsilon) \neq \emptyset \},\quad \epsilon>0.
\end{equation*}
When needed, we may write this quantity as $k(\epsilon;\Fcal,\Ucal)$ to specify the considered pair $(\Fcal,\Ucal)$.
Intuitively, $k(\epsilon)$ is the first index for which the adversary cannot enforce $k(\epsilon)$ learner mistakes with probability $\epsilon$. Formally, we have the following bounds on the minimax regret for adaptive adversaries. Both \cref{thm:quantitative_charact,thm:qualitative_charact} are proved in \cref{sec:proofs_adaptive}.

\begin{theorem}\label{thm:quantitative_charact}
    There exist constants $C_1,C_2>0$ such that the following holds. Let $\Fcal,\Ucal$ be a function class and distribution class on $\Xcal$. Then, for $T\geq 1$, respectively for the agnostic and realizable case,
    \begin{align*}
        C_1 \sup_{\epsilon\in(0,1]} \set{\min(k(\epsilon),\epsilon T) } &\leq \AdaptR_T(\Fcal,\Ucal) \leq C_2 \inf_{\epsilon\in(0,1]} \set{ \epsilon T +\sqrt{ \min(k(\epsilon)+1,T)\cdot  T\log T} } ,\\
        C_1 \sup_{\epsilon\in(0,1]} \set{\epsilon\cdot \min(\log k(\epsilon), T)} &\leq \AdaptRealizableR_T(\Fcal,\Ucal) \leq C_2 \inf_{\epsilon\in(0,1]} \set{\epsilon T + \min(k(\epsilon),T)}.
    \end{align*}
\end{theorem}

While these minimax regret bounds are not tight in general, it turns out that for all previously considered distribution classes in the smoothed online literature, the regret upper bounds from \cref{thm:quantitative_charact} are tight up to logarithmic factors in $T$. We refer to \cref{subsec:smooth_classes} for precise comparisons and statements.

We now give some details about the learners which achieve the upper bounds in \cref{thm:quantitative_charact}. We focus on the realizable case, for which the procedures are simpler. For convenience, we refer to the level of a dataset $D\in\Dcal$ as the largest index $k$ for which $D$ belongs to the corresponding set: $D\in\Lcal_k(\epsilon)$.
To achieve approximately $\epsilon$ average error, it suffices for the learner to choose at $t$ the label $y$ for which $D_t(y):=\{(x_s,y_s),s<t\}\cup\{(x_t,y)\}$ has the largest level. This ensures that in most cases, if the learner made a mistake at time $t$, the level of the dataset $D_{t+1}:=\set{(x_s,y_s), s\leq t}$ is strictly reduced compared to the level of $D_t:=\set{(x_s,y_s),s<t}$.

These procedures are similar to the algorithms in the fully adversarial setting for function classes with finite Littlestone dimension \cite{littlestone1988learning,ben2009agnostic}. In fact, we can check that in this case $\Ucal=\{\text{all distributions on }\Xcal\}$, the procedure described above exactly corresponds to the realizable algorithm from \cite{littlestone1988learning} and that the level of a dataset $D$ corresponds to the Littlestone dimension of the corresponding version space.

\section{Instantiations of Learnability Conditions}
\label{sec:examples}

\paragraph{Overview.} In this section, we detail some consequences and instantiations of our general characterizations of learnability. As a sanity check, we first confirm below that our general characterizations recover classical results for fully adversarial and stochastic settings. We then show that our characterizations also recover results from the smoothed online learning literature in \cref{subsec:smooth_classes}. Last, we instantiate our results for function classes with VC dimension 1 in \cref{subsec:VC1_classes} and for linear classifiers in \cref{subsec:linear_classifiers}.

\paragraph{Fully Adversarial Setting.}
We first consider the simple case when $\Fcal$ is finite or has finite Littlestone dimension. When $\Ucal=\{\text{all distributions $\mu$ on $\Xcal$}\}$, we can check that the sets $\Lcal_k(\epsilon)$ are invariant with $\epsilon\in(0,1]$ and exactly corresponds to the set of datasets $D\in\Dcal$ whose corresponding version space has Littlestone dimension at least $k$. In particular, for any $\epsilon\in(0,1]$, $k(\epsilon;\Fcal,\Ucal)$ is exactly the Littlestone dimension of $\Fcal$ (which is bounded by $\log_2|\Fcal|$ is $\Fcal$ is finite). Then, \cref{thm:qualitative_charact} recovers the known result that when $\Fcal$ has finite Littlestone dimension, all distribution classes are learnable for adaptive (a fortiori oblivious) adversaries in both realizable and agnostic settings \cite{littlestone1988learning,ben2009agnostic}. Further, from the above discussion, by taking $\epsilon=1/T$, the minimax regret upper bounds from \cref{thm:quantitative_charact} exactly recover the (tight) convergence rates for adversarial online learning with finite Littlestone function classes from the same papers.

\paragraph{Stochastic Setting.}
Next, we consider the standard stochastic setting where $\Ucal=\{\mu_0\}$ for some distribution $\mu_0$ on $\Xcal$ and $\Fcal$ has finite VC dimension. Using Sauer-Shelah's lemma \cite{sauer1972density,shelah1972combinatorial}, we can show the following bound on the quantities $k(\epsilon;\Fcal,\{\mu_0\})$, used in the characterization of learnability in \cref{thm:qualitative_charact,thm:quantitative_charact}.

\begin{lemma}\label{lemma:bound_k_stochastic_setting}
    Let $\Fcal$ be a function class on $\Xcal$ with finite VC dimension $d\geq 1$. Then for any $\epsilon\in(0,1]$,
    \begin{equation*}
        k(\epsilon;\Fcal,\{\mu_0\}) \lesssim d\log \frac{2d}{\epsilon}.
    \end{equation*}
\end{lemma}

This is proved in \cref{sec:proof_consequences}.
As a result, \cref{thm:qualitative_charact} recovers that all stochastic settings $(\Fcal,\{\mu_0\})$ are learnable for adaptive adversaries when $\Fcal$ has finite VC dimension (for both realizable and agnostic settings). Further, taking $\epsilon=1/T$, the minimax regret upper bounds from \cref{thm:quantitative_charact} exactly recover known (tight) convergence rates in this stochastic setting up to $\log T$ factors \cite{vapnik1971uniform,vapnik1974theory,valiant1984theory}.  

\subsection{Generalized Smoothed Settings}
\label{subsec:smooth_classes}

As a consequence of our general characterizations, we can 
recover known results on the learnability of smoothed distribution classes \cite{block2022smoothed,haghtalab2022oracle,haghtalab2024smoothed}. 
We recall that in smoothed settings, one focuses on distribution classes of the form $\Ucal_{\mu_0,\sigma}=\{\mu: \mu \ll \mu_0, \ \|\frac{d\mu}{d\mu_0}\|_\infty \leq \frac{1}{\sigma}\}$ for some base measure $\mu_0$ on $\Xcal$ and a smoothness parameter $\sigma> 0$. Subsequently, \cite{block2023sample} considered relaxations of this smoothed distribution classes in which $\|\frac{d\mu}{d\mu_0}\|_\infty$ is replaced by an $f$-divergence $\text{div}_f(\mu\parallel \mu_0)$. Conveniently, all these distribution classes can be either written as or included within \emph{generalized smoothed} distribution classes as defined below.

\begin{definition}
    Fix any base measure $\mu_0$ on $\Xcal$ and a function $\rho:[0,1]\to \Rbb_+$ satisfying $\lim_{\epsilon\to 0}\rho(\epsilon)=0$, and $\Bcal \subseteq \Sigma$ a set of measurable subsets of $\Xcal$. We define the corresponding \textbf{generalized smoothed distribution class} via
    \begin{equation}\label{eq:def_extended_smoothed_distribution_class}
        \Ucal_{\mu_0,\rho}(\Bcal):=\set{\mu: \forall B\in\Bcal,\, \mu(B)\leq \rho(\mu_0(B))}.
    \end{equation}
\end{definition}

\begin{remark}[Previous smoothed settings are captured]
    Smoothed settings $\Ucal_{\mu_0,\sigma}$ correspond to taking functions $\rho:\epsilon\mapsto \epsilon/\sigma$ and all measurable sets $\Bcal=\Sigma$. We can also check that the divergence-based distribution classes considered in \cite{block2021majorizing} can be generalized by these generalized smoothed classes by taking all measurable sets $\Bcal=\Sigma$ and choosing $\rho$ appropriately (see \cref{lemma:generalized_smoothness} in \cref{subsec:extension_smoothed_proofs} for a quick proof).
\end{remark}

 Note that the generalized smoothed distribution classes are richer as the set of constraints $\Bcal$ is smaller. In particular, one would aim to restrict the set of constraints $\Bcal$ depending on the considered function class $\Fcal$. Such a natural choice is to focus on disagreement regions: $\Bcal^{\text{pair}}(\Fcal):=\set{\{f\neq g\}: f,g\in\Fcal}$. For convenience, we will write $\Ucal^{\text{pair}}_{\mu_0,\rho}(\Fcal):=\Ucal_{\mu_0,\rho}(\Bcal^{\text{pair}}(\Fcal))$. We start with the case where $\Bcal=\Sigma$ then turn to disagreement regions.

\paragraph{First Generalization: $\Bcal=\Sigma$.}
In previous works on smoothed settings, one of the major underlying steps is to reduce the problem of learning under smoothed adversaries to the stochastic case where the instances are sampled according to the base measure. This is the intuition behind the coupling lemma from \cite[Theorem 2.1]{haghtalab2024smoothed}, and the rejection sampling argument from \cite{block2023sample}.
As a direct consequence of \cref{thm:qualitative_charact}, we can explicit this reduction to the stochastic setting in a unified and simplified manner as detailed below.\footnote{The reduction can be immediately extended to show that if a distribution class $\Ucal_0$ is adaptive learnable then so are any smoothened versions of this distribution class, e.g., of the form $\{\mu: \exists \mu_0\in\Ucal, \,\forall B\in\Sigma, \mu(B)\leq \rho(\mu_0(B))\}$.}

\begin{proposition}\label{prop:extended_smooth_learnable}
    Let $\Fcal$ be a function class on $\Xcal$. For any distribution $\mu_0$ on $\Xcal$ and any non-decreasing $\rho:[0,1]\to \Rbb_+$ with $\lim_{\epsilon\to 0}\rho(\epsilon)=0$, $(\Fcal,\Ucal_{\mu_0,\rho}(\Sigma))$ is learnable for adaptive adversaries whenever $(\Fcal,\{\mu_0\})$ is adaptive learnable. 
    
    Further, for any $\epsilon\in(0,1]$, letting $\rho^{-1}(\epsilon):=\sup\{\delta\in(0,1]: \rho(\delta)<\epsilon\}$, we have $k(\epsilon;\Fcal,\Ucal_{\mu_0,\rho}(\Sigma)) \leq k(\rho^{-1}(\epsilon);\Fcal,\{\mu_0\})$.
\end{proposition}
    
\begin{proof}
    Fix $\epsilon>0$. By definition of $\Ucal_{\mu_0,\rho}(\Sigma)$, if for some measurable set $B\in\Sigma$ one has $\sup_{\mu\in\Ucal}\mu(B) \geq \epsilon$, then $ \rho(\mu_0(B))\geq \epsilon$ and hence $\mu_0(B)\geq \rho^{-1}(\epsilon)$. Then, for all $k\geq 0$, $\Lcal_k(\epsilon)\subseteq \widetilde\Lcal_k(\rho^{-1}(\epsilon))$ where $\widetilde\Lcal$ refers to the corresponding sets for the distribution class $\{\mu_0\}$. In particular, $k(\epsilon;\Fcal,\Ucal_{\mu_0,\rho}(\Sigma)) \leq k(\rho^{-1}(\epsilon);\Fcal,\{\mu_0\})$. Therefore, \cref{thm:qualitative_charact} implies that if the distribution class $\{\mu_0\}$---which corresponds to a stochastic setting---is learnable for $\Fcal$, then $\Ucal$ is also learnable.
\end{proof}

\comment{
\begin{corollary}\label{cor:extended_smooth_learnable}
    Let $\Fcal$ be a function class and $\Ucal_0$ a distribution class on $\Xcal$. Fix a any non-decreasing function $\rho:[0,1]\to \Rbb_+$ with $\lim_{\epsilon\to 0}\rho(\epsilon)=0$. We consider the smoothened distribution class
    \begin{equation*}
        \Ucal:=\set{\mu \text{ on }\Xcal: \exists \mu_0\in\Ucal_0: \forall B\in\sigma(\Fcal), \mu(B)\leq \rho(\mu_0(B))}.
    \end{equation*}
    
    Then, $(\Fcal,\Ucal)$ is learnable for adaptive (in particular oblivious) adversaries whenever $(\Fcal,\Ucal_0)$ is learnable for adaptive adversaries.
\end{corollary}
}

Importantly, this recovers the fact that for function classes $\Fcal$ with finite VC dimension, generalized-smoothed distribution classes are learnable for adaptive (a fortiori oblivious) adversaries in both agnostic and realizable settings as detailed in the following.

\begin{corollary}[Learnability of generalized smoothed settings]
    If $\Fcal$ has finite VC dimension, since all stochastic settings $(\Fcal,\{\mu_0\})$ are learnable, \cref{prop:extended_smooth_learnable} shows that all generalized-smoothed distribution classes are learnable. I particular $(\Fcal,\{\mu_0\})$
\end{corollary}

Further, \cref{prop:extended_smooth_learnable} together with \cref{lemma:bound_k_stochastic_setting} imply that the generic upper bounds on the minimax adaptive regret from \cref{thm:quantitative_charact} recover all known regret bounds in smoothed settings up to logarithmic factors. Precisely, in the smoothed setting, the tolerance function is $\rho:\epsilon\mapsto \epsilon/\sigma$ and hence $\rho^{-1}(\epsilon) = \sigma\epsilon$ for $\epsilon\in(0,1]$. The bounds from \cref{thm:quantitative_charact} exactly give $\AdaptR_T(\Fcal,\Ucal_{\mu_0,\sigma}) \lesssim \sqrt{dT\log T\cdot \log\frac{T}{\sigma} }$ and $\AdaptRealizableR_T(\Fcal,\Ucal_{\mu_0,\sigma}) \lesssim d \log\frac{T}{\sigma}$, which are tight up to the $\log T$ factors \cite{block2022smoothed,haghtalab2024smoothed}. We can also check similarly that \cref{thm:quantitative_charact} recovers the bounds from \cite{block2023sample} for smoothed distributions classes with $f$-divergences up to $\log T$ factors. For conciseness, we defer the formal statement and proof to \cref{cor:rate_smooth_divergence} in \cref{subsec:extension_smoothed_proofs}.

\paragraph{Further Generalization: $\Bcal =\Bcal^{\text{pair}}(\Fcal)= \{ \{f\neq g\}: f, g \in \cal F\}$.}
In the previous paragraph we focused on the case where we include all measurable set constraints $\Bcal=\Sigma$ for generalized smoothed distribution classes (see \cref{eq:def_extended_smoothed_distribution_class}), which does not incorporate any potential dependencies between the distribution class and function class $\Fcal$.
In this paragraph, we turn to the case of pairwise-smoothed classes $\Ucal_{\mu_0,\rho}^{\text{pair}}(\Fcal)$ for which we restrict the set of constraints to disagreement regions for $\Fcal$. Naturally, these may be significantly richer distribution classes than for $\Bcal=\Sigma$.

Using \cref{thm:oblivious_adversaries}, we can check that these remain learnable for oblivious adversaries.

\begin{proposition}\label{prop:pairwise_oblivious}
    Let $\Fcal$ be a function class with finite VC dimension on $\Xcal$. For any distribution $\mu_0$ on $\Xcal$ and any non-decreasing $\rho:[0,1]\to \Rbb_+$ with $\lim_{\epsilon\to 0}\rho(\epsilon)=0$, $(\Fcal,\Ucal_{\mu_0,\rho}^{\text{pair}}(\Fcal))$ is learnable for oblivious adversaries.
\end{proposition}

The proof is given in \cref{subsec:extension_smoothed_proofs}.
Pairwise-smoothed distribution classes are not learnable in full generality for adaptive adversaries (even for finite VC dimension), however for simple function classes this may be the case. As an example, for the function class of thresholds on $[0,1]$: $\Fcal_{\text{threshold}}:=\{\1[\cdot \leq x],x\in[0,1]\}$, learnable distribution classes essentially correspond to pairwise-smoothed classes (and oblivious and adaptive adversaries coincide). The proof of the following result is given in \cref{subsec:linear_classifiers_proof}.

\begin{proposition}\label{prop:thresholds}
    A distribution class $\Ucal$ is learnable for $\Fcal_{\text{threshold}}$ (against either adaptive or oblivious adversaries) if and only if there exists some distribution $\mu_0$ on $\Xcal$, and a tolerance function $\rho:[0,1]\to\Rbb_+$ with $\lim_{\epsilon\to 0}\rho(\epsilon)=0$, such that
    $\Ucal \subset \Ucal_{\mu_0,\rho}^{pair}(\Fcal_{\text{threshold}})$.
\end{proposition}

\subsection{Function Classes of VC Dimension 1}
\label{subsec:VC1_classes}

We consider function classes $\Fcal$ with VC dimension 1, whose simplicity allow for a clear interpretation of the characterizations of \cref{sec:oblivious_adversaries,sec:adaptive_adversaries}: namely, such classes admit a simple tree representation \cite{ben20152} in terms of which critical regions for learnability are more concrete. 

Also, of importance as we will later see, VC 1 classes can yield intuition into the learnability of more complex classes such as linear and polynomial classes (see \cref{subsec:linear_classifiers}).  

We recall the following result of \cite{ben20152}.

\begin{definition}[Tree ordering \cite{ben20152}]
    A partial ordering $\preceq$ on $\Xcal$ is a {\bf tree ordering} if for any $x\in\Xcal$, the initial segment $\set{x'\in\Xcal: x'\preceq x}$ is totally ordered by $\preceq$. Additionally, we write $x\prec y$ if and only if $x\preceq y$ and $x\neq y$.
\end{definition}

One of the results from \cite{ben20152} is that functions in a VC dimension 1 function class are initial segments for some tree ordering up to a relabeling.

\begin{theorem}[Tree representation of VC 1 classes \cite{ben20152}]
\label{thm:representation_VC1}
    Let $\Fcal$ be a function class on $\Xcal$. $\Fcal$ has VC dimension 1 if and only if there exists a tree ordering $\preceq$ on $\Xcal$ and a relabeling function $r:\Xcal\to\{0,1\}$ such that every function $f\in \Fcal_r := \{x\mapsto \1[g(x) \neq r(x)], g\in\Fcal\}$ is an initial segment of $\preceq$. That is, if $x\preceq y$ and $f(y)=1$ then $f(x)=1$.
\end{theorem}

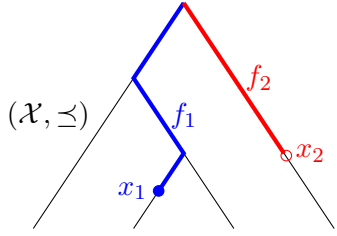
\begin{figure}
    \centering
    \begin{tikzpicture}[scale=1]

\draw (2,-3) -- (0,0) -- (-2,-3);
\draw (-2/3,-1) -- (0,-2);
\draw (-2/3,-3) -- (0,-2) -- (2/3,-3);

\node at (-1.8,-1.5) {$(\Xcal,\preceq)$};

\filldraw[blue] (-1/3,-2.5) circle (2pt);
\draw[blue,ultra thick] (0,0) -- (-2/3,-1) -- node[right] {$f_1$} (0,-2) -- (-1/3,-2.5) node[left] {$x_1$};

\draw[red] (4/3+0.03,-2-0.03) circle (2pt);
\draw[red,ultra thick] (0,0) -- node[right] {$f_2$} (4/3,-2) node[right]{$x_2$};

\end{tikzpicture}
    \caption{Tree representation of a VC 1 function class. For $x,x'\in\Xcal$ we have $x\preceq x'$ if $x'$ is a descendant of $x'$ in the tree (or $x$ itself). Functions have value $1$ on an initial segment of $(\Xcal,\preceq)$ and $0$ elsewhere. We represent two functions $f_1:=\1[\cdot\preceq x_1]$ and $f_2:=\1[\cdot \prec x_2]$.}
    \label{fig:VC_dimension_1}
\end{figure}

We give an illustration of this tree representation in \cref{fig:VC_dimension_1}. As a remark, if $\Xcal$ is complete, this implies that any relabeled function $f\in\Fcal_r$ it either holds that $f(x)= \1[x\preceq x_f]$ or that $f(x)= \1[x\prec x_f]$ for some unique $x_f\in\Xcal$. 
Given \cref{thm:representation_VC1}, we fix the corresponding tree ordering $\preceq$ and assume without loss of generality that all functions $f\in\Fcal$ are initial segments of $\preceq$. Up to deleting some elements of $\Xcal$ we also suppose that for all $x\in\Xcal$, there exists $f,g\in\Fcal$ with $f(x)\neq g(x)$, and for any $x_1,x_2\in\Xcal$ with $x_1\neq x_2$, there exists $f\in\Fcal$ such that $f(x_1)\neq f(x_2)$. We say that $\Fcal$ has been pruned if it satisfies these three conditions, and we identify $\Fcal$ with its pruned function class (which does not affect learning).

\paragraph{Oblivious Adversaries.} We can use this tree representation to easily interpret the characterization from \cref{thm:oblivious_adversaries} for learnable distribution classes. Informally, for VC-dimension-1 classes, this condition essentially asks that one cannot embed a large threshold class within $(\Fcal,\Ucal)$. To be more precise, we introduce an analogue of the \emph{threshold dimension} \cite{shelah1990classification,hodges1997shorter} for distributionally-constrained online learning.

\begin{definition}\label{eq:simplified_dim_VC_1}
    Let $\preceq$ be a tree ordering on $\Xcal$. For any $\epsilon>0$, denote by $\Tdim(\epsilon;\preceq,\Ucal)$ the maximum integer $k\geq 0$ such that there exists $x_1\prec x_2\prec\ldots\prec x_k\in\Xcal$ satisfying $\sup_{\mu\in\Ucal} \mu(\{x\in\Xcal: x\preceq x_1\})\geq \epsilon$ and $\sup_{\mu\in\Ucal} \mu(\{x\in\Xcal: x_{l-1}\prec x\preceq x_l\})\geq \epsilon$ for $l\in\{2,\ldots,k\}$.
\end{definition}

For any such linear sequence $x_1\prec \ldots \prec x_k\in\Xcal$, the functions $f_l=\1[\cdot\preceq x_l]$ for $l\in [k]$ form a threshold function class on $\set{x:x\preceq x_k}$. Translating \cref{thm:oblivious_adversaries} then gives the following; the proof is given in \cref{subsec:vc_1_classes_instantiation_proofs}.

\begin{proposition}\label{prop:VC_1_oblivious_statement}
    Let $\Fcal$ be a function class with VC dimension 1 with tree ordering $\preceq$ and $\Ucal$ a distribution class on $\Xcal$. Then, $(\Fcal,\Ucal)$ is learnable for oblivious adversaries (in either realizable or agnostic setting) if and only if for any $\epsilon>0$, $\Tdim(\epsilon;\preceq,\Ucal)<\infty$.
\end{proposition}

\paragraph{Adaptive Adversaries.} We next turn to the case of adaptive adversaries. Again, using the tree representation, we can simplify the characterization form \cref{thm:qualitative_charact} for function classes with VC dimension 1. More precisely, \cite{ben20152} show that using the tree representation, any realizable dataset $D:=(x_s,y_s)_{s\in[t]}$---such that $y_s=f(x_s)$ for $s\in[t]$ for some $f\in\Fcal$---can be compressed to a single instance $x(D):=\max_{\preceq}\{x_s: y_s=1,i\in[t]\}$, where we assume this set is non-empty for simplicity of exposition. Indeed, any function $f\in\Fcal$ consistent with the samples $(x_s,y_s)_{s\in[t]}$ must satisfy $f(x)=1$ for all $x\preceq x(D)$. 
Importantly, we can use this compression strategy to compress each dataset in the definition of $\Lcal_k(\epsilon)$ to a single element in $\Xcal$ for each $k\geq 0$, simplifying the definition of these sets. This gives the following recursive construction, where the critical regions are defined implicitly.

For $\epsilon>0$, let $\widetilde\Lcal_0(\epsilon) =\Xcal$ and for $k\geq 1$, we define
\begin{equation*}
    \widetilde \Lcal_k(\epsilon) := \set{x\in\Xcal: \sup_{\mu\in\Ucal}\mu\paren{\set{x'\in\Xcal: x\prec x'} \cap \widetilde \Lcal_{k-1}(\epsilon)}\geq \epsilon}.
\end{equation*}

This compression significantly simplifies the recursion: $\widetilde\Lcal_k(\epsilon)\subset \Xcal$ while $\Lcal_k(\epsilon)\subset 2^{\Xcal\times\{0,1\}}$. Further, the construction of the sets $\widetilde\Lcal_k(\epsilon)$ corresponds to the following behavior for the adversary. The adversary sequentially constructs a path on the tree space $(\Xcal,\preceq)$: at each iteration $k$, if possible, they place $\epsilon$ mass on descendants of the current element $x_k\in\Xcal$ using a distribution from $\Ucal$, then the next descendant $x_{k+1}$ of $x_k$ is sampled according to this distribution.
The critical region $B_k(D;\epsilon)$ associated with a dataset $D$ now becomes $B_k(D;\epsilon):=\set{x'\in\Xcal:x(D)\prec x'}\cap\widetilde\Lcal_{k-1}$ and intuitively corresponds to the set of descendants of $x(D)$ for which an adaptive adversary can construct such a path for $k-1$ next iterations.

Following the same notations as in the general case, we further denote
\begin{equation*}
    \tilde k(\epsilon):=\sup\{k\geq 0: \widetilde\Lcal_k(\epsilon)\neq\emptyset\} ,\quad \epsilon>0.
\end{equation*}
We next translate \cref{thm:qualitative_charact,thm:quantitative_charact} to this VC dimension 1 case; the proof is given in \cref{subsec:vc_1_classes_instantiation_proofs}.

\begin{theorem}\label{thm:characterization_VC1}
    Let $\Fcal$ be a function class with VC dimension 1, and let $\Ucal$ be a distribution class on $\Xcal$. Then, $(\Fcal,\Ucal)$ is learnable for adaptive adversaries if and only if for any $\epsilon>0$, there exists $k$ such that $\widetilde\Lcal_k(\epsilon)=\emptyset$.

    Further, for any $\epsilon>0$, we have $\tilde k(\epsilon) \in \{k(\epsilon), k(\epsilon)-1\}$ for all $\epsilon>0$. In particular, the bounds on $\AdaptR_T(\Fcal,\Ucal), \AdaptRealizableR_T(\Fcal,\Ucal)$ from \cref{thm:quantitative_charact} hold by replacing $k(\epsilon)$ with $\tilde k(\epsilon)$.
\end{theorem}

\paragraph{Optimistic Learning.} As introduced in \cref{subsec:unknown_distr_class}, a natural question is whether knowing the precise distribution class is necessary for the learner. It turns out that function classes with VC dimension 1 are in fact always optimistically learnable, that is, no prior knowledge on the distribution class is needed (other than it being learnable for the function class). This is stated in the following proposition, which we prove in \cref{subsec:vc_1_classes_instantiation_proofs}.

\begin{proposition}\label{thm:VC_1_optimistic_learning}
    Let $\Fcal$ be a function class of VC dimension 1. Then $\Fcal$ is optimistically learnable for both oblivious and adaptive adversaries, with the same algorithm (although learnable distributions for oblivious or adaptive adversaries may differ).
\end{proposition}

For ease of presentation, we describe this algorithm specified to the realizable case, in which case the learning rule is very simple and natural. At each iteration, the learning rule keeps in memory the maximum instance labeled as $1$: at iteration $t$, we define $x_{\max}(t):=\max_{\preceq}\{x_s:s<t,\,y_s=1\}$ where by convention, we pose $\max_{\preceq}\emptyset = x_\emptyset$ where $x_\emptyset$ is a null-value such that $x_\emptyset\prec x$ for all $x\in\Xcal$. The algorithm then simply uses the function $\1[\cdot\preceq x_{\max}(t)]$ to predict the value at time $t$. We can check that this algorithm is an optimistic learner for both oblivious and adaptive adversaries in the realizable case (see \cref{cor:simpler_alg_vc_1_realizable}). The corresponding algorithm is summarized in \cref{alg:vc_1_realizable}.

\begin{algorithm}[t]

    \caption{Optimistical algorithm for learning function classes of VC dimension 1 in the realizable case}\label{alg:vc_1_realizable}
    
    \LinesNumbered
    \everypar={\nl}
    
    \hrule height\algoheightrule\kern3pt\relax
    \KwIn{Tree representation $(\Fcal,\preceq)$ of the function class $\Fcal$}
    
    \vspace{2mm}

    $x_{\max}\gets x_\emptyset$

    \For{$t\geq 1$}{
        Observe $x_t$, predict $\hat y_t=\1[x_t\preceq x_{\max}]$, then receive $y_t$ as feedback

        \lIf{$y_t=1$}{
            $x_{\max}\gets \max_{\preceq}(x_{\max},x_t)$
        }
    }

    \hrule height\algoheightrule\kern3pt\relax
\end{algorithm}

\subsection{Linear Classifiers}
\label{subsec:linear_classifiers}

In this section, we focus on linear classifiers $\Fcal_{\text{lin}} := \set{f: f(x) = \text{sign}(a^\top x +b), a,b\in\Rbb^d}$ on $\Xcal=\Rbb^d$ and provide a simplified characterization of learnable distribution classes. The importance of this class is that many practical and general classification models can be reduced to linear classification, including polynomial or spline classifiers, classification with random Fourier features, or classification over the final layer of a neural network. Therefore, the results we present in this section directly extend to all these function classes (although the mapping from the considered model to linear classification may yield more complex characterizations).

For the linear classifier function class, we can show that the characterization of learnable distribution classes for oblivious and adaptive adversaries coincide. Further, the general condition for learnability can be simplified to essentially requiring that 1-dimensional projections of $\Fcal_{\text{lin}}$---thresholds---are learnable (see \cref{prop:thresholds}). The proof of the following result is given in \cref{subsec:linear_classifiers_proof}.

\begin{proposition}\label{prop:linear_classifiers}
    Let $d\geq 1$. Learnable distribution classes against oblivious or adaptive adversaries coincide for learning $\Fcal_{\text{lin}}$. Further, a distribution class $\Ucal$ is learnable if and only if there exists a distribution $\mu_0$ on $\Rbb^d$ and a tolerance function $\rho:[0,1]\to\Rbb_+$ with $\lim_{\epsilon\to 0}\rho(\epsilon)=0$ such that
    \begin{equation*}
        \Ucal\subseteq \set{ \mu: \mu(B_a(b_1,b_2)) \leq \rho(\mu_0(B_a(b_1,b_2))) ,\,a\in S_d,\,b_1<b_2\in\Rbb},
    \end{equation*}
    where $B_a(b_1,b_2):= \set{x: a^\top x\in[b_1,b_2]}$ and $S_d:=\{x\in\Rbb^d:\|x\|=1\}$.
\end{proposition}

\begin{remark}\label{remark:useful_characterization_obliv_vs_adaptive}
    
    The proof that learnable distribution classes are equivalent for both oblivious and adaptive adversaries relies on the general characterizations: namely, we show that the necessary condition for oblivious adversaries in \cref{thm:oblivious_adversaries} is also sufficient for learning with adaptive adversaries.
\end{remark}

Further, since the structure of learnable distribution classes is relatively simple for linear classifiers $\Fcal_{\text{lin}}$, we can show that there is a single optimistic algorithm that ensures sublinear regret for any learnable distribution class. 

\begin{proposition}\label{prop:lin_classifiers_optimistic}
    Let $d\geq 1$. The function class $\Fcal_{\text{lin}}$ of linear classifiers on $\Rbb^d$ is optimistically learnable in any of the settings considered with the same optimistic algorithm.
\end{proposition}

The proof is given in \cref{subsec:linear_classifiers_proof}.
The corresponding learning rule is a very natural and simple nearest-neighbor based algorithm. We give below the instantiation of this rule in the realizable case, for which its description is simpler. At any iteration $t\geq 1$, the current version space is defined via $\Fcal_t:=\set{f\in\Fcal_{\text{lin}}:\forall s<t, f(x_s)=y_s}$. We also define the regions that are labeled $0$ or $1$ with certainty, given available information as follows:
    \begin{equation*}
        S_t(y):=\set{ x\in \Xcal: \forall f\in \Fcal_t, f(x)=y},\quad y\in\{0,1\}.
    \end{equation*}
    Note that because we focus on linear classifiers, both $S_t(0)$ and $S_t(1)$ are convex sets are increasing with $t$. We use a simple nearest-neighbor strategy for the algorithm, with the following prediction function for time $t\geq 1$:
    \begin{equation*}
        f_t:x\in\Rbb^d \mapsto \argmin_{y\in\{0,1\}} d(x,S_t(y)).
    \end{equation*}
    Here, for any convex set $C\subseteq \Rbb^d$, we denoted by $d(x,C):=\min_{x'\in C}\|x-x'\|_2$ the distance to $C$. The choice of the Euclidean distance is rather arbitrary and other norms can work as well. The prediction at time $t$ is then $\hat y_t:=f_t(x_t)$. We can then show that this is an optimistic algorithm for both adaptive and oblivious adversaries in the realizable case (see \cref{remark:simpler_linear_algo_realizable}). We summarize the corresponding algorithm in \cref{alg:linear_realizable}.

        \begin{algorithm}[t]

    \caption{Optimistic algorithm for learning linear classifiers in the realizable case}\label{alg:linear_realizable}
    
    \LinesNumbered
    \everypar={\nl}
    
    \hrule height\algoheightrule\kern3pt\relax

    \For{$t\geq 1$}{
        Let $\Fcal_t:=\set{f\in\Fcal_{\text{lin}}:\forall s<t, f(x_s)=y_s}$ and define $S_t(y):=\set{ x\in \Xcal: \forall f\in \Fcal_t, f(x)=y}$ for $y\in\{0,1\}$

        Predict $\hat y_t=\argmin_{y\in\{0,1\}} d(x_t,S_t(y))$
    }

    \hrule height\algoheightrule\kern3pt\relax
    \end{algorithm}

\section{Conclusion}
\label{sec:conclusion}

In this work we provide characterizations of learnability for distributionally-constrained adversaries in online learning, which finely interpolates between stochastic and fully adversarial models. These characterizations are described in terms of the interaction between the function class $\Fcal$ and the distribution class $\Ucal$ from which the adversary selects distribution from, and more precisely, how much mass measures $\mu\in\Ucal$ can place on certain critical regions. In particular, these general characterizations easily recover known results in smoothed settings and ensuing relaxations.

Additionally, our characterizations reveal that for relevant function classes including linear separators, learnability for oblivious or adversarial regimes is in fact equivalent. On the algorithmic front, our results also show that, surprisingly, for many function classes no prior knowledge on the distribution class $\Ucal$ is needed for successful learning: such \emph{optimistic} learners achieve sublinear regret under any $\Ucal$ for which $(\Fcal,\Ucal)$ is learnable.

This work naturally raises a few interesting open questions. First, while our focus is mostly on characterizations of learnability, our results also provide general bounds on the minimax regret for distributionally-constrained settings $(\Fcal,\Ucal)$. It turns out that the upper bound convergence rates are tight (up to logarithmic factors) for the specific case of smoothed distribution classes, but this may not be true in general. This work leaves open the question of whether sharper rates can be achieved using other complexity measures for each setting $(\Fcal,\Ucal)$.

Second, while the characterization of learnability for oblivious adversaries is complete for finite VC dimension classes, we only provide necessary conditions for learning function classes with infinite dimension. This leaves open the following question.

\vspace{2mm}

\noindent\textbf{Open question 1:} \textit{What are necessary and sufficient conditions for $(\Fcal,\Ucal)$ to be obliviously learnable when $\Fcal$ has infinite VC dimension? }

\vspace{2mm}

Second, our results show that for some useful function classes, optimistic learning is possible, i.e., there are successful learning algorithms which do not require any prior knowledge on $\Ucal$. In particular this significantly relaxes the classical assumption in smoothed online learning that the base measure is known to the learner. A natural question is whether such optimistic learners exist more generally.

\vspace{2mm}

\noindent\textbf{Open question 2:} \textit{Is optimistic learning always achievable for oblivious adversaries and/or adaptive adversaries? If not, for which function classes $\Fcal$ is optimistic learning achievable?}

\section*{Acknowledgements} MB was supported by a Columbia Data Science Institute postdoctoral fellowship.

\bibliographystyle{alpha}
\bibliography{refs}

\appendix

\section{Proofs for Oblivious Adversaries}
\label{sec:proof_oblivious}

\subsection{Preliminaries}

Before proving the main regret bounds, we first introduce a few notations and further notions of dimension for $(\Fcal,\Ucal)$ (essentially equivalent to \cref{def:dimension_oblivious}) that will be useful within the proofs. Throughout, we will use the notation
\begin{equation*}
    N_\epsilon:=\sup_{\mu\in\bar\Ucal} \Ncov_{\mu}(\epsilon;\Fcal)
\end{equation*}
for the uniform covering numbers for $(\Fcal,\Ucal)$. Next, to formalize the binary trees defined in \cref{def:dimension_oblivious}, we introduce the following representation from now on. A labeled tree corresponds to distributions $\mu_v\in\bar\Ucal$ and functions $f_{v,0},f_{v,1}\in\Fcal$ (corresponding to left and right children edges from $v$) for any node $v\in\Tcal_d:=\bigcup_{s=0}^{d-1} \set{0,1}^{s}$. In particular, ancestors of a node $v\in\Tcal_d$ correspond to prefixes of $v$. We also refer to the depth of a node $v\in\Tcal_d$ as the integer $s\in[d]$ for which $v\in\{0,1\}^{s-1}$.

First, we introduce a strengthening of the dimension $\dim(\epsilon;\Fcal,\Ucal)$ introduced in \cref{def:dimension_oblivious}. Recall that the second condition for $\epsilon$-shattered trees allows for up to $\epsilon/3$ distance between functions and their descendants: $\mu_{w}(f_{v,0}\neq  g_w) ,  \mu_w(f_{v,1}\neq g_w) \leq \epsilon/3$ for any ancestor $w$ of $v$. We define the alternative dimension in which this distance between descendants is made arbitrarily small.

\begin{definition}
\label{def:dimension_oblivious_small_tolerance}
    Fix $\epsilon>0$ and $\eta\in(0,\epsilon/3]$. We say that a pair of function and distribution classes $(\Fcal,\Ucal)$ $(\epsilon,\delta)$-shatters a tree of depth $d$ if there exist distributions $\mu_v\in \bar\Ucal$ and functions $f_{v,0},f_{v,1}\in\Fcal$ for any inner nodes $v\in \Tcal_d$ of a full binary tree of depth $d$ satisfying the following: for any node $v\in\Tcal_d$, writing $v=(v_j)_{j\in[i]}$ for $i\in\{0,\ldots,d-1\}$,
    \begin{enumerate}
        \item $\mu_v(f_{v,0}\neq f_{v,1}) \geq \epsilon$,
        \item and for any ancestor $w=(v_j)_{j\in[i']}$ with $i'<i$, with $g_{w}=f_{w,v_{i'+1}}$,
        \begin{equation*}
            \mu_{w}(f_{v,0}\neq g_w) , \; \mu_{w}(f_{v,1}\neq g_w) \leq \eta.
        \end{equation*}
    \end{enumerate}
    We denote by $\dimu(\epsilon,\delta;\Fcal,\Ucal)$ the maximum depth of tree $(\epsilon,\delta)$-shattered by $(\Fcal,\Ucal)$.
\end{definition}

Second, we introduce a relaxation of the dimension $\dim(\epsilon;\Fcal,\Ucal)$ from \cref{def:dimension_oblivious}.
In $\epsilon$-shattered trees, for each node $v$, all right-descendant functions of $v$ must be close to $f_{v,0}$ with respect to $\mu$ and symmetrically, all left-descendant functions of $v$ must be close to $f_{v,1}$. We relax this condition by enforcing only the condition for left-descendants.

\begin{definition}
\label{def:relaxed_dimension}
    Fix $\epsilon>0$. We say that a pair of function and distribution classes $(\Fcal,\Ucal)$ has a relaxed $\epsilon$-shattered tree of depth $d$ if there exist $\mu_v\in\bar\Ucal$ and $f_{v,0}\in\Fcal$ for any node $v\in\Tcal_d$, as well as functions $f_{v,1}\in\Fcal$ for the last layer nodes $v\in\set{0,1}^{d-1}$ satisfying the following. For any node $v\in\Tcal_d$,
    \begin{enumerate}
        \item for any right-descendant function $f$ of $v$, $\mu_v(f_{v,0}\neq f) \geq 2\epsilon/3$
        \item and left-descendant function $f$ of $v$, $\mu_v(f_{v,0}\neq f) \leq \epsilon/3$.
    \end{enumerate}
    We denote by $\dimo(\epsilon;\Fcal,\Ucal)$ the maximum depth of a relaxed $\epsilon$-shattered tree by $(\Fcal,\Ucal)$.
\end{definition}

We introduce one last relaxation---although not necessary for the main results in this section, this can be useful for specific choices of function and distribution classes---which allows to focus on restricted regions of the space. For any measurable set $B\in\Sigma$, distribution $\mu$ on $\Xcal$, and measurable functions $f,g$, we introduce the notation $\mu(f\neq g;B):=\mu(\{x\in B:f(x)\neq g(x)\})$.

\begin{definition}\label{def:subregion_oblivious}
    Fix $\epsilon>0$. We denote by $\tildedim(\epsilon;\Fcal,\Ucal)$ the maximum depth $d$ of tree such that there exist distributions $\mu_v\in \bar\Ucal$, regions $B_v\in\Sigma$ and functions $f_{v,0},f_{v,1}\in\Fcal$ for any inner nodes $v\in \Tcal_d:=\bigcup_{s=0}^{d-1} \set{0,1}^{s}$ of a full binary tree of depth $d$ satisfying the following: for any node $v\in\Tcal_d$, writing $v=(v_j)_{j\in[i]}$ for $i\in\{0,\ldots,d-1\}$,
    \begin{enumerate}
        \item $\mu_v(f_{v,0}\neq f_{v,1};B_v) \geq \epsilon$,
        \item and for any ancestor $w=(v_j)_{j\in[i']}$ with $i'<i$, with $g_{w}=f_{w,v_{i'+1}}$,
        \begin{equation*}
            \mu_{w}(f_{v,0}\neq g_w;B_w) , \; \mu_{w}(f_{v,1}\neq g_w;B_w) \leq \epsilon/3.
        \end{equation*}
    \end{enumerate}
\end{definition}

These notions of dimension are all essentially equivalent: provided that $N_\epsilon<\infty$ for all $\epsilon>0$, the second condition from \cref{prop:necessary_condition} is equivalent by replacing the dimension notion from \cref{def:dimension_oblivious} with those introduced in this section, as stated below.

\begin{proposition}\label{prop:equivalent_oblivious}
    Let $\Fcal,\Ucal$ be a function class and distribution class on $\Xcal$ such that for all $\epsilon>0$, $N_\epsilon<\infty$. Then, the following are equivalent:
    \begin{itemize}
        \item For all $\epsilon>0$, $\dim(\epsilon;\Fcal,\Ucal)=\infty$.
        \item For all $\epsilon>0$ and $\eta\in(0,\epsilon/3]$, $\dimu(\epsilon,\eta;\Fcal,\Ucal)<\infty$.
        \item For all $\epsilon>0$, $\dimo(\epsilon;\Fcal,\Ucal)=\infty$.
        \item For all $\epsilon>0$, $\tildedim(\epsilon;\Fcal,\Ucal)=\infty$.
    \end{itemize}
\end{proposition}

To prove \cref{prop:equivalent_oblivious} we use the following simple combinatorial result on binary trees.

\begin{lemma}\label{lemma:boosting_trees}
    Let $d,N\geq 1$ be integers. Let $T_d:=\bigcup_{s=0}^d \set{0,1}^s$ be a full binary tree of depth $d$ and let $\phi:\set{0,1}^d\to[N]$ be a labeling function for the leaves of $T_d$. There exists sub-tree of $\Tcal_d$ which is a full binary tree of depth $d'=\floor{d/N}$ such that all its leaves are leaves of $T_d$ and share the same label. Formally, there exists a mapping $M:T_{d'}\to T_d$ satisfying:
    \begin{enumerate}
        \item $M$ respects the tree topology: for any $v,w\in T_{d'}$ where $v$ is an ancestor of $w$, then $M(v)$ is an ancestor of $M(w)$.
        \item Leaves of $M$ are leaves of $T_{d'}$: for any $v\in\set{0,1}^{d'}$, $M(v)\in\set{0,1}^d$.
        \item Leaves share the same label: there is $n_0\in[N]$ such that for any $v\in\set{0,1}^{d'}$, $\phi(M(v))=n_0$.
    \end{enumerate}
\end{lemma}

\begin{proof}
    Up to focusing on a smaller sub-tree of $T_d$, we assume without loss of generality that $d'=d/N$ is an integer.
    We prove by induction the following: for any $k\in\{0,\ldots,N-1\}$ there exists a sub-tree of $T_d$ which is a full binary tree of depth $d-kd'$ such that all its leaves are leaves of $T_d$ and have at most $N-k$ different labels. This is immediate for $k=0$ by using the full tree $T_d$. The depth of a node in $v\in T_d$ is defined as $s$ where $v\in\{0,1\}^{s-1}$.

    Let $k\in[N-1]$ and suppose the induction holds for $k-1$. We will work only on the corresponding sub-tree of depth $d_{k-1}:=d-(k-1)d'=(N-k+1)d'$. and identify it with $T_{d_{k-1}}$ in the rest of this proof. Denote by $S_{k-1}$ the set of labels for $\phi$ represented in the leaves of this tree. By assumption $|S_{k-1}| \leq N-k+1$. Fix any label $i\in S_{k-1}$. Suppose that there exists a node $v$ of depth $(N-k)d'+1$ such that all leaves in the sub-tree rooted at $v$ of its descendants in $T_{d_{k-1}}$ have label $i$. Then, this corresponds to a full binary sub-tree of depth $d'$ such all leaves have the same label, ending the proof. Suppose that this is not the case, then we replace each node $v$ of depth $(N-k)d'+1$ with one of its leaves that has label in $S_k:=S_{k-1}\setminus\{i\}$. The result is a full binary tree of depth $d_k=(N-k)d'$ such that its leaves all have labels in $S_k$. Since $|S_k|\leq |S_{k-1}|-1\leq N-k$ this ends the recursion.

    The sub-tree constructed for $k=d-1$ satisfies the desired properties, ending the proof.
\end{proof}

With this lemma at hand, we show the following result which compares the different dimensions introduced in this section with $\dim(\epsilon;\Fcal,\Ucal)$. In particular, this result directly implies \cref{prop:equivalent_oblivious}.

\begin{lemma}\label{lemma:comparing_various_dimensions}
    Fix a function class $\Fcal$ and a distribution class $\Ucal$ on $\Xcal$ such that for any $\epsilon>0$, $N_\epsilon<\infty$. Then, for any $\epsilon>0$ and $\eta\in(0,\epsilon/3]$,
    \begin{equation*}
        \floor{\log_{N_\eta+2} \dim(6\epsilon;\Fcal,\Ucal)} \leq \dimu(\epsilon,\eta;\Fcal,\Ucal)\leq \dim(\epsilon;\Fcal,\Ucal) \leq \dimo(\epsilon;\Fcal,\Ucal) \leq  (N_{\epsilon/24}+2)^{\dim(\epsilon/4;\Fcal,\Ucal)}.
    \end{equation*}
    Similarly,
    \begin{equation*}
        \dim(\epsilon;\Fcal,\Ucal) \leq \tildedim(\epsilon;\Fcal,\Ucal) \leq  (N_{\epsilon/15}+2)^{\dim(\epsilon/5;\Fcal,\Ucal)}.
    \end{equation*}
\end{lemma}

\begin{proof}
    Fix $\Fcal,\Ucal$ satisfying the assumption and let $\epsilon>0$ and $\eta\in(0,\epsilon/3]$. We start with the proof of the bounds on $\dimu(\epsilon,\eta;\Fcal,\Ucal)$. 
    
    \paragraph{Bounds on $\dimu(\epsilon,\delta;\Fcal,\Ucal)$.} First, since the constraint on $(\epsilon,\delta)$-shattered trees is stronger than on $\epsilon$-shattered trees, we directly have $\dimu(\epsilon,\eta;\Fcal,\Ucal)\leq \dim(\epsilon;\Fcal,\Ucal)$ because $\eta\leq \epsilon/3$.
    
    Let $d\geq 1$ such that $(N_\eta+2)^d\leq \dim(6\epsilon;\Fcal,\Ucal)$ (if such $d$ does not exist the lower bound on $\dimu(\epsilon,\eta;\Fcal,\Ucal)$ is immediate).
    We construct by induction for any $k\in\{0,\ldots,d-1\}$ a tree of depth at least $D_k:=k+2\cdot\sum_{s=0}^{d-k-1} N_\eta^s -1$ which satisfies the following. Each node $v$ is associated with a distribution $\mu_v^{(k)}\in\bar\Ucal$ but only nodes $v\in\bigcup_{s=0}^{k-1} \set{0,1}^s$ of depth at most $k$ or the last layer of nodes $v$ of depth $D_k$ have functions $f_{v,0}^{(k)},v_{v,1}^{(k)}$. Further, these satisfy:
    \begin{enumerate}
        \item Let $v$ of depth $\leq k$. Then, $\mu_v^{(k)}(f_{v,0}^{(k)} \neq f_{v,1}^{(k)})\geq 2\epsilon -2\eta$. \label{it:1}
        \item Let $v$ of depth $\leq k$. For any left-descendant function $f$ of $v$ (that is, $f=f_{w,y}^{(k)}$ where $(v,0)$ is a prefix of $(w,y)$ and $w$ has depth $D_k$), we have $\mu_v^{(k)}(f_{v,0}^{(k)} \neq f)\leq \eta$. Similarly, for any right-descendant function $f$ of $v$ and $y\in\{0,1\}$, we have $\mu_v^{(k)}(f_{v,1}^{(k)}\neq f)\leq \eta$. \label{it:2}
        \item Let $v$ of depth $>k$. Then, for any left-descendant $f^0$ of $v$ and right-descendant $f^1$ of $v$, $\mu_v(f^0\neq f^1) \geq 2\epsilon$.\label{it:3}
    \end{enumerate}

    Since $D_0\leq (N_\eta+2)^d\leq \dim(6\epsilon;\Fcal,\Ucal)$ there exists a $6\epsilon$-shattered tree of depth at least $D_0$. This satisfies all properties. Indeed, only \cref{it:3} is non-vacuous for $k=0$. And, by definition of an $6\epsilon$-shattered tree, for any node $v$, left-descendant function $f^0$ of $v$, and right-descendant function $f^1$ of $v$,
    \begin{equation*}
        \mu_v(f^0\neq f^1) \geq  \mu_v(f_{v,0}\neq f_{v,1}) - \mu_v(f_{v,0}\neq f^0) - \mu_v(f_{v,1}\neq f^1) \geq 6\epsilon -2\epsilon -2\epsilon = 2\epsilon.
    \end{equation*}
    In the first inequality we used the triangular inequality.
    
    We next fix $k\in[d-1]$ and suppose that $\mu_v^{(k-1)},f_{v,0}^{(k-1)},f_{v,1}^{(k-1)}$ have been constructed for all nodes $v$ of a full binary tree of depth $D_{k-1}$, satisfying the induction hypothesis for $k-1$. Fix any node $v_0\in\set{0,1}^{k}$ of depth $k+1$. We denote by $\mu_0=\mu_{p(v_0)}$ the distribution at the parent of $v_0$. We consider the tree rooted at $v_0$ formed by descendants of $v_0$. This is a full binary tree of depth $d_k=D_{k-1}-k $. If we also count the functions $f_{v,0},f_{v,1}$ this gives a tree of depth $d_k+1=N_\eta\cdot 2\sum_{s=0}^{d-k-1} N_\eta^s$. For convenience, we identify this tree with $\Tcal_{d_k+1}$ and focus on this tree only for now. By construction of $N_\eta$, there exists an $\eta$-covering $S_0\subseteq\Fcal$ of $\Fcal$ for $\mu_0$ of size at most $N_\eta$. We then group the leaves of the tree according to their corresponding cover:
    \begin{equation*}
        \phi:l\in\set{0,1}^{d_k}\mapsto \argmin_{f\in S_0} \mu_0(f\neq f_l),
    \end{equation*}
    where $f_l=f_{(l_i)_{i<d_k},l_{d_k}}$ is the function associated with the leaf $l$. By \cref{lemma:boosting_trees} there exists a sub-full-binary-tree of depth $(d_k+1)/N_\eta$ within this sub-tree rooted at $v_0$ such that all its leaves are leaves of the original tree, and there exists $f^0\in S_0$ such that any function $f_l$ corresponding to one of its leaves $l$ satisfies $\mu_0(f^0\neq f_l)\leq \eta$. We then replace the sub-tree rooted at $v_0$ in the original tree of depth $D_{k-1}=k+d_k$ with this full binary tree of depth $(d_k+1)/N_\eta = 2\sum_{s=0}^{d-k-1} N_\eta^s$. Next, letting $y_0:=(v_0)_k$ be the last entry of $v_0$, we also replace $f_{p(v_0),y_0}$ with $f^0$ in the original tree.

    As described in the previous paragraph, we replace all sub-trees of nodes of depth $k$ from the original tree with full-binary trees of depth $2\sum_{s=0}^{d-k-1} N_\eta^s -1$ (the $-1$ corresponds to not counting the functions $f_{v,y}$ as nodes in the last layer), as well as the functions $f_{v,y}$ for nodes $v\in\set{0,1}^{k-1}$ of depth $k$ and $y\in\{0,1\}$. The resulting tree is a full binary tree of depth $k+2\sum_{s=0}^{d-k-1} N_\eta^s-1=D_k$. We denote by $\mu_v^{(k)}f_{v,0}^{(k)},f_{v,1}^{(k)}$ the corresponding variables at each node and now check that it satisfies the desired properties. By induction and construction, it satisfies \cref{it:1} for all nodes $v$ of depth $\leq k-1$ and also satisfies \cref{it:2,it:3}. It remains to check \cref{it:1} for nodes $v\in\set{0,1}^{k-1}$ of depth $k$. Fix $f^0,f^1$ a left-descendant function and right-descendant function of $v$. By construction, these were also left-descendant and right-descendants of $v$ in the tree for iteration $k-1$. By induction, we have $\mu_v(f^0\neq f^1) \geq 2\epsilon$. Then, using \cref{it:2} on node $v$ (which is a consequence of the construction) we have
    \begin{equation*}
        \mu_v^{(k)}(f_{v,0}^{(k)}\neq f_{v,1}^{(k)}) \geq \mu_v^{(k)}(f^0\neq f^1) - \mu_v^{(k)}(f_{v,0}^{(k)}\neq f^0) - \mu_v^{(k)}(f_{v,1}^{(k)} \neq f^1) \geq 2\epsilon - 2\eta.
    \end{equation*}
    Hence, \cref{it:1} is also satisfied, which completes the induction.
    
    At the end of the recursion for $k=d-1$, we obtained a full binary tree of depth $D_{d-1}=d$ satisfying all the desired properties. Noting that $2\epsilon-2\eta\geq \epsilon$, this shows that $\dimu(\epsilon,\eta;\Fcal,\Ucal)\geq d$. This ends the proof of the lower bound on $\dimu(\epsilon,\eta;\Fcal,\Ucal)$.

    \paragraph{Bounds on $\dimo(\epsilon;\Fcal,\Ucal)$.} We next turn to proving the bounds on $\dimo(\epsilon;\Fcal,\Ucal)$.
    The lower bound is immediate: fix any $\epsilon$-shattered tree of depth $d$ and take $f_v:=f_{v,0}$ for all nodes $v\in\Tcal_d$. Fix $v\in\Tcal_d$. By definition of shattered trees, any left-descendant function $f$ satisfies the desired property $\mu_v(d_v\neq f)\leq \epsilon/3$. Now fix a right-descendant function $f$ of $v$. Using the triangular inequality, $\mu_v(f_v,\neq ) \geq \mu_v(f_v\neq f_{v,1}) - \mu_v(f_{v,1}\neq f)  \geq \epsilon - \epsilon/3 = 2\epsilon/3$. Hence, this tree is also a relaxed $\epsilon$-shattered tree. This gives $\dimo(\epsilon;\Fcal,\Ucal) \geq \dim(\epsilon;\Fcal,\Ucal)$.

    Fix $d\geq 1$ and $\eta=\epsilon/24$, and suppose that there exists a relaxed $\epsilon$-shattered tree of depth $D_\eta=2\cdot \sum_{s=0}^{d-1}N_\eta^s-1 \leq (N_\eta+2)^d$ with $\nu_v\in\bar\Ucal$ and $g_v\in\Fcal$ for all $v\in\Tcal_D$ satisfying the required properties.
    Then, the same induction as above (when bounding $\dimu(\epsilon,\eta;\Fcal,\Ucal)$) exactly constructs from the tree $\Tcal_D$ a full binary tree of depth $d$ with a distribution $\mu_v\in\bar\Ucal$ and functions $f_{v,0},f_{v,1}\in\Fcal$ for any node $v\in\Tcal_d$ satisfying the following. For any $v=(v_j)_{j\in[i]}\in\Tcal_d$, (1) $\mu_v(f_{v,0}\neq f_{v,1}) \geq \frac{\epsilon}{3}-2\eta\geq \frac{\epsilon}{4}$,
    and (2) for any ancestor $w=(v_j)_{j\in[i']}$ with $i'<i$, letting $g_{w}=f_{w,v_{i'+1}}$, we have $\mu_{w}(f_{v,0}\neq  g_w) ,  \mu_{w}(f_{v,1}\neq g_w) \leq \eta\leq \frac{\epsilon}{24}$.
    Hence, this is an $\epsilon/4$-shattered tree of depth $d$. Since $D_\eta\leq  (N_\eta+2)^d$ this ends the proof.

    \paragraph{Bounds on $\tildedim(\epsilon;\Fcal,\Ucal)$.} We clearly have $\tildedim(\epsilon;\Fcal,\Ucal) \geq \dim(\epsilon;\Fcal,\Ucal)$ since we can always choose $B_v=\Xcal$ for all nodes $v$ of a shattered tree. For the upper bound, fix $d\geq 1$ and suppose that $\tildedim(\epsilon;\Fcal,\Ucal) \geq (N_{\eta}+2)^d=:D$, where $\eta=\epsilon/15$. We fix $\nu_v\in\bar\Ucal$, $B_v\in\Sigma$, ad $g_{v,0},g_{v,1}\in\Fcal$ for $v\in\Tcal_D$ satisfying the corresponding constraints from \cref{def:subregion_oblivious}. Note that for any node $v\in \Tcal_D$, any left-descendant function $g^0$ of $v$, and right-descendant function $g^1$ of $v$, we have
    \begin{equation*}
        \nu_v(g^0\neq g^1) \geq \nu_v (g^0\neq g^1;B_v) \geq \nu_v(g_{v,0}\neq g_{v,1};B_v) - \nu_v(g_{v,0}\neq g^0;B_v) -\nu_v(g_{v,1}\neq g^1;B_v) \geq \epsilon/3.
    \end{equation*}
    Then, since $\epsilon/3-2\eta \geq 3\eta$, the same arguments as above precisely construct a $3\eta$-shattered tree of depth $d$ for $(\Fcal,\Ucal)$. That is, $d\leq \dim(\epsilon/5;\Fcal,\Ucal)$.
\end{proof}

\subsection{Necessary Conditions for Learning under Oblivious Adversaries}
\label{subsec:necessary_conditions_oblivious}

We start by proving \cref{prop:necessary_condition} which provides two necessary conditions for learning with oblivious adversaries. We prove the necessity of each condition separately.

\begin{lemma}\label{lemma:covering_numbers_finite_necessary}
    Let $\Fcal,\Ucal$ be a function class and distribution class on $\Xcal$. Then, for any $\epsilon>0$, 
    \begin{equation*}
        \OblivR_T(\Fcal,\Ucal) \geq \OblivRealizableR_T(\Fcal,\Ucal) \geq \frac{\epsilon}{2}\min(\floor{\log_2 N_\epsilon},T).
    \end{equation*}
    In particular, if there exists $\epsilon>0$ such that $\sup_{\mu\in\bar \Ucal}\Ncov_{\mu}(\epsilon;\Fcal)=\infty$, then $(\Fcal,\Ucal)$ is not learnable for oblivious adversaries in neither the agnostic nor noiseless setting.
\end{lemma}

\begin{proof}
    Fix $\Fcal,\Ucal$ and $\epsilon>0$. Also, fix any strategy for the learner $\Lcal$ for $T'\geq 1$. We let $T:=\min(\floor{\log_2 N_\epsilon},T)$ and only focus on the regret incurred during these $T$ iterations. Then, let $\mu\in\bar \Ucal$ such that $N:=\Ncov_{\mu}(\epsilon;\Fcal) \geq 2^{T'}$. By definition, there exist functions $f^1,\ldots,f^{N}\in\Fcal$ such that $\mu(f^i\neq f^j)\geq \epsilon$ for all $i,j\in[N]$ with $i\neq j$. For each $i\in[N]$ we consider the adversary $\Acal^i$ that always chooses the distribution $\mu_t=\bar \mu$ and function $f_t=f^i$ for all $t\in[T]$. Note that any adversary $\Acal^i$ for $i\in[N]$ is realizable because their values are always consistent with the function $f^i$. In particular,
    \begin{equation*}
        \OblivReg_T(\Lcal,\Acal^i) = \Ebb_{\Lcal,(x_i)_{i\in[T]}\overset{iid}{\sim}\mu}\sqb{\sum_{t=1}^T \1[\hat y_t\neq f^i(x_t)]}.
    \end{equation*}
    
    We aim to show that the learner $\Lcal$ incurs at least $\Theta(\epsilon T)$ oblivious regret under one of these adversaries. For convenience, we denote by $\Ucal_N$ the uniform distribution over $[N]$.
    To do so, we compute
    \begin{align*}
        \Ebb_{i\sim\Ucal_N} \sqb{\OblivReg_T(\Lcal,\Acal^i)} &=\sum_{t=1}^T\Ebb_{(x_i)_{i\in[t]}\overset{iid}{\sim}\mu} \Pbb_{\Lcal,i\sim\Ucal_N}\paren{\hat y_t\neq f^i(x_t)}.
    \end{align*}
    Recall that at any iteration $t\in[T]$, before making the prediction $\hat y_t$, the learner only has access to the samples $(x_s,f^i(x_s))$ for $s<t$. We therefore introduce the notation 
    \begin{equation*}
        S^i(t):=\set{j\in[N]: \forall s<t: f^i(x_s)=f^j(x_s)}
    \end{equation*}
    for the set of functions that agree with the samples from iterations prior to $t$. Then, conditionally on $(x_s)_{s\in[t]}$, we have
    \begin{align*}
        \Pbb_{\Lcal,i\sim\Ucal(N)}\sqb{\hat y_t\neq f^i(x_t)} &\geq \Ebb_{i\sim\Ucal_N} \sqb{\frac{\min\paren{\abs{\set{j\in S^i(t): f^j(x_t)=0}}, \abs{\set{j\in S^i(t): f^j(x_t)=1}} }}{|S^i(t)|} }\\
        &\geq \Ebb_{i\sim\Ucal_N}\sqb{ \frac{1}{|S^i(t)|^2}\sum_{j_1,j_2\in S^i(t)} \1[f^{j_1}(x_t)\neq f^{j_2}(x_t)] }.
    \end{align*}
    Plugging this within the previous equation gives
    \begin{align*}
        \Ebb_{i\sim\Ucal_N} \sqb{\OblivReg_T(\Lcal,\Acal^i)} &\geq \sum_{t=1}^T \Ebb_{(x_s)_{s<t}\overset{iid}{\sim}\mu} \Ebb_{i\sim\Ucal_N} \sqb{\frac{1}{|S^i(t)|^2} \sum_{j_1,j_2\in S^i(t)} \Pbb_{x_t\sim\mu}[f^{j_1}(x_t)\neq f^{j_2}(x_t)] }\\
        &\overset{(i)}{\geq} \frac{\epsilon}{2} \sum_{t=1}^T \Ebb_{(x_s)_{s<t}\overset{iid}{\sim}\mu} \sqb{1-\Pbb_{i\sim\Ucal_N}(|S^i(t)|= 1)} \\
        &\overset{(ii)}{\geq } \frac{\epsilon}{2} \sum_{t=1}^T \paren{1-\frac{2^{t-1}}{N}} \overset{(iii)}{\geq} \frac{\epsilon T}{4}.
    \end{align*}
    In $(i)$ we used the definition of the functions $f^1,\ldots,f^N$ which implies that whenever $j_1\neq j_2$ one has $\mu(f^{j_1}\neq f^{j_2})\geq \epsilon$. In $(ii)$ we noted that there are at most $2^{t-1}$ possible labelings of $(x_s)_{s<t}$. Last, in $(iii)$ we used $N\geq 2^T$. In summary, there must exist $i\in[N]$ for which $\OblivReg_T(\Lcal,\Acal^i)\geq \epsilon T/4$. Next, since $\mu\in\bar\Ucal$ is a mixture of distributions in $\Ucal$, by the law of total probability there must exist a sequence $\mu_1,\ldots,\mu_T$ such that an oblivious adversary $\Acal^\star$ choosing this sequence and the function $f^i$ at all times $t\in[T]$ (hence $\Acal^\star$ is realizable), we have $\OblivReg_T(\Lcal,\Acal^\star)\geq \epsilon T/4$. Because this holds for any strategy for the learner, this shows the desired bound on $\Rcal_{T'}^{\text{obliv,realizable}}(\Fcal,\Ucal)$.
\end{proof}

We next turn to the second condition and prove that it is necessary against oblivious adversaries for the agnostic setting. 

\begin{lemma}\label{lemma:finite_litt_dimension_necessary}
    Let $\Fcal,\Ucal$ be a function class and distribution class on $\Xcal$. Then, for any $\epsilon>0$,
    \begin{equation*}
        \OblivR_T(\Fcal,\Ucal) \gtrsim \min\paren{\epsilon T,\sqrt{\min(\dim(\epsilon;\Fcal,\Ucal),T)\cdot T}}.
    \end{equation*}

    In particular, if there exists $\epsilon>0$ such that $\dim(\epsilon;\Fcal,\Ucal)=\infty$, then $(\Fcal,\Ucal)$ is not learnable for oblivious adversaries in the agnostic setting.
\end{lemma}

\begin{proof}
    Fix $\Fcal,\Ucal$ and $\epsilon\in(0,1]$. Let $\Lcal$ be a learner strategy for $T\geq 1$ iterations. By definition of $\dim(\epsilon;\Fcal,\Ucal)$, there must exist a tree of depth $d:=\min(T,\dim(\epsilon;\Fcal,\Ucal)) $ that is $\epsilon$-shattered by $(\Fcal,\Ucal)$. Following the notations from \cref{def:dimension_oblivious}, we then fix for any inner node $v\in\Tcal_T$ a distribution $\mu_v\in\bar\Ucal$ and functions $f_{v,0},f_{v,1}\in\Fcal$ that satisfy all the required properties of a shattered tree of depth $d$. 
    
    We construct an adversary $\Acal^l$ for each leaf $l\in\{0,1\}^d$ as follows. First, let $n_0:= \floor{T/d}\geq 1$. For convenience, let $t_k:= k\cdot n_0$ for all $k\in\{0,\ldots, d\}$. We define the adversary $\Acal^l$ which at times $t\in(t_{k-1},t_k]$ for an $k\in[d]$ chooses the distribution $\mu_t$ and function $f_t$ according to the path leading to the leaf $l$. Precisely, $v_k=(l_i)_{i<k}$, we define
    $\mu_t=\mu_{v_k}$.
    Next, we let $p_k:=\frac{\min(\epsilon,1/\sqrt{n_0})}{8\cdot \mu_{v_k}(f_{v_k,0}\neq f_{v_k,1})} $. We recall that by construction, we have $\mu_{v_k}(f_{v_k,0}\neq f_{v_k,1})\geq \epsilon$ hence in all cases we have $p_k\leq 1/8$. We draw two independent Bernoulli $B_t\sim \Ber(1/2+p_k)$ and $C_t\sim\Ber(1/2)$ independently from the past and select
    \begin{equation}\label{eq:def_f_t_agnostic}
        f_t:x\mapsto C_t \1[f_{v_k,0}(x)=f_{v_k,1}(x)] + ( B_t f_{v_k,l_k}(x) +(1-B_t) f_{v_k,1-l_k}(x)) \1[f_{v_k,0}(x)\neq f_{v_k,1}(x)].
    \end{equation}
    For any $t>t_d$ we choose $f_t=f_{t_d}$ and can choose $\mu_t$ arbitrarily. For any leaf $l$, letting $f^l=f_{t_d}:=f_{v_d,l_d}$, we have
    \begin{align}\label{eq:bound_best_oblivious}
        \inf_{f\in\Fcal} \Ebb_{\Lcal,\Acal^l}\sqb{\sum_{t=1}^T \1[f(x_t)\neq y_t]} 
        &\leq \Ebb_{\Lcal,\Acal^i}\sqb{\sum_{t=1}^T \1[f^l(x_t)\neq f_t(x_t)]} \notag\\
        &\leq \sum_{k=1}^d n_0 \paren{\frac{1}{2} - p_k \paren{ \mu_{v_k}(f_{v_k,0}\neq f_{v_k,1})  - \mu_{v_k}(f^l\neq f_{v_k,l_k})  } } \notag\\
        &\leq \frac{d n_0}{2} - n_0 \sum_{k=1}^d \frac{p_k \cdot \mu_{v_k}(f_{v_k,0}\neq f_{v_k,1})}{3} = \frac{dn_0}{2} - \frac{d\sqrt{n_0}}{24}\min\paren{\epsilon\sqrt {n_0},1}.
    \end{align}
    where in the last equality we used the properties of the $\epsilon$-shattered tree as per \cref{def:dimension_oblivious}.
    
    As in the proof of \cref{lemma:covering_numbers_finite_necessary}, we aim to show that for at least one of these adversaries, the learner incurs linear oblivious regret. To do so, we consider the uniform distribution over adversaries $\Acal^l$ for all leaves $l\in\set{0,1}^T$. For convenience, we simply denote this uniform distribution as $\Unif$. Importantly, by construction, the sequential observations of the samples $(x_t,y_t)$ for $t\in[T]$ when the adversary $\Acal^l$ is sampled according to $l\sim\Unif$ is stochastically equivalent to an adversary which follows a random walk from the root to some leaf of the tree, moving to a children of the current node at times $t_k$ for $k\in[d]$. Formally, starting from the root, at the beginning of each epoch $(t_{k-1},t+k]$ for $k\in[d]$, this adversary samples an independent Bernoulli $l_k\sim\Ber(1/2)$, follows the corresponding distributions for choosing $\mu_t$ and $f_t$ during $t\in(t_{k-1},t_k]$ as described above, then moves to the node $v_{k+1} = (l_i)_{i\in[k]}$ for the next epoch. We focus on a fixed epoch $E_k:=(t_{k-1},t_k]$ for $k\in[d]$ and reason conditionally on $v_k$ and $l_k$. Note that the values $(x_t,y_t)$ are i.i.d.\ for $t\in E_k$ and denote by $\Dcal_k$ their common distribution. Note that $\Dcal_k$ is a Bernoulli distribution with parameter within $[\frac{1}{2}-\mu_{v_k}(f_{v_k,0}\neq f_{v_k,1})\cdot p_k,\frac{1}{2}+ \mu_{v_k}(f_{v_k,0}\neq f_{v_k,1})\cdot p_k ]$. Then, by Pinsker's inequality,
    \begin{align*}
        \TV\paren{\Dcal_k^{\otimes (n_0-1)},(\mu_{v_k}\times\Ber(1/2))^{\otimes (n_0-1)}} &\leq \sqrt{\frac{n_0-1}{2} \Dkl \paren{\Dcal_k\parallel \Ber(1/2)}}\\
        &\leq \sqrt{2(n_0-1)} \cdot p_k \mu_{v_k}(f_{v_k,0}\neq f_{v_k,1}) \leq \frac{1}{4}.
    \end{align*}
    We denote by $\hat y_t$ the prediction of the learner at time $t$. Then, for $t\in E_k$,
    \begin{align*}
        \Pbb_{l\sim\Unif,x_t\sim \mu_{v_k},y_t}(\hat y_t\neq y_t \mid v_k) &\geq \frac{1-p_k\cdot \mu_{v_k}(f_{v_k,0}\neq f_{v_k,1})}{2} + \frac{p_k\cdot \mu_{v_k}(f_{v_k,0}\neq f_{v_k,1})}{2}\paren{1-\frac{1}{4}}\\
        &=\frac{1}{2} - \frac{\min(\epsilon\sqrt{n_0},1)}{64\sqrt{n_0}},
    \end{align*}
    where we used the fact that with probability $1-p_k\cdot \mu_{v_k}(f_{v_k,0}\neq f_{v_k,1})$ the sample $y_t$ is exactly a Bernoulli $\Ber(1/2)$ and in the remaining probability the learner can use the previous samples $(x_s,y_s)_{t_{k-1}<s<t}$. 
    In summary, we have
    \begin{align*}
        \Ebb_{l\sim\Unif}\sqb{\sum_{t=1}^T \1[\hat y_t\neq y_t]} \geq \Ebb_{l\sim\Unif}\sqb{\sum_{k=1}^d \sum_{t\in(t_{k-1},t_k]} \1[\hat y_t\neq y_t]} \geq \frac{dn_0}{2} -\frac{d\sqrt{n_0}}{64} \min(\epsilon\sqrt{n_0},1).
    \end{align*}
    Together with \cref{eq:bound_best_oblivious}, we obtained
    \begin{equation*}
        \Ebb_{l\sim\Unif}\sqb{\OblivReg_T(\Lcal;\Acal^l)} \geq \frac{ d\sqrt{n_0}}{50}\min(\epsilon\sqrt{n_0},1) \geq \frac{\min(\epsilon T,\sqrt{dT})}{100}.
    \end{equation*}
    In the last inequality we used the definition of $n_0$ and the fact that $T\geq d$.
    Hence, there must exist a leaf $l\in \set{0,1}^T$ for which $\OblivReg_T(\Lcal;\Acal^l)\geq \min(\epsilon T,\sqrt{dT}) /100$. By the law of total probability, we there is a deterministic oblivious adversary that also achieves this regret bound for $\Lcal$ and that always selected distributions in $\Ucal$ (recall that all distributions chosen by $\Acal^l$ are mixtures from $\Ucal$). This holds for any learner strategy $\Lcal$, which ends the proof.
\end{proof}

We next extend the necessary condition of \cref{lemma:finite_litt_dimension_necessary} from the agnostic to the realizable setting. Within the proof for the agnostic case, we used an adversary which simply follows a path from a shattered tree of sufficient depth (larger than the time horizon $T$). This adversary is not realizable however, since functions $f_{v,0},f_{v,1}$ at a given node $v$ may not agree with ancestor functions exactly. Instead, we use the stronger dimension $\dimu(\epsilon,\eta;\Fcal,\Ucal)$ for sufficiently small $\eta>0$, which gives the following regret lower bound for realizable adversaries.

\begin{lemma}\label{lemma:finite_litt_dimension_necessary_realizable}
    Let $\Fcal,\Ucal$ be a function class and distribution class on $\Xcal$. Then, for any $\epsilon>0$,
    \begin{equation*}
        \OblivRealizableR_T(\Fcal,\Ucal) \geq \epsilon \cdot\min(\dimu(\epsilon, \epsilon/(4 T);\Fcal,\Ucal),T).
    \end{equation*}
    
    In particular, if there exists $\epsilon>0$ such that $\dim(\epsilon;\Fcal,\Ucal)=\infty$, then $(\Fcal,\Ucal)$ is not learnable for oblivious adversaries in the realizable setting.
\end{lemma}

\begin{proof}
    Fix $\Fcal,\Ucal$ and $\epsilon>0$. We define $d:=\min(\dimu(\epsilon, \eta;\Fcal,\Ucal),T)$, where $\eta:=\epsilon/(4 T)$. We will only focus on the first $d$ times. We can use essentially the same construction as in the proof of \cref{lemma:finite_litt_dimension_necessary}. We fix an $(\epsilon,\eta)$-shattered tree of depth $d$, and denote for all $v\in\Tcal_d$ the corresponding variables $\mu_v\in\bar\Ucal$ and $f_{v,0},f_{v,1}\in\Fcal$ satisfying the required properties for shattered trees.
    For each leaf $l\in\{0,1\}^d$ we construct the adversary which at time $t\in[d]$ selects the distribution $\mu_t:=\mu_{v_t}$ where $v_t=(l_i)_{i<t}$, and the function $f_t:=f_{v_t,l_t}$. For any $t>d$ the adversary chooses $f_t=f_d$ and an arbitrary distribution $\mu_t\in\Ucal$. Essentially this corresponds to $n_0=1$ and $p_k=1$ for $k\in[d]$ in the construction from the proof of \cref{lemma:finite_litt_dimension_necessary}. Then, taking the uniform distribution $\Ucal$ over these adversaries, using similar arguments we have for any learner $\Lcal$:
    \begin{align*}
        \Ebb_{l\sim\Ucal}\sqb{\sum_{t=1}^T \1[\hat y_t\neq y_t]} \geq \Ebb_{l\sim\Ucal}\sqb{\sum_{t=1}^d \1[\hat y_t\neq y_t]} \geq \Ebb_{l\sim\Ucal}\sqb{\sum_{t=1}^d \frac{1}{2} \mu_{v_t}(f_{v_t,0}\neq f_{v_t,1}) } \geq \frac{\epsilon d}{2}.
    \end{align*}
    Hence, we can fix a leaf $l$ for which $\OblivReg_T(\Lcal,\Acal^l) \geq \epsilon d/2$. At this point, the adversary $\Acal^l$ is not realizable since at each iteration $t\in[d]$ the adversary chooses a different function. However, by the union bound,
    \begin{equation*}
        \Pbb\paren{\exists t\in[d]: y_t\neq f^l(x_t)} \leq \sum_{t=1}^d \mu_t(f_t\neq f^l) \leq \eta d,
    \end{equation*}
    where in the last inequality we used the fact that by construction $f^l$ is equal to or a descendant of the function $f_t$. Next, denote by $\Acal^r$ the realizable adversary which selects the distributions $\mu_1,\ldots,\mu_T$ exactly as $\Acal$ but selects $f^l$ at all times. The previous equation then implies
    \begin{align*}
        \OblivReg_T(\Lcal;\Acal^r) &\geq \Ebb_{\Lcal,\Acal^r}\sqb{\sum_{t=1}^d\1[\hat y_t\neq f^l(x_t)]}\\
        &\geq \Ebb_{\Lcal,\Acal}\sqb{\sum_{t=1}^d\1[\hat y_t\neq f_t(x_t)]} - d\cdot \Pbb\paren{\exists t\in[d]: y_t\neq f^l(x_t)} \geq \frac{\epsilon d}{2} - \eta d^2 \geq \frac{\epsilon d}{4}.
    \end{align*}
    As in the proof of \cref{lemma:finite_litt_dimension_necessary}, by the law of total probability, there is an adversary that always selects distributions in $\Ucal$ and achieves the same regret bound as above.
    This holds for any learner $\Lcal$, which ends the proof.
\end{proof}

Combining \cref{lemma:covering_numbers_finite_necessary,lemma:finite_litt_dimension_necessary,lemma:finite_litt_dimension_necessary_realizable} proves \cref{prop:necessary_condition}. 

\subsection{Regret Bounds for the Realizable Setting}
\label{subsec:oblivious_realizable}

We next prove that these necessary conditions are also sufficient when $\Fcal$ has finite VC dimension $d$. In this case, the first condition that for any $\epsilon>0$, $N_\epsilon<\infty$ is directly implied by the fact that for any distribution $\mu$ on $\Xcal$, $\Ncov_{\mu}(\epsilon;\Fcal)=\Ocal(d\log\frac{1}{\epsilon})$ (e.g., see \cite[Lemma 13.6]{boucheron2013concentration}). In particular, this can be used to bound $\dimo(\epsilon;\Fcal,\Ucal)$ and $\dimu(\epsilon,\eta;\Fcal,\Ucal)$ in terms of $\dim(\epsilon;\Fcal,\Ucal)$ using \cref{lemma:comparing_various_dimensions}.

We start by designing an algorithm to achieve approximately $\epsilon$ average regret for some fixed $\epsilon>0$ in the realizable setting. The algorithm proceeds by epochs $E_k:=(t_{k-1},t_k]$ for $k\geq 1$. At the beginning of each epoch $E_k$, we construct a finite set of experts $S_k\subseteq \Fcal$ as follows. First, define
\begin{equation}\label{eq:def_version_space}
    \Fcal_k:=\set{f\in\Fcal:\forall s\leq t_{k-1},f(x_s)=y_s},
\end{equation}
the current version space. We construct a relaxed $\epsilon$-shattered tree for $(\Fcal_k,\Ucal)$ with maximum depth $d_k:=\dimo(\epsilon;\Fcal_k,\Ucal)$. If $d_k=0$, we simply define $S_k:=\{f\}$ for any arbitrary $f\in\Fcal_k$. Otherwise, when $d_k\geq 1$ we define the set of experts
\begin{equation}\label{eq:def_updated_set_experts}
    S_k:=\set{f_{v,y}: v\in\set{0,1}^{d_k-1},y\in\{0,1\} } 
\end{equation}
as all leaf functions from the relaxed $\epsilon$-shattered tree. With this set of experts $S_k$ at hand, the algorithm implements a minimum-loss-dependent version of the classical Hedge algorithm \cite{cesa2006prediction}. Last, we end the current epoch at the first time $t_k$ when all experts have at least $2\epsilon$ average loss and provided that the epoch has length at least $n_\epsilon:=c_0 \frac{d \log T}{\epsilon}$.
    The corresponding algorithm is summarized in \cref{alg:oblivious_realizable}.

\begin{algorithm}[t]

    \caption{Algorithm for $\Ocal(\epsilon)$ average oblivious regret in the realizable setting}\label{alg:oblivious_realizable}
    
    \LinesNumbered
    \everypar={\nl}
    
    \hrule height\algoheightrule\kern3pt\relax
    \KwIn{Function class $\Fcal$, distribution class $\Ucal$, horizon $T$, tolerance $\epsilon>0$}
    
    \vspace{3mm}

    Initialization: $n_\epsilon:=c_0 \frac{d \log T}{\epsilon}$, $t\gets  1$
    
    \For{$k\geq 1$}{
        Define $t_{k-1}=t-1$, the version space $\Fcal_k$ via \cref{eq:def_version_space},
        and construct a relaxed $\epsilon$-shattered tree for $(\Fcal_k,\Ucal)$ of depth $d_k=\dimo(\epsilon;\Fcal_k,\Ucal)$

        \lIf{$d_k=0$}{$S_k= \{f\}$ for any arbitrary $f\in\Fcal_k$}
        \lElse{Define $S_k$ via \cref{eq:def_updated_set_experts}}

        \vspace{3mm}

        Initialize Hedge algorithm $L_{t,f}=0$ for all $f\in S_k$
        
        \While{$t-t_{k-1}\leq n_\epsilon$ or $\min_{f\in S_k}L_{t,f} \leq  2\epsilon(t-t_{k-1}-1) $ \label{line:check_good_prediction}}{
            Let $p_{t,f} = \frac{e^{-\eta_t L{t,f}}}{\sum_{f\in S_k}e^{-\eta_k  {t,f}}}$ where $\eta_t:=\sqrt{\frac{2\log|S_k|}{ 1+\min_{f\in S_k}L_{t,f} }}$, and sample $\hat f_t\sim p_t$ independently from history

            Observe $x_t$ and predict $\hat y_t=\hat f_t(x_t)$

            Observe $y_t$ and update $L_{t+1,f}=L_{t,f}+\1[f(x_t)\neq y_t]$ for $f\in S_k$

            $t\gets t+1$
        }
    }
    
    \hrule height\algoheightrule\kern3pt\relax
    \end{algorithm}

To prove an oblivious regret bound on \cref{alg:oblivious_realizable}, we need the following regret bound for the Hedge algorithm variant.

\begin{theorem}[Corollary 2.4 + Exercise 2.10 + Lemma A.8 of \cite{cesa2006prediction}]
\label{thm:adaptive_hedge_guarantee}
    Consider the learning with $K$ experts problem, and denote by $\ell_{t,i}\in[0,1]$ the loss of expert $i\in[K]$ at time $t\geq 1$. Denote by $L_{t,i}=\sum_{s=1}^{t-1} \ell_{s,i}$ the cumulative loss of expert $i$ before time $t$. We focus on the Hedge algorithm (exponentially weighted average forecaster) which at each iteration follows the prediction of expert $\hat i_t\sim p_t$ where $p_{t,i} = \frac{e^{-\eta_t L_{t,i}}}{\sum_{j\in[K]} e^{-\eta_t L_{t,j}}}$ for $i\in[K]$, with learning rate $\eta_t= \sqrt{\frac{2\log K}{1+\min_{i\in[K]} L_{t,i}}}$. Then, for any $T\geq 1$, with probability at least $1-\delta$,
    \begin{equation*}
        \sum_{t=1}^T \ell_{t,\hat i_t} \leq \min_{i\in[K]} L_{T,i} + c\sqrt{\min_{i\in[K]} L_{T,i} \cdot \log \frac{K}{\delta}} + c\log \frac{K}{\delta},
    \end{equation*}
    for some universal constant $c\geq 1$.
\end{theorem}

We also need the following uniform concentration result for VC classes.

\begin{theorem}[Uniform concentration for VC classes]
\label{lemma:uniform_concentration}
    There exist a universal constant $C>0$ such that the following holds. Let $\Fcal$ be a function class on $\Xcal$ with VC dimension $d\geq 1$. Let $n\geq 2$ and fix distributions $\mu_1,\ldots,\mu_n$ on $\Xcal$. Let $x_1,\ldots,x_n$ be independent variables such that $x_i\sim\mu_i$ for $i\in[n]$. Denote
    \begin{equation*}
        L_n(f):=\frac{1}{n}\sum_{i=1}^n \Ebb_{x\sim\mu}[f(x)] \quad \text{and} \quad \hat L_n(f):=\frac{1}{n}\sum_{i=1}^n f(x_i),\quad f\in\Fcal.
    \end{equation*}
    Then, with probability at least $1-\delta$,
    \begin{equation*}
        \sup_{f\in\Fcal} \abs{\hat L_n(f)-L_n(f)} \leq C\sqrt{\frac{d\log n+\log\frac{1}{\delta}}{n}}.
    \end{equation*}
    With probability at least $1-\delta$, we also have
    \begin{equation*}
        \forall f\in\Fcal,\quad \abs{\hat L_n(f)- L_n(f) } \leq  C\sqrt{\hat L_n(f) \frac{d\log n + \log\frac{1}{\delta}}{n}} + C\frac{d\log n+\log\frac{1}{\delta}}{n}.
    \end{equation*}
\end{theorem}

\begin{proof}
    These bounds are classical when the variables $x_1,\ldots,x_n$ are i.i.d. \cite{vapnik1998statistical,cortes2019relative}.The same proof applies even if these are not identically distributed (so long as these are independent); namely the symmetrization lemma (e.g. see \cite[Lemma 2 and 3]{cortes2019relative}) still applies and we can then use Hoeffding's inequality together with Sauer-Shelah's lemma \cite{sauer1972density,shelah1972combinatorial} to bound the symmetrized relative deviation term, as in the original proof of \cite{vapnik1998statistical}.
\end{proof}

We are now ready to present the regret bound on \cref{alg:oblivious_realizable}.

\begin{theorem}\label{thm:oblivious_regret_realizable}
    Let $\Fcal$ be a function class with finite VC dimension $d\geq 1$. Let $\Ucal$ be a distribution class such that for all $\epsilon>0$, $\dim(\epsilon;\Fcal,\Ucal)<\infty$. Then, for any $\epsilon>0$, $T\geq 1$, and any oblivious and realizable adversary $\Acal$, \cref{alg:oblivious_realizable} run with the tolerance parameter $\epsilon$ satisfies:
    \begin{equation*}
        \OblivReg_T(\cref{alg:oblivious_realizable}(\epsilon);\Acal) \lesssim \epsilon T + (\dimo(\epsilon;\Fcal,\Ucal)+1)\paren{ \frac{d\log T }{\epsilon } +\dimo(\epsilon;\Fcal,\Ucal)}.
    \end{equation*}
\end{theorem}

\begin{proof}
    We fix an oblivious and realizable adversary $\Acal$ and denote by $\mu_1,\ldots,\mu_T$ the distributions chosen by the oblivious adversary. In this realizable setting, there exists $f^*\in\Fcal$ such that the event $\Ecal:=\set{\forall t\in[T], y_t=f^\star(x_t)}$ has probability one. The function class $\{f-f^\star,f\in\Fcal\}$ has VC dimension $d$. For convenience, for any $0\leq t_1<t_2\leq T$, we introduce the notation
    \begin{equation*}
        L_{t_1:t_2}(f;f^\star):= \frac{1}{t_2-t_1}\sum_{t=t_1+1}^{t_2} \mu_t(f\neq f^\star) \quad \text{and}\quad \hat L_{t_1:t_2}(f;f^\star):=\frac{1}{t_2-t_1}\sum_{t=t_1+1}^{t_2} \1[f(x_t)\neq f^\star(x_t)]
    \end{equation*}
    Then, \cref{lemma:uniform_concentration} together with the union bound implies that the event
    \begin{equation}\label{eq:def_event_G}
        \Gcal:=\bigcap_{0\leq t_1<t_2\leq T}\set{\forall f\in\Fcal : |\hat L_{t_1:t_2}(f;f^\star) - L_{t_1:t_2}(f;f^\star) | \leq c_1\sqrt{\hat L_{t_1:t_2}(f;f^\star) \frac{d\log T}{t_2-t_1}} + c_1\frac{d\log T}{t_2-t_1} }
    \end{equation}
    has probability at least $1-1/T$ for some universal constant $c_1>0$. Denote by $k_{max}$ the total number of epochs when running \cref{alg:oblivious_realizable}. By convention we pose $t_{k_{max}}=T$.
    
    We now focus on a single completed epoch $E_k$ for $k\in[k_{max}-1]$. By construction, $S_k$ is the set of leaves of a relaxed $\epsilon$-shattered tree for $(\Fcal_k,\Ucal)$. Also, because the epoch ended before the horizon $T$, by design of \cref{alg:oblivious_realizable}, under $\Ecal$ we have $t_k-t_{k-1}\geq n_\epsilon=c_0 \frac{d\log T}{\epsilon}$ and
    \begin{equation}\label{eq:termination_criterion}
        \min_{f\in S_k} \hat L_{t_{k-1},t_k}(f;f^\star) = \frac{1}{t_k-t_{k-1}} \min_{f\in S_k} \sum_{t=t_{k-1}+1}^{t_k} \1[f(x_s)\neq y_s] > 2\epsilon.
    \end{equation}
    Then, under $\Ecal\cap\Gcal$, any $f\in S_k$ satisfies
    \begin{align*}
        L_{t_{k-1},t_k}(f;f^\star) &\geq \hat L_{t_{k-1},t_k}(f;f^\star) - c_1\sqrt{\hat L_{t_{k-1},t_k}(f;f^\star) \frac{d\log T}{t_k-t_{k-1}}} - c_1\frac{d\log T}{t_k-t_{k-1}} \\
        &\overset{(i)}{\geq} \hat L_{t_{k-1},t_k}(f;f^\star) - \frac{c_1}{\sqrt c_0}\sqrt{\epsilon \hat L_{t_{k-1},t_k}(f;f^\star)} - \frac{c_1}{c_0}\epsilon\\
        &\overset{(ii)}{\geq} \frac{1}{2}\hat L_{t_{k-1},t_k}(f;f^\star) \overset{(iii)}{\geq} \epsilon,
    \end{align*}
    whenever $c_0\geq 16 \max(c_1^2,1)$.
    In $(i)$ we used $t_k-t_{k-1}\geq n_\epsilon$ and in $(ii)$ and $(iii)$ we used \cref{eq:termination_criterion}. Next, define
    \begin{equation*}
        \bar\mu_k:= \frac{1}{t_k-t_{k-1}} \sum_{t=t_{k-1}+1}^{t_k} \mu_t.
    \end{equation*}
    We showed that under $\Ecal\cap\Gcal$, 
    \begin{equation}\label{eq:all_S_k_far}
        \forall f\in S_k,\quad \bar \mu_k(f\neq f^\star) \geq \epsilon.
    \end{equation}
    On the other hand, any function $f\in\Fcal_{k+1}$ has $f(x_t)=y_t$ for all $t\in E_k$. Hence, under $\Gcal$ we obtain
    \begin{align}
        \bar\mu_k(f\neq f^\star)= L_{t_{k-1},t_k}(f,f^\star) 
        &\leq \hat L_{t_{k-1},t_k}(f;f^\star) + c_1\sqrt{\hat L_{t_{k-1},t_k}(f;f^\star) \frac{d\log T}{t_k-t_{k-1}}} + c_1\frac{d\log T}{t_k-t_{k-1}} \notag \\
        &\leq c_1\frac{d\log T}{t_k-t_{k-1}} \overset{(i)}{\leq} \frac{c_1}{c_0}\epsilon \leq \frac{\epsilon}{4}, \label{eq:functions_S_k+1_close}
    \end{align}
    where in $(i)$ we used $t_k-t_{k-1}\geq n_\epsilon$. 
    
    With these properties at hand, we now aim to show that under $\Ecal\cap\Gcal$, $d_{k+1}<d_k$.
    Denote by $\Tcal^{(k)}$ the relaxed $\epsilon$-shattered full binary tree of depth $d_k$ that was constructed for $(\Fcal_k,\Ucal)$ at the beginning of epoch $k\in[k_{max}]$. Fix $k\in[k_{max}-1]$ and suppose by contradiction that $d_{k+1}\geq d_k$ (we recall that $d_k:=\dimo(\epsilon;\Fcal_k,\Ucal)$). Since $\Fcal_{k+1}\subseteq \Fcal_k$ this implies $d_{k+1}=d_k$. We then consider a new full-binary tree $\Tcal$ of depth $d_k+1$ such that the root $r$ has $\mu_r:=\bar\mu_k$ and $f_{r,0}:=f^\star$, the left sub-tree from the root is defined as $\Tcal^{(k+1)}$ and the right sub-tree from the root is defined as $\Tcal^{(k)}$. Under $\Ecal\cap\Gcal$, \cref{eq:all_S_k_far} precisely shows that any right-descendant function $f$ of the root $r$ in $\Tcal$ satisfies $\mu_r(f_{r,0}\neq f) = \bar\mu_k(f^\star\neq f)\geq \epsilon$. On the other hand, any left-descendant function $f$ of $r$ in $\Tcal$ satisfies $f\in\Fcal_{k+1}$ by definition of $\Tcal$. Hence, \cref{eq:functions_S_k+1_close} implies $\mu_r(f_{r,0}\neq f)= \bar\mu_k(f^\star\neq f)\leq \epsilon/4$. Last, note that under $\Ecal\cap\Gcal$, all functions from $\Tcal$ lie in $\Fcal_k$.
    In summary, $\Tcal$ is a relaxed $\epsilon$-shattered binary tree for $(\Fcal_k,\Ucal)$ of depth $d_k+1$, contradicting the definition of $d_k$.

    As a result, under $\Ecal\cap\Gcal$, there are at most $d_1=\dimo(\epsilon;\Fcal,\Ucal)$ completed epochs, that is, $k_{max}\leq 1+\dimo(\epsilon;\Fcal,\Ucal)$. We next bound the regret incurred on each epoch $k\in[k_{max}]$. By construction of \cref{alg:oblivious_realizable}, each epoch $k\in[k_{max}]$ satisfies
    \begin{align}
        L_k^\star:=\min_{f\in S_k} \sum_{t\in E_k} \1[f(x_t)\neq y_t] 
        &\leq \begin{cases}
            n_\epsilon &\text{if } t_k-t_{k-1}-1<n_\epsilon\\
            1+ 2\epsilon(t_k-t_{k-1}-1) &\text{otherwise}
        \end{cases}  \notag \\
        &\leq n_\epsilon + 2\epsilon(t_k-t_{k-1}) +1 \label{eq:best_regret_each_epoch}
    \end{align}
    We next define the following event in which we bound the regret of the modified Hedge algorithm at all epochs:
    \begin{equation*}
        \Hcal:=\set{\forall k\in[k_{max}], \sum_{t\in E_k} \1[\hat y_t\neq y_t] \leq L_k^\star + c_2\sqrt{L_k^\star  \log(T|S_k|)} + c_2\log(T|S_k|) },
    \end{equation*}
    where $c_2=3c$ and $c$ is the universal constant appearing in \cref{thm:adaptive_hedge_guarantee}. We then apply \cref{thm:adaptive_hedge_guarantee} to each epoch and take the union bound which implies that $\Hcal$ has probability at least $1-1/T$. Here we remarked that for each epoch $k\in[k_{max}]$, the end time of the epoch $t_k$ is not known a priori, however, we can take the union bound over all possible end times $t_k\in[T]$ (and there are at most $T$ epochs). Under $\Ecal\cap\Gcal\cap\Hcal$, we obtain
    \begin{align*}
        \sum_{t=1}^T \1[\hat y_t\neq y_t] &= \sum_{k=1}^{k_{max}} \sum_{t\in E_k} \1[\hat y_t\neq y_t]\\
        &\overset{(i)}{\leq} 2c_2 \sum_{k=1}^{k_{max}} \paren{L_k^\star + \log(T|S_k|) }\\
        &\overset{(ii)}{\leq} 2c_2 \cdot \epsilon T + 2c_2\cdot  k_{max}(n_\epsilon +1 + \log T + \dimo(\epsilon;\Fcal,\Ucal)\cdot\log 2 )\\
        &\overset{(iii)}{\leq} c_3\cdot  \epsilon T + c_3 \cdot (\dimo(\epsilon;\Fcal,\Ucal)+1)\paren{ \frac{d\log T }{\epsilon} +\dimo(\epsilon;\Fcal,\Ucal)}
    \end{align*}
    for some universal constant $c_3>0$.
    In $(i)$ we used $\Hcal$, in $(ii)$ we used \cref{eq:best_regret_each_epoch} together with the fact that $\log_2(|S_k|)=d_k \leq d_1= \dimo(\epsilon;\Fcal,\Ucal)$. In $(iii)$ we used $k_{max}\leq 1+\dimo(\epsilon;\Fcal,\Ucal)$ under $\Ecal\cap\Gcal$. Recalling that $\Pbb[\Ecal^c]+\Pbb[\Gcal^c]+\Pbb[\Hcal^c]\leq 2/T$, we obtained the desired oblivious regret bound by taking the expectation.
\end{proof}

\subsection{Regret Bounds for the Agnostic Setting}
\label{subsec:oblivious_adaptive}

We next turn to the agnostic case for which we use a similar algorithm to \cref{alg:oblivious_realizable} with adjusted parameters. Intuitively, while in \cref{alg:oblivious_realizable} we considered a specific choice of epochs and performed a Hedge variant algorithm on each epoch, in the agnostic case we will simply perform the Hedge algorithm on all experts for any possible choice of epochs.

Formally, fix $\epsilon>0$ and define the shorthand $D:=\dimo(\epsilon;\Fcal,\Ucal)+1$. We recall that here $N_\epsilon= \Ocal(d\log\frac{1}{\epsilon})$ where $d$ is the VC dimension of $\Fcal$. We let $n_\epsilon :=c_0 \frac{d\log T}{\epsilon}$. Fix a sequence of integers $t_0=0 <t_1<\ldots<t_{k'}=T$ with $k'\leq D$. We construct online a sub-function class $\Fcal_k\subseteq \Fcal$ and a corresponding cover $S_k\subseteq \Fcal_k$ for $k\in[k']$ as follows. We first let $\Fcal_1=\Fcal$. For any $k\in[k']$, having defined $\Fcal_k$, we construct a relaxed $\epsilon$-shattered tree for $(\Fcal,\Ucal)$ of maximum depth $d_k:=\dimo(\epsilon;\Fcal_k,\Ucal)$. We then define $S_k$ as the set of all leaf functions of this tree, exactly as in \cref{eq:def_updated_set_experts}. At the end of this epoch, if $k<k'$, we define the function class for the next epoch via
\begin{equation}\label{eq:def_function_class}
    \Fcal_{k+1} := \set{f\in\Fcal_k: \forall g\in S_k, \;\frac{1}{t_k-t_{k-1}}\sum_{t\in(t_{k-1},t_k]} \1[f(x_t)\neq g(x_t)] \geq 2\epsilon}.
\end{equation}
Since for any $k\in[k']$, $\dimo(\epsilon;\Fcal_k,\Ucal)\leq D-1$, we have $|S_k|\leq 2^{D-1}$. Hence, by enumerating the functions in $S_1,\ldots,S_{k'}$, we can specify all possible combinations of functions $(f_1,\ldots,f_{k'})\in S_1\times\ldots\times S_{k'}$ using sequences of their indices $i_1,\ldots,i_{k'}\in[2^{D-1}]$.

We consider all experts obtained from any pair of sequences $(t_1,i_1),\ldots,(t_{k'},i_{k'})\in [T]\times [2^{D-1}]$ with $k'\leq D$ and $1\leq t_1\leq \ldots\leq t_{k'}=T$, by constructing online the corresponding functions $f_k\in S_k$ for $k\in[k']$ and following the predictions of $f_k$ on epoch $(t_{k-1},t_k]$ (if for some $k\in[k']$ we have $i_k>|S_k|$, we can define this expert arbitrarily). We denote this expert by $E((t_k,i_k)_{k\in[k']})$ and let $\Scal_E$ be the set of all these experts. The algorithm runs the classical Hedge algorithm on the set of experts $\Scal_E$. For completeness, we include its pseudo-code in \cref{alg:oblivious_agnostic}. To analyze the oblivious regret of \cref{alg:oblivious_agnostic} we use the following classical regret bound on Hedge.

\begin{theorem}[Corollary 4.2 \cite{cesa2006prediction}]
\label{thm:regret_hedge}
    Consider the learning with $K$ experts problem with horizon $T$, and denote by $\ell_{t,i}\in[0,1]$ the loss of expert $i\in[K]$ at time $t\geq 1$. Denote by $L_{t,i}=\sum_{s=1}^{t-1} \ell_{s,i}$ the cumulative loss of expert $i$ before time $t$. We focus on the Hedge algorithm (exponentially weighted average forecaster) which at each iteration follows the prediction of expert $\hat i_t\sim p_t$ where $p_{t,i} = \frac{e^{-\eta L_{t,i}}}{\sum_{j\in[K]} e^{-\eta L_{t,j}}}$ for $i\in[K]$, with learning rate $\eta= \sqrt{\frac{8\log K}{T}}$. Then, for any $T\geq 1$, with probability at least $1-\delta$,
    \begin{equation*}
        \sum_{t=1}^T \ell_{t,\hat i_t} -\min_{i\in[K]} L_{T,i} \leq  \sqrt{\frac{T\log K}{2}} + \sqrt{\frac{T}{2}\log \frac{1}{\delta}},
    \end{equation*}
    for some universal constant $c>0$.
\end{theorem}

\begin{algorithm}[t]

    \caption{Hedge algorithm}\label{alg:oblivious_agnostic}
    
    \LinesNumbered
    \everypar={\nl}
    
    \hrule height\algoheightrule\kern3pt\relax
    \KwIn{Horizon $T$, set of experts $\Scal_E$. $y_{t,E}$ denotes the prediction of expert $E\in\Scal_E$ at time $t\in[T]$}
    
    \vspace{3mm}

    Initialization: $\eta:=\sqrt{\frac{8\log|\Scal_E|}{T}}$, $L_{1,E} = 0$ for all $E\in \Scal_E$
    
    \For{$t\in[T]$}{
        Let $p_{t,E} = \frac{e^{-\eta L_{t,E}}}{\sum_{E'\in \Scal_E}e^{-\eta  L_{t,E'}}}$ and sample $\hat E_t\sim p_t$ independently from history

        Observe $x_t$ and predict $\hat y_t= y_{t,\hat E_t}$

        Observe $y_t$ and update $L_{t+1,E}=L_{t,E}+\1[y_{t,E}\neq y_t]$ for $E\in \Scal_E$
    }
    
    \hrule height\algoheightrule\kern3pt\relax
\end{algorithm}

Our regret bound for \cref{alg:oblivious_agnostic} is the following.

\begin{theorem}\label{thm:oblivious_regret_agnostic}
    Let $\Fcal$ be a function class with finite VC dimension $d\geq 1$. Let $\Ucal$ be a distribution class such that for all $\epsilon>0$, $\dim(\epsilon;\Fcal,\Ucal)<\infty$. Then, for any $\epsilon>0$, $T\geq 1$, and any oblivious and realizable adversary $\Acal$, \cref{alg:oblivious_agnostic} run with the tolerance parameter $\epsilon$ satisfies:
    \begin{equation*}
        \OblivReg_T(\cref{alg:oblivious_agnostic}(\epsilon);\Acal) \lesssim \epsilon T + (\dimo(\epsilon;\Fcal,\Ucal)+1)\paren{\sqrt{T\log T} + \frac{d\log T}{\epsilon
        }  }.
    \end{equation*}
\end{theorem}

\begin{proof}
    Fix an oblivious adversary $\Acal$ and denote by $\mu_1,\ldots,\mu_T$ and $f_1,\ldots,f_T$ its decisions.
    By construction, there are $|\Scal_E|\leq D(T2^D)^D$ experts. Then, \cref{thm:regret_hedge} implies that the event
    \begin{equation*}
        \Ecal:=\set{ \sum_{t=1}^T \1[\hat y_t\neq y_t] -\min_{E\in\Scal_E} \sum_{t=1}^T \1[y_{t,E}\neq y_t] \leq  c_1\sqrt{DT(D+\log T)} }
    \end{equation*}
    has probability at least $1-1/T$ for some universal constant $c_1>0$ (note that $D\geq 1$). Next, we fix $f^\star\in\Fcal$ such that
    \begin{equation}\label{eq:definition_f_star}
        \Ebb\sqb{\sum_{t=1}^T \1[f^\star(x_t)\neq y_t]} \leq \inf_{f\in\Fcal}\Ebb\sqb{\sum_{t=1}^T \1[f(x_t)\neq y_t]} + 1.
    \end{equation}

    We next consider the realizable adversary $\Acal^r$ that chooses the same distributions $\mu_1,\ldots,\mu_T$ as $\Acal$ but always selects the function $f^\star$ at all times. We consider running the algorithm for the realizable case \cref{alg:oblivious_realizable} with adversary $\Acal^r$. Denote the epochs defined by \cref{alg:oblivious_realizable} by $E_k:=(t_{k-1},t_k]$ for $k\in[k_{max}]$. Define the same event $\Gcal$ as defined within the proof of \cref{thm:oblivious_regret_realizable} in \cref{eq:def_event_G}. The proof then shows that under $\Gcal$, we have $k_{max}\leq 1+\dimo(\epsilon;\Fcal,\Ucal)=D$ (note that by construction the event $\set{\forall t\in[T]: y_t=f^\star(x_t)}$ is automatically satisfied). Further, by construction of these epochs within \cref{alg:oblivious_realizable}, as detailed within \cref{eq:best_regret_each_epoch} we have
    \begin{equation*}
        \min_{f\in S_k} \sum_{t\in E_k} \1[f(x_t)\neq f^\star(x_t)] \leq n_\epsilon+2\epsilon(t_k-t_{k-1})+1,
    \end{equation*}
    where $n_\epsilon=c_0\frac{d\log T}{\epsilon}$ for some universal constant $c_0>0$.
    In particular, under $\Gcal$, the expert $E^\star\in\Scal_k$ which had the parameters $t_1,\ldots,t_{k_{max}}$ and the indices of the functions $f_k:=\argmin_{f\in S_k} \sum_{t\in E_k} \1[f(x_t)\neq f^\star(x_t)]$ for all $k\in[k_{max}]$ has 
    \begin{equation*}
        \sum_{t=1}^T \1[y_{t,E^\star}\neq f^\star(x_t)] \leq \sum_{k=1}^{k_{max}} \min_{f\in S_k} \sum_{t\in E_k} \1[f(x_t)\neq f^\star(x_t)]  \leq 2\epsilon T+ D(n_\epsilon+1).
    \end{equation*}
    In the last inequality we used the fact that $k_{max}\leq D$ under $\Gcal$. Putting things together, under $\Ecal\cap\Gcal$ we obtained
    \begin{align*}
        \sum_{t=1}^T\1[\hat y_t\neq y_t] &\overset{(i)}{\leq} \sum_{t=1}^T\1[y_{t,E^\star}\neq y_t] + c_1\sqrt{DT(D+\log T)}\\
        &\leq  \sum_{t=1}^T\1[f^\star(x_t)\neq y_t] + \sum_{t=1}^T\1[y_{t,E^\star}\neq f^\star(x_t)] + c_1\sqrt{DT(D+\log T)}\\
        &\leq \sum_{t=1}^T\1[f^\star(x_t)\neq y_t] + 2\epsilon T + c_1\sqrt{DT(D+\log T)} + 2c_0 \frac{D d\log T}{\epsilon},
    \end{align*}
    where in $(i)$ we used $\Ecal$. Taking the expectation and recalling that $\Pbb[\Ecal^c]+\Pbb[\Gcal^c]\leq 2/T$, we obtained
    \begin{align*}
        \OblivReg_T(\cref{alg:oblivious_agnostic};\Acal) &\overset{(i)}{\leq}
        \Ebb\sqb{\sum_{t=1}^T\1[\hat y_t\neq y_t]-\sum_{t=1}^T\1[f^\star(x_t)\neq y_t]} + 1\\
        &\leq 2\epsilon T + c_1 D\sqrt{T\log T} + 2c_0 \frac{D d\log T}{\epsilon} + 3.
    \end{align*}
    In $(i)$ we used the definition of $f^\star$ from \cref{eq:definition_f_star}. This ends the proof.
\end{proof}

\cref{thm:oblivious_adversaries} is obtained by combining
\cref{prop:necessary_condition,thm:oblivious_regret_realizable,thm:oblivious_regret_agnostic}, together with the fact that function classes $\Fcal$ with finite VC dimension always satisfy $N_\epsilon\lesssim d\log\frac{1}{\epsilon}$. More precisely, we obtained the following non-asymptotic bounds on the minimax oblivious regret.

\begin{theorem}\label{thm:combined_oblivious_non_asymptotic}
    Let $\Fcal,\Ucal$ be a function class and distribution class on $\Xcal$. Then,
    \begin{multline*}
        \sup_{\epsilon\in(0,1]}\set{\min\paren{\epsilon T,\sqrt{\min(\dim(\epsilon;\Fcal,\Ucal),T)\cdot T}} } \lesssim \OblivR_T(\Fcal,\Ucal)\\
        \lesssim \inf_{\epsilon\in(0,1]} \set{ \epsilon T+ (\dimo(\epsilon;\Fcal,\Ucal)+1)\paren{\sqrt{T\log T } + \frac{d\log T}{\epsilon }  } }.
    \end{multline*}
    \begin{multline*}
        \sup_{\epsilon\in(0,1]}\set{\epsilon  \cdot  \min(\dimu(\epsilon,\epsilon/(4T);\Fcal,\Ucal), T)} 
        \lesssim  
        \OblivRealizableR_T(\Fcal,\Ucal)\\
        \lesssim \inf_{\epsilon\in(0,1]} \set{ \epsilon T + (\dimo(\epsilon;\Fcal,\Ucal)+1)\paren{\frac{d\log T}{\epsilon }   +  \dimo(\epsilon;\Fcal,\Ucal) }   }.
    \end{multline*}
\end{theorem}

\begin{proof}
    Both upper bounds are obtained by optimally tuning the parameter $\epsilon>0$ within \cref{lemma:finite_litt_dimension_necessary,lemma:finite_litt_dimension_necessary_realizable,thm:oblivious_regret_realizable,thm:oblivious_regret_agnostic}.
\end{proof}
Note that we can bound $\dimo(\epsilon;\Fcal,\Ucal)$ in terms of $\dim(\epsilon/4;\Fcal,\Ucal)$ and $N_{\epsilon/24}$ as detailed in \cref{lemma:comparing_various_dimensions}.

\subsection{An Example Beyond VC Function Classes}
\label{subsec:counterexample_VC_class}

In the previous sections, we showed that the two conditions
\begin{equation}\label{eq:two_conditions_oblivious}
    \forall \epsilon>0,\quad N_\epsilon<\infty\quad \text{and}\quad  \dim(\epsilon;\Fcal,\Ucal)<\infty
\end{equation}
are necessary for learning with oblivious adversaries (\cref{prop:necessary_condition}) and they are also sufficient if $\Fcal$ additionally has finite VC dimension (\cref{thm:oblivious_adversaries}). These conditions \cref{eq:two_conditions_oblivious} are however not sufficient in general and we give such an example in this section. All statements in this section are specifically for the function class and distribution class defined below.

Let $\Xcal:=\bigsqcup_{d\geq 1} S_{d-1}$, where $S_{d-1}=\{x\in\Rbb^d:\|x\|=1\}$ is the unit sphere and $\bigsqcup$ denotes a formal disjoint union. For convenience, we write $\Xcal_d:= S_{d-1}$. For instance, $(1)\in\Xcal_1$ and $(1,0)\in\Xcal_2$ are distinct points in $\Xcal$. We consider the function class $\Fcal$ obtained as the collection of all hyperplanes:
\begin{equation*}
    \Fcal :=\bigsqcup_{d\geq 1} \set{x\in\Xcal\mapsto \1[x\in \Xcal_d] \1[v^\top x=0], v\in S_{d-1}}.
\end{equation*}
Next, we consider the distribution class $\Ucal$ containing all uniform distributions on any such hyperplane:
\begin{equation*}
    \Ucal:= \bigsqcup_{d\geq 1} \set{ \Unif(\Xcal_d\cap\{x:v^\top x=0\}), v\in S_{d-1}}.
\end{equation*}
By construction, for any function $f= \1[\cdot\in \Xcal_d] \1[v^\top \cdot =0]\in\Fcal$  with $v\in S_{d-1}$, and any $\mu\in\Ucal$,
\begin{equation}\label{eq:property_counterexample}
    \mu(f=1) =\begin{cases}
        1 & \text{if } \mu = \Unif(\Xcal_d\cap\{x:v^\top x=0\})  \\
        0 & \text{otherwise}.
    \end{cases}
\end{equation}
Using this remark we can then show that $(\Fcal,\Ucal)$ satisfies \cref{eq:two_conditions_oblivious}.

\begin{lemma}\label{lemma:counter_example_has_low_dims}
    For the function class and distribution class $(\Fcal,\Ucal)$ constructed above, for any $\epsilon\in[0,1]$ and $\mu\in\Ucal$,
    $\Ncov_{\mu}(\epsilon;\Fcal)=2$. Also, for $\epsilon>0$,
    \begin{equation*}
        N_\epsilon \leq \frac{1}{\epsilon}+1 \quad \text{and} \quad \dim(\epsilon;\Fcal,\Ucal)=1.
    \end{equation*}
\end{lemma}

\begin{proof}
    \cref{eq:property_counterexample} directly implies that $\Ncov_{\mu}(0;\Fcal)=2$ for any $\mu\in\Ucal$. Next, fix $\epsilon>0$ and a mixture $\mu\in\bar\Ucal$. Let $\Scal_\mu=\set{f: \mu(f=1)>\epsilon}$ be the set of all hyperplanes that have at least $\epsilon$ mass under $\mu$. By \cref{eq:property_counterexample}, we have
    \begin{equation*}
        1\geq \Pbb_{x\sim\mu}(\exists f\in \Scal_\mu: f(x)=1) = \sum_{f\in\Scal_\mu} \Pbb_{x\sim\mu}(f(x)=1) \geq \epsilon |\Scal_\mu|.
    \end{equation*}
    Hence, adding a function $f_0\in\Fcal$ for which $\mu(f_0=1)=0$ to the set $\Scal_\mu$, we obtained an $\epsilon$-cover of $\Fcal$ with respect to $\mu$ that has cardinality at most $1+1/\epsilon$. This proves that $N_\epsilon\leq 1+1/\epsilon$.

    Next, we consider an $\epsilon$-shattered tree for $(\Fcal,\Ucal)$ of depth $d\geq 1$. By definition, we have $\mu_r(f_{r,0}\neq f_{r,1})\geq \epsilon$ where $r$ denotes the root. Hence, there exists $y\in\{0,1\}$ such that $\mu_r(f_{r,y}=1)\geq \epsilon/2$. Suppose without loss of generality that $y=0$. Next, using the second property of shattered trees, any left-descendant function $f$ of the root $r$ must satisfy $\mu_r(f_{r,0}\neq f)\leq \epsilon/3$. We combine these properties to obtain
    \begin{equation*}
        \mu_r(f_{r,0}=f=1) \geq \mu_r(f_{r,0}=1) - \mu_r(f_{r,0}\neq f) \geq \frac{\epsilon}{2}-\frac{\epsilon}{3} >0.
    \end{equation*}
    In particular, since $\mu_r\in\bar\Ucal$ there must exist $\mu\in\Ucal$ such that $\mu(f_{r,0}=f=1)>0$. From \cref{eq:property_counterexample} we conclude that $f=f_{r,0}$. In summary, all left-descendant functions of the root coincide, which forces $d=1$. This proves $\dim(\epsilon;\Fcal,\Ucal)\leq 1$. We can also easily check that $\dim(\epsilon;\Fcal,\Ucal)\geq 1$.
\end{proof}

While \cref{lemma:counter_example_has_low_dims} shows that $(\Fcal,\Ucal)$ easily satisfy the necessary conditions from \cref{prop:necessary_condition}, it turns out that these are not learnable for oblivious adversaries.

\begin{proposition}\label{prop:counterexample}
    For the function class and distribution class $(\Fcal,\Ucal)$ constructed above, for any $T\geq 1$, $ \OblivR_T(\Fcal,\Ucal)\geq  \OblivRealizableR_T(\Fcal,\Ucal) \geq T/4$.
\end{proposition}

\begin{proof}
    It suffices to focus on the realizable case. Fix $T\geq 1$, a learner strategy $\Lcal$, and let $d\geq T$ to be defined later. The adversaries that we will consider will focus only on $\Xcal_d$. We denote by $\bar\mu$ the rotationally-invariant distribution of $x$ obtained by first sampling $v\sim\Unif(S_{d-1})$, then sampling $x\sim\Unif(\Xcal_d\cap\{z:v^\top z=0\})$. Note that by construction $\bar\mu\in\bar\Ucal$. We consider a first adversary $\Acal^{(0)}$ that always chooses the distribution $\bar\mu$ at all iterations $t\in[T]$ and such that the values $y_t\sim\Ber(\frac{1}{2})$ are i.i.d.\ Bernoulli and independent from any other random variables. Note that this is a randomized oblivious adversary that is not realizable a priori. By design of $(y_t)_{t\in[T]}$,
    \begin{equation}\label{eq:lower_bound_adversary_0}
        \Ebb_{\Lcal,\Acal_0}\sqb{\sum_{t=1}^T \1[\hat y_t\neq y_t]} = \frac{T}{2}.
    \end{equation}

    Next, for each $v\in S_{d-1}$, we define an adversary $\Acal_v$ as follows. For convenience, we denote $\mu_v:=\Unif(\Xcal_d\cap\set{x:v^\top x=0})$ and $f_v:=\1[\cdot\in\Xcal_d]\1[v^\top\cdot=0]\in\Fcal$. For each iteration $t\in[T]$ we sample $B_t\sim\Ber(\frac{1}{2})$ a Bernoulli independent from the history. If $B_t=1$, the adversary  $\Acal_v$ chooses $\mu_t=\mu_v$ and otherwise chooses $\mu_t=\bar\mu$. At all iterations, $\Acal_v$ chooses the function $f_t=f_v$. Hence, $\Acal_v$ is a realizable randomized adversary. By design of $\Acal_v$, we have $(y_t)_{t\in[T]}=(B_t)_{t\in[T]}$ which is an i.i.d.\ sequence of Bernoulli $\Ber(\frac{1}{2})$. We introduce $\Acal^{(1)}$ the adversary which first samples $v\sim\Unif(S_{d-1})$ uniformly, then follows adversary $\Acal_v$.

    We aim to show that the learner cannot distinguish significantly between $\Acal^{(0)}$ and $\Acal^{(1)}$. Denote by $\Dcal^{(i)}$ the distribution of the sequence $(x_t^{(i)},y_t^{(i)})_{t\in[T]}$ obtained by the adversary $\Acal^{(i)}$ for $i\in\{0,1\}$. In the following, by abuse of notation we identify a random variable and its distribution. Next, for any finite sequence of points $S\subseteq \Xcal_d$ with $
    |S|<d$, we denote by $\bar\mu(S)$ the distribution on $\Xcal_d$ obtained by first sampling $v\sim\Unif(S_{d-1}\cap\set{x:\forall z\in S,x^\top z=0})$ then sampling $x\sim\mu_v$. We note that this is well-defined since $|S|<d$ and hence $S_{d-1}\cap\set{v:\forall z\in S,v^\top z=0}\neq\emptyset$.
    By Pinsker's inequality and the chain rule,
    \begin{align}
        2 \TV(\Dcal^{(0)},\Dcal^{(1)})^2 &\leq \Dkl(\Dcal^{(0)} \parallel\Dcal^{(1)}) \notag \\
        &=\sum_{t=1}^T \Dkl\paren{ x_t^{(0)},y_t^{(0)} \parallel x_t^{(1)},y_t^{(1)}  \mid (x_s^{(0)},y_s^{(0)})_{s<t}}\notag \\
        &\overset{(i)}{=} \sum_{t=1}^T \Dkl\paren{ x_t^{(0)}\parallel x_t^{(1)}  \mid (x_s^{(0)},y_s^{(0)})_{s<t}, y_t^{(0)}} \notag \\
        &\overset{(ii)}{=} \frac{1}{2}\sum_{t=1}^T \Dkl\paren{ x_t^{(0)}\parallel x_t^{(1)}  \mid (x_s^{(0)},y_s^{(0)})_{s<t}, y_t^{(0)}=1} \notag \\
        &\overset{(iii)}{=} \frac{1}{2} \sum_{t=1}^T \Ebb_{\substack{(x_s)_{s<t}\overset{iid}{\sim}\bar\mu,\\ (y_s)_{s<t}\overset{iid}{\sim}\Ber(\frac{1}{2})}}\sqb{ \Dkl\paren{\bar\mu\parallel \bar\mu(\set{x_s: y_s=1,s<t})} } \notag \\
        &= \frac{1}{2} \sum_{t=1}^T \Ebb_{N\sim\Binom(t-1,\frac{1}{2}),(x_s)_{s\in[N]}\overset{iid}{\sim}\bar\mu}\sqb{ \Dkl\paren{\bar\mu\parallel \bar\mu(\set{x_s,s\in[N]})} } \label{eq:bounding_TV_distance}
    \end{align}
    In $(i)$ we used the fact that in both distributions, $y_t\sim\Ber(\frac{1}{2})$ independently from the past history. In $(ii)$ we notice that whenever $y_t=0$, almost surely the samples $x_t$ were both sampled from $\bar\mu$ independently from the past history in both distributions. This is always the case for $\Acal^{(0)}$ and for $\Acal^{(1)}$, almost surely if $B_t=0$ we have $y_t=f_v(x_t)=0$ while if $B_t=1$ we always have $y_t=f_v(x_t)=1$. In $(iii)$ we used the definition of $\Acal^{(1)}$ to compute the distribution of $x_t^{(1)}$ conditional on the other variables and $y_t=1$.

    In the following, we fix any $S\subseteq \Xcal_d$ with $|S|<d$ and aim to upper bound $\Dkl(\bar\mu\parallel \bar\mu(S))$. Let $E=\Span(S)$ be the linear span of $S$. Then, $\Unif(S_{d-1}\cap\{v:\forall z\in S,v^\top z=0\}) = \Unif(S_{d-1}\cap E^\perp)$. Next, since $\bar\mu$ is rotationally invariant, it is simply the uniform distribution on $S_{d-1}$. On the other hand, by construction of $\bar\mu(S)$, it is invariant by any rotation preserving $E$. Hence, the density of $\bar\mu(S)$ at $x\in S_{d-1}$ only depends on its projection $\|\Proj_E(x)\|^2$ with $E$. We denote by $P_E^{(0)}$ (resp. $P_{E}^{(1)}$) the distribution of $\|\Proj_E(x)\|^2$ for $x\sim\bar\mu$ (resp. $x\sim\bar\mu(S)$). Then, since both $\bar\mu$ and $\bar\mu(S)$ are invariant by any rotation preserving $E$, we have
    \begin{equation*}
        \Dkl(\bar\mu\parallel\bar\mu(S)) = \Dkl(P_{E}^{(0)}\parallel P_{E}^{(1)}).
    \end{equation*}
    Next, note that for any fixed $v_0\in S_{d-1}\cap E^\perp$, the distribution of $\|\Proj_E(x)\|^2$ for $x\sim\Unif(S_{d-1}\cap\{z:v_0^\top z=0\})$ is identical. Denote by $\widetilde P$ this distribution. Recalling that $P_E^{(1)}$ is obtained by taking the mixture of these distributions for $v_0\sim\Unif(S_{d-1}\cap E^\perp)$, this shows that $P_{E}^{(1)}$ has the same distribution as $\widetilde P$. Recall that $E\subseteq \set{z:v_0^\top z=0}$. Hence, $\widetilde P$ and $P_{E}^{(1)}$ have the same distribution as $\|\Proj_F(x)\|^2$ for $x\sim\Unif(S_{d-2})$ and $F$ a $\dim(E)$-dimensional linear subspace of $\Rbb^{d-1}$. We note that $k:=\dim(E)\leq T-1 \leq d-1$. In summary, $P_E^{(1)}=\Beta\paren{\frac{k}{2},\frac{d-1-k}{2}}$. In comparison, the distribution $P_E^{(0)}$ of $\|\Proj_E(x)\|$ for $x\sim \bar\mu$ is exactly $\Beta\paren{\frac{k}{2},\frac{d-k}{2}}$. Using known explicit bounds for the KL divergence of two Dirichlet distributions, we have for $k\leq T<d-1$,
    \begin{equation*}
        \Dkl(P_{E}^{(0)}\parallel P_{E}^{(1)}) \leq \ln \frac{\Gamma(\frac{d}{2})\Gamma(\frac{d-k-1}{2})}{\Gamma(\frac{d-1}{2})\Gamma(\frac{d-k}{2})} =\frac{c}{d-k-1} \leq \frac{c}{d-T},
    \end{equation*}
    for some universal constant $c>0$.
    Putting everything together within \cref{eq:bounding_TV_distance}, we proved that for $d>T$, 
    \begin{equation*}
        \TV(\Dcal^{(0)},\Dcal^{(1)}) \leq \sqrt{ \frac{cT}{2(d-T)}}.
    \end{equation*}
    Hence, taking $d\geq 2cT^2+T$, \cref{eq:lower_bound_adversary_0} implies
    \begin{align*}
        \Ebb_{v\sim\Unif(S_{d-1})} \Ebb_{\Lcal,\Acal_v}\sqb{\sum_{t=1}^T\1[\hat y_t\neq y_t]} \geq \Ebb_{\Lcal,\Acal_0}\sqb{\sum_{t=1}^T\1[\hat y_t\neq y_t]} - T\cdot \TV(\Dcal^{(0)},\Dcal^{(1)}) \geq  \frac{T}{4}.
    \end{align*}
    As a result, there must exist $v\in S_{d-1}$ such that $\Ebb_{\Lcal,\Acal_v}\sqb{\sum_{t=1}^T\1[\hat y_t\neq y_t]} \geq \frac{T}{4}$. We recall that $\Acal_v$ is a randomized adversary since at each iteration, with probability half it selects $\mu_t=\bar\mu$, which is equivalent to selecting $\mu_t=\mu_{v'}$ for $v'\sim\Unif(S_{d-1})$. By the law of total probability there must exist a sequence $v_1,\ldots,v_T$ such that the adversary $\Acal^\star$ which chooses $\mu_t=\mu_{v_t}$ and $f_t=f_v$ has
    \begin{equation*}
        \Ebb_{\Lcal,\Acal^\star}\sqb{\sum_{t=1}^T\1[\hat y_t\neq y_t]} \geq \frac{T}{4}.
    \end{equation*}
    This holds for any learner $\Lcal$, which ends the proof that $\frac{1}{T}\OblivRealizableR_T(\Fcal,\Ucal)\geq 1/4$.
\end{proof}
In particular, \cref{prop:counter_example_infinite_dim} is obtained by combining \cref{lemma:counter_example_has_low_dims,prop:counterexample}.

While \cref{prop:counter_example_infinite_dim} shows that the conditions from \cref{prop:necessary_condition} are not sufficient for learning with oblivious adversaries in general, for alternative learning models this may be the case, as discussed in the following remark.

\begin{remark}
    The two conditions $\sup_{\mu\in\Ucal} \Ncov_{\mu}(\epsilon;\Fcal)<\infty$ and $\dim(\epsilon;\Fcal,\Ucal)<\infty$ for all $\epsilon>0$ from \cref{prop:necessary_condition} would be necessary and sufficient for learning for general function and distribution classes $\Fcal,\Ucal$ if at the end of each iteration $t\in[T]$, in addition to the value $y_t$, the learner additionally receives as feedback the distribution $\mu_t$ that the adversary used on that iteration (within \cref{item:feedback_learner} of the online learning setup introduced in \cref{subsec:formal_setup}).
\end{remark}

\section{Proofs for Adaptive Adversaries}
\label{sec:proofs_adaptive}

We first check that the sets $\Lcal_k(\epsilon)$ are non-increasing in $k\geq 0$.

\begin{lemma}\label{lemma:decreasing_level_sets}
    Let $\Fcal,\Ucal$ be a function class and distribution class on $\Xcal$, and $\epsilon>0$. The sets $(\Lcal_k(\epsilon))_{k\geq 0}$ are non-increasing.
\end{lemma}

\begin{proof}
    This is a simple induction. Suppose that $\Lcal_k(\epsilon) \subset \Lcal_{k-1}(\epsilon)$ for $k\geq 0$ (this is vacuous for $k=0$). Then, for any $F\subset\Fcal$, we directly have $B_{k+1}(F;\epsilon) \subseteq B_k(F;\epsilon)$. Hence, $\Lcal_{k+1}(\epsilon)\subseteq \Lcal_k(\epsilon)$.
\end{proof}

\subsection{Lower Bounds on the Adaptive Regret}
\label{subsec:lower_bounds_adaptive_regret}

In this section, we prove that the condition from \cref{thm:qualitative_charact} is necessary for learning with adaptive adversaries.

\begin{lemma}\label{lemma:lower_bound_adaptive}
    Let $\Fcal,\Ucal$ be a function class and distribution class on $\Xcal$. Then, for any $\epsilon\in(0,1]$,
    \begin{equation*}
        \AdaptR_T(\Fcal,\Ucal) \geq \frac{\min(k(\epsilon),\epsilon T)}{8}.
    \end{equation*}
\end{lemma}

\begin{proof}
    We fix a horizon $T\geq 1$, a learner strategy $\Lcal$, and $\epsilon\in(0,1]$. 
    We define an adaptive adversary as follows. We start by initializing a empty dataset $D_1=\emptyset$. By construction, we will always have $D_t\in\Lcal_0(\epsilon)$ for any $t\in[T]$. At the beginning of any iteration $t\in[k]$, we define $k_t=\max\set{k\leq T: D_t\in \Lcal_k(\epsilon)}$. If $k_t=0$ the adversary chooses an arbitrary distribution $\mu_t\in\Ucal$. Otherwise, the adversary chooses $\mu_t\in\Ucal$ such that
    \begin{equation*}
        \mu_t(B_{k_t}(D_t;\epsilon)) \geq \frac{\epsilon}{2}.
    \end{equation*}
    The existence of such a distribution is guaranteed by the definition of $\Lcal_{k_t}(\epsilon)$. The adversary then chooses the constant function $f_t=B_t$, where $B_t\sim\Ber(\frac{1}{2})$ is an independent sample from the history. At the end of the round, we pose
    \begin{equation*}
        D_{t+1} = \begin{cases}
            D_t \cup \{(x_t,y_t)\} &\text{if } k_t>0\text{ and }x_t\in B_{k_t}(D_t;\epsilon)\\
            D_t &\text{otherwise}.
        \end{cases}
    \end{equation*}
    We denote by $\Acal$ this adaptive adversary.

    Importantly, for any $t\in[T]$, if $k_t>0$ and $x_t\in B_{k_t}(D_t;\epsilon)$, we always have $D_{t+1}\in \Lcal_{k_t-1}(\epsilon)$. Hence, this implies $k_{t+1}\geq k_t-1$. When $k_t=0$ or $x_t\in B_{k_t}(D_t;\epsilon)$ we simply have $D_{t+1}=D_t$ hence $k_{t+1}=0$. As a result in all cases we have
    \begin{equation}\label{eq:k_does_not_decrease_too_fast}
        k_{t+1} \geq (k_t \1[x_t\notin B_{k_t}(D_t;\epsilon)]  +  (k_t-1) \1[x_t\in B_{k_t}(D_t;\epsilon)])\cdot \1[k_t>0].
    \end{equation}
    In particular, as claimed earlier, we always have $k_t\geq 0$ for $t\in[T]$.
    Next, note that the true value $y_t=B_t$ is always an independent Bernoulli, hence denoting by $\Hcal_t$ the history available at the beginning of that round, 
    \begin{equation}\label{eq:mistakes_learner}
        \Ebb[\hat y_t\neq y_t \mid \Hcal_t] \geq \frac{1}{2}.
    \end{equation}
    Note that the sequence of sets $(D_t)_{t\geq 1}$ are only dependent on the $(x_t)_{t\in[T]}$ as well as the labels $y_t=B_t$ for iterations $t$ when $k_t>0$ and $x_t\in B_{k_t}(D_t;\epsilon)$. In particular, $D_T$ is independent from the values $y_t=B_t$ when $k_t=0$ or $x_t\notin B_{k_t}(D_t;\epsilon)$. Hence, since $k_T\geq 0$, we can fix a function $f^\star$ that realizes the dataset $D_T$ (that is $f^\star(x)=y$ for all $(x,y)\in D_T$) and that is independent of these variables also. Putting everything together we have
    \begin{align}
        \AdaptReg_T(\Lcal;\Acal) &\geq \Ebb_{\Lcal,\Acal}\sqb{\sum_{t=1}^T\1[\hat y_t\neq y_t] - \sum_{t=1}^T\1[f^\star(x_t)\neq y_t]} \notag\\
        &\overset{(i)}{\geq} \Ebb_{\Lcal,\Acal}\sqb{\frac{T}{2} - \sum_{t=1}^T\1[f^\star(x_t)\neq y_t] \1[k_t=0 \text{ or }x_t\notin B_{k_t}(D_t;\epsilon)]} \notag\\
        &\overset{(ii)}{\geq} \frac{1}{2}\cdot \Ebb_{\Lcal,\Acal}\sqb{\sum_{t=1}^T\1[k_t>0 ]\1[x_t\in B_{k_t}(D_t;\epsilon)]} \label{eq:lower_bound_reg_sum_good_times}
    \end{align}
    In $(i)$ we used \cref{eq:mistakes_learner} and the fact that $f^\star$ realizes $D_T$. In $(ii)$ we used the fact that $y_t$ is independent from $f^\star$, $x_t$ conditional on $k_t=0$ or $x_t\notin B_{k_t}(D_t;\epsilon)$.
    Next, by construction of $\mu_t$ when $k_t>0$, we have 
    \begin{equation*}
        \Pbb(x_t\in B_{k_t}(D_t;\epsilon)\mid\mu_t,k_t>0) = \mu_t(B_{k_t}(D_t;\epsilon)) \geq \frac{\epsilon}{2}.
    \end{equation*}
    Hence, furthering the bound from \cref{eq:lower_bound_reg_sum_good_times} gives
    \begin{align*}
        \AdaptReg_T(\Lcal;\Acal) \geq \frac{\epsilon}{4}\cdot \Ebb_{\Lcal,\Acal}\sqb{\sum_{t=1}^T\1[k_t>0 ]}=\frac{\epsilon}{4} \cdot\sum_{t=1}^T\Pbb(k_t>0) \geq \frac{\epsilon T}{4} \cdot\Pbb(k_T>0).
    \end{align*}
    In the last inequality we used the fact that $k_t$ is non-increasing since $D_{t}\subseteq D_{t+1}$ for any $t\in[T-1]$. On the other hand, from  \cref{eq:k_does_not_decrease_too_fast}, $k_{t+1}=k_t$ unless $k_t>0$ and $x_t\in B_{k_t}(D_t;\epsilon)$ in which case we have $k_{t+1}\geq k_t-1$. Therefore, \cref{eq:lower_bound_reg_sum_good_times} implies
    \begin{equation*}
        \AdaptReg_T(\Lcal;\Acal) \geq \frac{k_1}{2}\cdot \Pbb(k_T=0).
    \end{equation*}
    Since $\Pbb(k_T=0)+\Pbb(k_T>0)=1$, 
    combining the two previous result shows that
    \begin{equation*}
        \AdaptReg_T(\Lcal;\Acal) \geq \min\paren{\frac{k_1}{4},\frac{\epsilon T}{8}} \geq \frac{\min(k_1,\epsilon T)}{8}
    \end{equation*}
    By definition $k_1=\max\set{k\leq T: \emptyset\in \Lcal_k(\epsilon)}=\min(k(\epsilon),T)$. As a result, $\min(k_1,\epsilon T)=\min(k(\epsilon),\epsilon T)$, which ends the proof.
\end{proof}

We next turn to the realizable case for which we will need the following lemma.

\begin{lemma}\label{lemma:realizable_adaptive_choice}
    Let $\Fcal,\Ucal$ be a function class and distribution class on $\Xcal$. Fix any $\epsilon>0$, $k\geq 0$, $x\in\Xcal$. For any dataset $D\in \Lcal_k(\epsilon)$, there exists $y\in\{0,1\}$ such that $D\cup\{(x,y)\}\in \Lcal_{\floor{k/2}}(\epsilon)$.
\end{lemma}

\begin{proof}
    When needed, to avoid confusions, we may specify the function class within the functionals $\Lcal_k$ or $B_k$ for $k\geq 0$. We also denote $\Fcal(D)=\set{f\in\Fcal:\forall(x,y)\in D, f(x)=y}$ for any dataset $D\subseteq\Xcal\times\{0,1\}$. Note that by construction of the sets $\Lcal_k(\epsilon)$ for $k\geq 0$, we have
    \begin{equation}\label{eq:useful_equivalence_class}
        D\in \Lcal_k(\epsilon;\Fcal) \Longleftrightarrow \Lcal_k(\epsilon;\Fcal(D))\neq\emptyset.
    \end{equation}
    Hence, up to renaming $\Fcal$ we can assume without loss of generality that $D=\emptyset$. We then fix $\epsilon>0$, $k\geq 0$, and $x\in\Xcal$. Suppose that $\{(x,1)\}\notin \Lcal_{\floor{k/2}}(\epsilon)$, which is equivalent to $\Lcal_{\floor{k/2}}(\epsilon;\Fcal(\{(x,1)\}))$. For conciseness, we write $\Fcal_x^y$ for $\Fcal(\{(x,y)\})$ for $y\in\{0,1\}$. We prove by induction that for any $l\geq 0$,
    \begin{equation}\label{eq:induction_hypothesis}
        \Lcal_{\floor{k/2}+l}(\epsilon;\Fcal)\subseteq \Lcal_l(\epsilon;\Fcal_x^0).
    \end{equation}

    We start with the case $l=0$. Let $D\in \Lcal_{\floor{k/2}}(\epsilon;\Fcal)$ and suppose by contradiction that $D\notin \Lcal_0(\epsilon;\Fcal_x^0)$. This implies that any function $f$ that realizes $D$ must have $f\in\Fcal_x^1$ since $\Fcal=\Fcal_x^0\cup\Fcal_x^1$. Therefore $\Fcal(D)=\Fcal(D\cup\{(x,1)\})$. Using \cref{eq:useful_equivalence_class} this implies
    \begin{equation*}
        \emptyset \overset{(i)}{\neq} \Lcal_{\floor{k/2}}(\epsilon;\Fcal(D)) = \Lcal_{\floor{k/2}}(\epsilon;\Fcal(D\cup\{(x,1)\})) \overset{(ii)}{\subseteq} \Lcal_{\floor{k/2}}(\epsilon;\Fcal_x^1).
    \end{equation*}
    In $(i)$ we used $D\in \Lcal_{\floor{k/2}}(\epsilon;\Fcal)$ and in $(ii)$ we noted that $\Lcal_q(\epsilon;\Fcal)$ is non-decreasing with $\Fcal$. Then, \cref{eq:useful_equivalence_class} implies $\{(x,1)\} \in \Lcal_{\floor{k/2}}(\epsilon;\Fcal)$ contradicting our assumption. In summary, we proved \cref{eq:induction_hypothesis} for $l=0$.
    
    Now suppose \cref{eq:induction_hypothesis} holds for $l\geq 0$. Then, for any $D\in \Lcal_{\floor{k/2}+l+1}(\epsilon;\Fcal)$, this directly implies $B_{\floor{k/2}+l+1}(D;\epsilon,\Fcal) \subseteq B_{l+1}(D;\epsilon,\Fcal_x^0)$.
    As a result,
    \begin{equation*}
        \sup_{\mu\in\Ucal} \mu(B_{l+1}(D;\epsilon,\Fcal_x^0)) \geq \sup_{\mu\in\Ucal} \mu(B_{\floor{k/2}+l+1}(D;\epsilon,\Fcal)) \geq \epsilon.
    \end{equation*}
    This implies $D\in \Lcal_{l+1}(\epsilon;\Fcal_x^0)$ completing the induction for \cref{eq:induction_hypothesis}.

    As a result, we obtain $\emptyset\neq \Lcal_k(\epsilon;\Fcal)\subseteq \Lcal_{k-\floor{k/2}}(\epsilon;\Fcal_x^0)$. From \cref{eq:useful_equivalence_class} this precisely shows that $\{(x,0)\}\in \Lcal_{k-\floor{k/2}}(\epsilon;\Fcal)\subseteq \Lcal_{\floor{k/2}}(\epsilon;\Fcal)$, where in the last inequality we used the monotonicity property from \cref{lemma:decreasing_level_sets}. This ends the proof.
\end{proof}

We are now ready to prove an adaptive regret bound for the realizable setting.

\begin{lemma}\label{lemma:lower_bound_adaptive_realizable}
    Let $\Fcal,\Ucal$ be a function class and distribution class on $\Xcal$. Then, for any $\epsilon\in(0,1]$,
    \begin{equation*}
        \AdaptRealizableR_T(\Fcal,\Ucal) \geq \frac{\epsilon\cdot \min(\log_2 k(\epsilon),T)}{4}.
    \end{equation*}
\end{lemma}

\begin{proof}
    Fix $T\geq 1$, a learner strategy $\Lcal$ and $\epsilon\in(0,1]$. We use a similar adaptive adversary as defined in the proof of \cref{lemma:lower_bound_adaptive}, except that the dataset always contains all previous samples:
    \begin{equation*}
        D_t:=\{(x_s,y_s),s<t\}, \quad t\in[T]
    \end{equation*}
    We will show that by design $D_t\in\Lcal_0(\epsilon)$ for all $t\in[T]$. At iteration $t\in[k]$ we define the same quantity $k_t=\max\{k\leq 2^T:D_t\in\Lcal_k(\epsilon)\}$. The adversary selects $\mu_t$ as in \cref{lemma:lower_bound_adaptive}: if $k_t=0$ this choice is arbitrary, otherwise $\mu_t$ is chosen so that $\mu(B_{k_t}(D_t;\epsilon))\geq \epsilon/2$. The only difference is in the choice of the function. Let $B_t\sim\Ber(\frac{1}{2})$ a Bernoulli sample independent from the history. We pose
    \begin{equation*}
        f_t:x\in\Xcal\mapsto\begin{cases}
            B_t & \text{if }k_t>0\text{ and }x_t\in B_{k_t}(D_t;\epsilon)\\
            \argmax_{y\in\{0,1\}} \max\set{k\leq 2^T: D_t\cup\{(x,y)\}\in\Lcal_k(\epsilon)} &\text{otherwise},
        \end{cases}
    \end{equation*}
    where ties can be broken arbitrarily in the second scenario. We denote by $\Acal^r$ this adaptive adversary.

    For convenience, we define $D_{T+1}$ and $k_{T+1}$ as above.
    As in \cref{lemma:lower_bound_adaptive}, for any $t\in[T]$ if $k_t>0$ and $x_t\in B_{k_t}(D_t;\epsilon)$, then $k_{t+1}\geq k_t-1$. The second case when $k_t=0$ or $x_t\notin B_{k_t}(D_t;\epsilon)$ is different. By construction of $f_t(x_t)$ in this case, since $D_{t+1}=D_t\cup\{(x_t,y_t)\}$ and $D_t\in\Lcal_{k_t}(\epsilon)$, \cref{lemma:realizable_adaptive_choice} implies that $D_{t+1}\in\Lcal_{\floor{k_t/2}}(\epsilon)$. Hence, $k_{t+1}\geq \floor{k_t/2}$. In particular, this also confirms the claim that $k_t\geq 0$ for all $t\in[T+1]$. In summary, in all cases,
    \begin{equation}\label{eq:at_worst_halved_k}
        k_{t+1} \geq \floor{\frac{k_t}{2}},\quad t\in[T].
    \end{equation}
    Also, we showed that $k_{T+1}\geq 0$. This implies that $D_{T+1}=\set{(x_t,y_t),t\in[T]}\in\Lcal_0(\epsilon)$ is realizable in $\Fcal$. Therefore, $\Acal^r$ is realizable, and we obtain
    \begin{align*}
        \AdaptReg_T(\Lcal;\Acal^r) &= \Ebb_{\Lcal,\Acal^r}\sqb{\sum_{t=1}^T \1[\hat y_t\neq f_t(x_t)]}\\
        &\geq \Ebb_{\Lcal,\Acal^r}\sqb{\sum_{t=1}^T \1[k_t>0]\1[x_t\in B_{k_t}(D_t;\epsilon)] \1[\hat y_t\neq B_t]}\\
        &\overset{(i)}{=} \frac{1}{2}\cdot \Ebb_{\Lcal,\Acal^r}\sqb{\sum_{t=1}^T \1[k_t>0]\1[x_t\in B_{k_t}(D_t;\epsilon)] }
    \end{align*}
    In $(i)$ we used the fact that $B_t\sim\Ber(\frac{1}{2})$ is sampled independently from the previous history. We recall that by construction, for any $t\in[T]$ such that $k_t>0$, $\mu_t(B_{k_t}(D_t;\epsilon))\geq \epsilon/2$. Therefore, furthering the previous bound we obtain
    \begin{align*}
        \AdaptReg_T(\Lcal;\Acal^r) &\geq \frac{\epsilon}{4}\cdot \Ebb_{\Lcal,\Acal^r}\sqb{\sum_{t=1}^T \1[k_t>0] } \geq \frac{\epsilon}{4}\cdot\min(\floor{\log_2(k_1)}+1,T),
    \end{align*}
    where in the last inequality we used \cref{eq:at_worst_halved_k} which implies $k_t\geq \floor{k_1/2^{t-1}}$. Since $k_1=\min(k(\epsilon),2^T)$, we have $\min(\floor{\log_2(k_1)}+1,T) \geq \min(\log_2 k(\epsilon) ,T)$, which ends the proof. 
\end{proof}

\subsection{Regret Bounds for the Realizable Setting}

In the previous section, we showed in particular that the condition $k(\epsilon)<\infty$ for all $\epsilon>0$ is necessary for learning against adaptive adversaries. We now show that it is also sufficient for learning and focus on the realizable setting in this section.

We start by defining an algorithm designed to achieve roughly $ \epsilon$ average regret for some fixed $\epsilon>0$. For convenience, for any function class $\Fcal$, we define
\begin{equation*}
    k(\epsilon;\Fcal):=\sup\set{k\geq 0: \Lcal_k(\epsilon; \Fcal)\neq\emptyset },
\end{equation*}
with the convention $k(\epsilon;\emptyset)=-\infty$.
We consider the algorithm which is essentially the same as for adversarial learning with finite Littlestone classes \cite{littlestone1988learning} but replacing the notion of Littlestone dimension with $k(\epsilon;\cdot)$. The algorithm is detailed in \cref{alg:adaptive_realizable}.

\begin{algorithm}[t]

    \caption{Algorithm for $\Ocal(\epsilon)$ average adaptive regret in the realizable setting}\label{alg:adaptive_realizable}
    
    \LinesNumbered
    \everypar={\nl}
    
    \hrule height\algoheightrule\kern3pt\relax
    \KwIn{Function class $\Fcal$, distribution class $\Ucal$, horizon $T$, tolerance $\epsilon>0$}
    
    \vspace{3mm}

    \For{$k\geq 1$}{
        Observe $x_t$ and define $\Fcal_t(y):=\set{f\in\Fcal, \forall s<t,f(x_s)=y_s,\text{ and } f(x_t)=y}$
        
        Predict $\hat y_t:= \argmax_{y\in\{0,1\}} k(\epsilon;\Fcal_t(y))$ breaking ties arbitrarily, then observe $y_t$
    }
    
    \hrule height\algoheightrule\kern3pt\relax
    \end{algorithm}

We then prove the following adaptve regret bound on \cref{alg:adaptive_realizable} which is an analog of the regret bounds in the fully adversarial case \cite{littlestone1988learning}.

\begin{lemma}\label{lemma:adaptive_regret_realizable}
    Let $\Fcal,\Ucal$ be a function class and distribution class on $\Xcal$. For any $\epsilon>0$, $T\geq 1$, and any adaptive and realizable adversary $\Acal$, \cref{alg:adaptive_realizable} run with the tolerance parameter $\epsilon$ satisfies:
    \begin{equation*}
        \AdaptReg_T(\cref{alg:adaptive_realizable}(\epsilon);\Acal) \leq \epsilon T + \min(k(\epsilon),T)
    \end{equation*}
\end{lemma}

\begin{proof}
    Fix $T\geq 1$ an $\epsilon>0$. Clearly, if $k(\epsilon)\geq T$, we result is immediate. We suppose this is not the case from now.
    Denote by $\mu_1,\ldots,\mu_T$ and $f_1,\ldots,f_T$ the distributions and functions chosen by an adaptive realizable adversary $\Acal$. For any $t\in[T]$, we define $\Fcal_t:=\set{f:\forall s<t,f(x_s)=y_s}$ the set of functions realizing the previous observed samples. We also use the notation $\Fcal_t(y)$ for $y\in\{0,1\}$ as defined within \cref{alg:adaptive_realizable}. Last, we denote by $\Hcal_t$ the history available at the beginning of iteration $t$. We also denote by $\Ecal=\{\exists f\in\Fcal: \forall t\in[T],f(x_t)=y_t\}$ the full probability event that the adversary samples are realizable.

    For any $t\in[T]$, since $\Fcal_t\subseteq\Fcal$ we have $k_t:=k(\epsilon;\Fcal_t)\leq k(\epsilon;\Fcal)\leq T$. Also, by definition we have $\Lcal_{k_t+1}(\epsilon;\Fcal_t)=\emptyset$. From \cref{eq:useful_equivalence_class} this is equivalent to $D\notin \Lcal_{k_t+1}(\epsilon;\Fcal)$, which implies
    \begin{equation*}
        \mu_t\paren{B_{k_t+1}(D;\epsilon)}  <\epsilon.
    \end{equation*}
    We then define the event
    \begin{equation*}
        \Fcal_t= \{D_t\cup\{(x_t,0)\} \notin \Lcal_{k_t}(\epsilon;\Fcal)\}\cup \{ D_t\cup\{(x_t,1)\}\notin\Lcal_{k_t}(\epsilon;\Fcal) \}.
    \end{equation*}
    By the previous equation, we have
    $\Pbb(\Fcal_t^c\mid \Hcal_t) < \epsilon$. Next, under $\Fcal_t$ we have $\min_{y\in\{0,1\}} k(\epsilon;\Fcal_t(y)) <k_t$.
    In particular, if \cref{alg:adaptive_realizable} made a mistake $\hat y_t\neq y_t$ at time $t$, then $k_{t+1}\leq k_t-1$. We also recall that we always have $k_{t+1}\leq k_t$. On the other hand, under $\Ecal$ the samples are realizable, hence $k_{T+1}\geq 0$. Therefore,
    \begin{align*}
        \AdaptReg_T(\cref{alg:adaptive_realizable};\Acal) = \Ebb\sqb{\sum_{t=1}^T\1[\hat y_t\neq y_t]}
        &\leq \Ebb\sqb{\sum_{t=1}^T\1[\Fcal_t^c]} + \Ebb\sqb{\sum_{t=1}^T\1[\Fcal_t]\1[\hat y_t\neq y_t]}\\
        &< \epsilon T + \Ebb\sqb{k_1-k_{T+1}} \\
        &= \epsilon T + \Ebb\sqb{\1[\Ecal](k(\epsilon;\Fcal)-k_{T+1})} \leq \epsilon T + k(\epsilon;\Fcal).
    \end{align*}
    This ends the proof.
\end{proof}

\subsection{Regret Bounds for the Agnostic Setting}

We next turn to the agnostic setting where again we can use an algorithm in the same spirit as those for adversarial learning with finite Littlestone function classes \cite{ben2009agnostic}.

For a given tolerance parameter $\epsilon>0$, we construct a set of experts $\Scal_E$. For any subset of times $S\subseteq[T]$ with $|S|\leq \min(k(\epsilon),T)$, we define a corresponding expert $E(S)$ which uses $S$ as a set of mistake times as follows. At the beginning of each iteration $t\in[T]$ we iteratively construct a sequence of datasets $D_t(S)$. This is initialized by $D_1(S)=\emptyset$. At any iteration, we let $k_t(S):=\sup\set{k: D_t(S)\in\Lcal_k(\epsilon)}$ with the convention $\sup\emptyset=-\infty$. If $k_t(S)=-\infty$ or $x_t\in B_{k_t+1}(D_t;\epsilon)$ the expert makes any arbitrary prediction $\hat y_t(S)$. Otherwise, as in \cref{alg:adaptive_realizable}, the expert predicts
\begin{equation*}
    \hat y_t(S) = \argmax_{y\in\{0,1\}} \sup\set{k: D_t(S)\cup\{(x_t,y)\}\in\Lcal_k(\epsilon)}.
\end{equation*}
Next, we update the dataset as follows
\begin{equation*}
    D_{t+1}(S) = \begin{cases}
        D_t(S) &\text{if } k_t=-\infty\text{ or } x_t\in B_{k_t+1}(D_t(S);\epsilon) \text{ or }t\notin S\\
        D_t(S)\cup\{(x_t,1-\hat y_t(S) )\} &\text{otherwise}.
    \end{cases}
\end{equation*}

\begin{algorithm}[t]

    \caption{Expert $E(S)$ for $\Ocal(\epsilon)$ average adaptive regret $\epsilon$ in the agnostic setting}\label{alg:adaptive_agnostic}
    
    \LinesNumbered
    \everypar={\nl}
    
    \hrule height\algoheightrule\kern3pt\relax
    \KwIn{Function class $\Fcal$, distribution class $\Ucal$, horizon $T$, $\epsilon>0$, mistake times $S\subseteq [T]$}
    
    \vspace{3mm}

    Initialization: $D_1(S)=\emptyset$

    \For{$k\geq 1$}{
        Let $k_t(S) = \sup\{k:D_t(S)\in\Lcal_k(\epsilon)\}$ and observe $x_t$

        \lIf{$k_t=-\infty$ or $x_t\in B_{k_t+1}(D_t;\epsilon)$}{
            Predict $\hat y_t(S)=0$ and set $D_{t+1}(S)=D_t(S)$
        }
        \Else{
            Predict $\hat y_t(S) = \argmax_{y\in\{0,1\}} \sup\set{k: D_t(S)\cup\{(x_t,y)\}\in\Lcal_k(\epsilon)}$

            \lIf{$t\in S$}{Set $D_{t+1}(S)=D_t(S)\cup\{(x_t,1-\hat y_t(S))\}$}
            \lElse{Set $D_{t+1}(S)=D_t(S)$}
        }
    }
    
    \hrule height\algoheightrule\kern3pt\relax
    \end{algorithm}

The corresponding expert $E(S)$ is summarized in \cref{alg:adaptive_agnostic}.
The final algorithm then implements the classical Hedge algorithm over all experts $\Scal_E=\set{E(S), S\subseteq[T],|S|\leq \min(k(\epsilon),T)}$, as detailed in \cref{alg:oblivious_agnostic}. We have the following adaptive regret guarantee for this algorithm.

\begin{lemma}\label{lemma:adaptive_regret_agnostic}
    Let $\Fcal,\Ucal$ be a function class and distribution class on $\Xcal$. For any $\epsilon>0$, $T\geq 1$, and any adaptive adversary $\Acal$, denote by $\Lcal(\epsilon)$ the Hedge algorithm (\cref{alg:oblivious_agnostic}) run with all experts from \cref{alg:adaptive_agnostic} for the tolerance parameter $\epsilon$ and sets $S\subseteq[T]$ with $|S|\leq\min(k(\epsilon),T)$. We have
    \begin{equation*}
        \AdaptReg_T(\Lcal(\epsilon);\Acal) \lesssim \epsilon T + \sqrt{\min(k(\epsilon)+1,T)\cdot T\log T} .
    \end{equation*}
\end{lemma}

\begin{proof}
    For convenience, write $k:=\min( k(\epsilon)+1,T)$. By construction there are at most $|\Scal_E|\leq T^{k+1}$ experts. Then, the regret guarantee for the Hedge algorithm \cref{thm:regret_hedge} implies that the event
    \begin{equation*}
        \Ecal:=\set{ \sum_{t=1}^T \1[\hat y_t\neq y_t] -\min_{S\subseteq [T],|S|\leq k} \sum_{t=1}^T \1[\hat y_t(S)\neq y_t] \leq  2\sqrt{Tk\log T}}
    \end{equation*}
    has probability at least $1-1/T$. We now reason conditionally on the complete sequence $(x_t,y_t)_{t\in[T]}$. Let $f^\star\in\Fcal$ be such that
    \begin{equation*}
        \sum_{t=1}^T \1[f^\star(x_t)\neq y_t] = \inf_{f\in\Fcal} \sum_{t=1}^T \1[f(x_t)\neq y_t].
    \end{equation*}
    We construct a corresponding set of times $S\subseteq [T]$ by simulating a run of \cref{alg:adaptive_agnostic} with the current samples. Precisely, we initialize $D_1=S_1:=\emptyset$ and at each iteration $t\in[T]$, we define $k_t:=\sup\set{k:D_t\in\Lcal_k(\epsilon)}$ as for the expert. If $k_t=-\infty$ or $x_t\in B_{k_t+1}(D_t;\epsilon)$ we do not update the dataset $D_{t+1}:=D_t$ not the set $S_{t+1}:=S_t$. Otherwise, letting $\hat y_t$ be the same prediction as made at time $t$ by the expert $E(S_t)$, we set 
    \begin{equation*}
        S_{t+1} = \begin{cases}
            S_t &\text{if }\hat y_t=f^\star(x_t)\\
            S_t\cup\{t\} &\text{otherwise}.
        \end{cases} \quad \text{and} \quad 
        D_{t+1} = \begin{cases}
            D_t &\text{if }\hat y_t=f^\star(x_t)\\
            D_t\cup\{(x_t,f^\star(x_t))\} &\text{otherwise}.
        \end{cases}
    \end{equation*}
    We then denote $S^\star:=S_{T+1}\subseteq[T]$. We aim to prove that $|S^\star|\leq k$. First, note that by construction, at all times $t\in[T]$ the dataset $D_t(S^\star)$ for $t\in[T]$ is realizable by $f^\star$. In particular this implies $k_t(S^\star)\geq 0$ for all $t\in[T+1]$. Next, as in the proof of \cref{lemma:adaptive_regret_realizable} at any time $t\in[T]$ such that $k_t(S^\star)>0$ and $x_t\notin B_{t_k(S^\star)+1}(D_t(S^\star);\epsilon)$, if the expert made a mistake, i.e., if $t\in S^\star$, then $k_{t+1}\leq k_t-1$. Recalling that the sequence $(k_t)_{t\in[T+1]}$ is non-increasing, the previous arguments exactly show that $|S^\star|\leq k_1=k(\epsilon)$. Of course, since $S^\star\subset[T]$ we also have $|S^\star|\leq \min(k(\epsilon),T)\leq k$.

    By construction, at all times $t\in[T]$ such that $k_t(S^\star)\geq 0$ (which always holds by the previous discussion) and $x_t\notin B_{k_t(S^\star)+1}(D_t(S^\star);\epsilon)$, the expert $E(S^\star)$ always predicts $\hat y_t(S^\star)=f^\star(x_t)$. In summary, under $\Ecal$, we obtained
    \begin{align*}
        \sum_{t=1}^T\1[\hat y_t\neq y_t] - \inf_{f\in\Fcal} \sum_{t=1}^T\1[f(x_t)\neq y_t] &\overset{(i)}{\leq} \sum_{t=1}^T\1[\hat y_t(S^\star)\neq y_t] - \sum_{t=1}^T\1[f^\star(x_t)\neq y_t]  + 2\sqrt{Tk\log T} \\
        &\leq  \sum_{t=1}^T\1[\hat y_t(S^\star)\neq f^\star(x_t)] + 2\sqrt{Tk\log T} \\
        &\leq \sum_{t=1}^T \1[x_t\in B_{k_t(S^\star)+1}(D_t(S^\star);\epsilon)] + 2\sqrt{Tk\log T} \\
        &\leq \sup_{S\subseteq [T],|S|\leq k} \sum_{t=1}^T \1[x_t\in B_{k_t(S)+1}(D_t(S);\epsilon)] + 2\sqrt{Tk\log T} .
    \end{align*}
    In $(i)$ we used $\Ecal$, $|S^\star|\leq k$, and the definition of $f^\star$. For any $S\subseteq[T]$ with $|S|\leq k$, the same arguments as in the proof of \cref{lemma:adaptive_regret_realizable} show that
    \begin{equation*}
        \Pbb(x_t\in B_{k_t(S)+1}(D_t(S);\epsilon)\mid\Hcal_t) <\epsilon,
    \end{equation*}
    where $\Hcal_t$ denotes the history available at the beginning of iteration $t$. Therefore, Bernstein's inequality together with the union bound imply that the event
    \begin{equation*}
        \Fcal:=\set{\forall S\subset[T],|S|\leq k: \sup_{S\subseteq [T],|S|\leq k} \sum_{t=1}^T \1[x_t\in B_{k_t(S)+1}(D_t(S);\epsilon)] \leq 2\epsilon T+6k\log T} 
    \end{equation*}
    has probability at least $1-1/T$. Noting that $\Ecal\cap\Fcal$ has probability at least $1-2/T$, we combine the previous estimates to obtain
    \begin{align*}
        \AdaptReg_T(\Lcal(\epsilon);\Acal) \leq 2\epsilon T + 6k\log T + 2\sqrt{Tk\log T} + 2.
    \end{align*}
    Recalling that $k\leq T$ ends the proof.
\end{proof}

Combining together \cref{lemma:lower_bound_adaptive,lemma:lower_bound_adaptive_realizable,lemma:adaptive_regret_realizable,lemma:adaptive_regret_agnostic} proves \cref{thm:qualitative_charact,thm:quantitative_charact}.

\section{Proofs for the Consequences and Instantiations}
\label{sec:proof_consequences}

We start by bounding $k(\epsilon;\Fcal,\{\mu_0\})$ for function classes $\Fcal$ with finite VC dimension.

\vspace{3mm}

\begin{proof}[of \cref{lemma:bound_k_stochastic_setting}]
    We use the probabilistic method together with Sauer-Shelah's lemma \cite{sauer1972density,shelah1972combinatorial}. Fix $\epsilon\in(0,1]$. Consider an i.i.d.\ sequence of samples $(x_t)_{t\geq 1}\sim\mu_0$. For each binary sequence $y\in\{0,1\}^{k(\epsilon)}$, we construct an increasing sequence of times as follows. For convenience, we initialize $t_y(0)=0$. Having constructed $t_y(0),\ldots,t_y(l-1)$ for $l\in[k(\epsilon)]$, we let 
    \begin{equation*}
        D_y(l):= \set{(x_t,y_t):t\in\{t_y(l'),l'\in[l-1]\} }
    \end{equation*}
    be the dataset composed of pairs $(x_t,y_t)$ for the constructed times $t$. In particular, $D_y(1)=\emptyset$. Next, we define
    \begin{equation*}
        t_y(l):=\min\set{ t>t_y(0)+1 : x_t\in B_{k(\epsilon)-l+1}(D_y(l);\epsilon)} \cup\{\infty\}.
    \end{equation*}
    
    First, we check that by construction, we always have for $l\in[k(\epsilon)+1]$, either $t_y(l-1)=\infty$ or $D_y(l)\in \Lcal_{k(\epsilon)-l+1}(\epsilon)$. Indeed, by construction $\Lcal_{k(\epsilon)}(\epsilon)\neq\emptyset$. Hence, it includes at least $\emptyset = D_y(1)$. Next, for $l\in[k(\epsilon)]$,
    if $t_y(l)<\infty$, then $x_{t_y(l)}\in B_{k(\epsilon)-l+1}(D_y(l))$ and as a result, 
    \begin{equation*}
        D_y(l+1) = D_y(l) \cup \{(x_{t_y(l)},y_l ) \} \in\Lcal_{k(\epsilon)-l}(\epsilon),
    \end{equation*}
    by definition of the critical region. In particular, for each $l\in[k(\epsilon)+1]$, this implies that
    \begin{equation*}
        \sup_{\mu\in\{\mu_0\}} \mu \paren{B_{k(\epsilon)-l+1}(D_y(l);\epsilon)} = \mu\paren{B_{k(\epsilon)-l+1}(D_y(l);\epsilon)} \geq \epsilon.
    \end{equation*}
    
    In summary, for each sequence $y\in\{0,1\}^{k(\epsilon)}$, $t_y(k(\epsilon))$ is the sum of $k(\epsilon)$ i.i.d.\ exponential random variables $\Ecal(\epsilon)$. In particular, by Chernoff's bound, for $T=\ceil{ 2\frac{k(\epsilon)}{\epsilon}}$, we have
    \begin{equation*}
        \Ebb\sqb{\sum_{y\in\{0,1\}^{k(\epsilon)}} \1[t_y(k(\epsilon))\leq T]} = 2^{k(\epsilon)} \Pbb(t_0(k(\epsilon))\leq T) \geq 2^{k(\epsilon)-1}.
    \end{equation*}
    Hence, there must be a realization $x_1,\ldots,x_T$ for which
    \begin{equation*}
        \Scal:=\set{y\in\{0,1\}^{k(\epsilon)+1} : t_y(k(\epsilon)+1)\leq T}
    \end{equation*}
    has cardinality at least $2^{k(\epsilon)-1}$. Now note that by the previous arguments, for $y\in\Scal$ we have $D_y(k(\epsilon)+1)\in\Lcal_0(\epsilon)$. Hence, there exists a function $f_y\in\Fcal$ which realizes this dataset: $f_y(x_{t_y(l)}) =y_l$ for all $l\in[k(\epsilon)]$. We now check that for any distinct sequences $y\neq y'\in\Scal$, we have
    \begin{equation*}
        (f_y(x_t))_{t\in[T]} \neq (f_{y'}(x_t))_{t\in[T]}.
    \end{equation*}
    Indeed, letting $l$ be the first index for which $y_l\neq y_l'$, we can check by construction that we have $t_y(l')=t_{y'}(l')$ for all $l\in\{0,\ldots,l\}$. Then, $f_y(x_{t_y(l)})= y_l \neq y_l' = f(x_{t_y(l)})$.

    In conclusion, the projection of $\Fcal$ onto $\{x_1,\ldots,x_T\}$ has cardinality at least $|\Scal|\geq 2^{k(\epsilon)-1}$. However, by Sauer-Shelah's lemma \cite{sauer1972density,shelah1972combinatorial}, the size of this projection is at most $2T^d$ where $d$ is the VC dimension of $\Fcal$. Therefore,
    \begin{equation*}
        2^{k(\epsilon)-1} \leq 2T^2 \leq 2\paren{\frac{2k(\epsilon)}{\epsilon}+1}^d.
    \end{equation*}
    This implies the desired result $k(\epsilon)\lesssim d\log\frac{2d}{\epsilon}$.
\end{proof}

\subsection{Extensions on Smoothed Distribution Classes}
\label{subsec:extension_smoothed_proofs}

Beyond smoothed distribution classes $\Ucal_{\mu_0,\sigma}$, \cite{block2023sample} introduced a generalization based on $f$-divergences as follows. Fix any convex function $f:[0,\infty]\to[0,\infty]$ satisfying $f(1)=f'(1)=0$ and $f'(\infty)=\infty$, where $f'$ refers to the maximum subgradient of $f$. The $f$-divergence between two distributions $\mu$ and $\nu$ on $\Xcal$ is defined as follows if $\mu\ll \nu$:
\begin{equation*}
    \text{div}_f(\mu\parallel\nu) := \Ebb_{Z\sim \mu}\sqb{ f\paren{\frac{d\mu}{d\nu}(Z)}},
\end{equation*}
and otherwise, $\text{div}_f(\mu\parallel\nu):=\infty$.

For some base measure $\mu_0$ on $\Xcal$ and smoothness parameter $\sigma>0$, they consider the following distribution class
\begin{equation*}
    \Ucal_{\mu_0,\sigma}^{(f)}:= \set{\mu : \text{div}_f(\mu\parallel \mu_0)\leq \frac{1}{\sigma}}.
\end{equation*}

We can check that these distribution classes are either instantiations or generalized by the generalized-smoothed distribution classes introduced in \cref{eq:def_extended_smoothed_distribution_class}.

\begin{lemma}\label{lemma:generalized_smoothness}
    Fix a base measure $\mu_0$ on $\Xcal$ and a parameter $\sigma>0$. Then, $\Ucal_{\mu_0,\sigma}= \Ucal_{\mu_0,\frac{1}{\sigma}Id}(\Sigma)$. 
    
    Next, for any convex function $f:[0,\infty]\to[0,\infty]$ with $f(1)=f'(1)=0$ and $f'(\infty)=\infty$, we have $\Ucal_{\mu_0,\sigma}^{(f)} \subseteq \Ucal_{\mu_0,\rho}(\Sigma)$, where $\rho:[0,1]\to\Rbb_+$ is defined via $\rho(\epsilon)=\inf_{\alpha>0} \set{\alpha\epsilon + \frac{1}{\sigma f'(\alpha)} }$ and satisfies $\lim_{\epsilon\to 0}\rho(\epsilon)=0$.
\end{lemma}

\begin{proof}
    For smoothed classes, the result is immediate. Indeed, $\mu\in\Ucal_{\mu_0,\sigma}$ if and only if for any measurable set $B\in\Sigma$, $\mu(B)\leq \frac{1}{\sigma}\mu_0(B)$.

    Next, we fix a function $f$ satisfying the assumptions, and $\sigma>0$. Fix $\epsilon\in(0,1]$ and $\alpha>0$ such that $f'(\alpha)>0$. For any distribution $\mu\in\Ucal_{\mu_0,\sigma}^{(f)}$ and $B\in\Sigma$,
    \begin{align*}
        \mu(B) &\leq \alpha \mu_0(B)  + \Ebb_{\mu_0}\sqb{\paren{\frac{d\mu}{d\mu_0}(Z) -\alpha } \1\sqb{\frac{d\mu}{d\mu_0}(Z)\geq \alpha} \1[Z\in B]}\\
        &\overset{(i)}{\leq} \alpha \mu_0(B) + \frac{\text{div}_f(\mu\parallel \mu_0)}{f'(\alpha)} \overset{(ii)}{\leq} \alpha\mu_0(B) + \frac{1}{\sigma f'(\alpha)}.
    \end{align*}
    In $(i)$ we used the fact that $f$ is convex, and in $(ii)$ we used $\mu\in\Ucal_{\mu_0,\sigma}^{(f)}$. As a result, we obtained $\mu(B)\leq \rho(\mu_0(B))$. It only remains to check that $\lim_{\epsilon\to 0}\rho(\epsilon) =0$. This is a consequence from the fact that $f$ is convex and $f'(\infty)=\infty$. Indeed, for any $\delta>0$ there exists $\alpha(\delta)$ such that $f'(\alpha(\delta)) \geq \delta/\sigma$ and this implies that for $\epsilon \leq 1/\alpha(\delta)$, we have $\rho(\epsilon) \leq 2\delta$. This ends the proof.
\end{proof}

In particular, \cref{prop:extended_smooth_learnable} directly recovers the fact that these distribution classes considered in the literature, are learnable. 
We can also use the bounds from \cref{prop:extended_smooth_learnable} to explicitly write the minimax regret bounds from \cref{thm:quantitative_charact} for generalized smoothed classes. As a comparison with prior works, below are the bounds implied by \cref{thm:quantitative_charact} for smoothed and divergence-based distribution classes

\begin{corollary}\label{cor:rate_smooth_divergence}
    Let $\Fcal$ be a function class on $\Xcal$ with finite VC dimension $d$. Fix any distribution $\mu_0$ on $\Xcal$, $\sigma\in(0,1]$, and $T\geq 2$. Then,
    \begin{equation*}
        \AdaptR_T(\Fcal,\Ucal_{\mu_0,\sigma}) \lesssim \sqrt{dT\log T\cdot \log\frac{T}{\sigma} } \quad \text{and} \quad \AdaptRealizableR_T(\Fcal,\Ucal_{\mu_0,\sigma}) \lesssim d \log\frac{T}{\sigma}.
    \end{equation*}
    Next, fix a convex function $f:[0,\infty]\to[0,\infty]$ with $f(1)=f'(1)=0$ and $f'(\infty)=\infty$. Then,
    \begin{align*}
        \AdaptR_T(\Fcal,\Ucal_{\mu_0,\sigma}^{(f)}) &\lesssim \inf_{\epsilon\in(0,1]}\set{\epsilon T + \sqrt{dT\log T\cdot \log\paren{T(f')^{-1}\paren{\frac{1}{\epsilon \sigma}}}}  +\sqrt{T\log T}}\\
        \AdaptRealizableR_T(\Fcal,\Ucal_{\mu_0,\sigma}^{(f)}) &\lesssim \inf_{\epsilon\in(0,1]}\set{\epsilon T + d\log  \paren{T(f')^{-1}\paren{\frac{1}{\epsilon \sigma}}} +1}.
    \end{align*}
\end{corollary}

\begin{proof}
    If $d\geq T$ the bound is immediate hence without loss of generality $d\leq T$.
    We start with smooth classes. From \cref{prop:extended_smooth_learnable,lemma:bound_k_stochastic_setting}, we immediately have for any $\epsilon\in(0,1]$,
    \begin{equation*}
        k(\epsilon;\Fcal,\Ucal_{\mu_0,\sigma}) \leq k(\sigma\epsilon;\Fcal,\{\mu_0\}) \lesssim d\log\frac{2d}{\sigma\epsilon} \lesssim d\log\frac{T}{\sigma\epsilon}.
    \end{equation*}
    Plugging this into \cref{thm:quantitative_charact} with $\epsilon=1/T$ gives the desired result.

    Next, we turn to $f$-divergence smooth classes. Fix a function $f$ satisfying the assumptions. From \cref{lemma:generalized_smoothness}, if we define $\rho(\epsilon)=\inf_{\alpha>0}\set{\alpha\epsilon+\frac{1}{\sigma f'(\alpha)}}$, then we have $\lim_{\epsilon\to 0}\rho(\epsilon)=0$ and $\Ucal_{\mu_0,\sigma}^{(f)}\subseteq \Ucal_{\mu_0,\rho}$. Therefore, \cref{prop:extended_smooth_learnable,lemma:bound_k_stochastic_setting} imply that for any $\epsilon\in(0,1]$,
    \begin{equation*}
        k(\epsilon;\Fcal,\Ucal_{\mu_0,\sigma}^{(f)}) \leq k(\rho^{-1}(\epsilon);\Fcal,\{\mu_0\}) \lesssim d\log \frac{T}{\rho^{-1}(\epsilon)}.
    \end{equation*}
    Now note that with $\alpha = (f')^{-1}(\frac{1}{\epsilon\sigma})>0$ we have
    $\rho\paren{\frac{\epsilon}{\alpha}} \leq 2\epsilon$.
    In particular, for $\epsilon>1/T$, we have
    \begin{equation*}
        k(2\epsilon;\Fcal,\Ucal_{\mu_0,\sigma}^{(f)}) \lesssim d\log\paren{\frac{T}{\epsilon} (f')^{-1}\paren{\frac{1}{\epsilon\sigma}}} \lesssim d\log\paren{T (f')^{-1}\paren{\frac{1}{\epsilon\sigma}}}.
    \end{equation*}
    Plugging this bound within \cref{thm:quantitative_charact} gives the desired bound.
\end{proof}

This recovers the regret bounds for smoothed settings \cite{block2022smoothed,haghtalab2024smoothed} and for $f$-divergence-based smooth settings \cite[Theorem 10]{block2023sample} up to the $\log T$ factors.

We next turn to pairwise-smoothed classes and show that these are learnable for oblivious adversaries.

\vspace{3mm}

\begin{proof}[of \cref{prop:pairwise_oblivious}]
    Fix a function class $\Fcal$ on $\Xcal$ with finite VC dimension $d$, a base measure $\mu_0$ on $\Xcal$ and a non-decreasing function $\rho:[0,1]\to\Rbb_+$ with $\lim_{\epsilon\to 0}\rho(\epsilon)=0$. Fix $\epsilon\in(0,1]$ and let $\delta\in(0,1]$ such that $\rho(\delta)<\epsilon/2$.

    Fix $\eta:=\delta\epsilon/6$. Let $d\leq \dim(\epsilon,\eta ;\Fcal,\Ucal_{\mu_0,\rho}^{\text{pair}}(\Fcal))$ and a $(\epsilon,\eta)$-shattered tree, with distributions $\nu_v\in\bar\Ucal_{\mu_0,\sigma}^{pair}$ and functions $f_{v,0},f_{v,1}\in\Fcal$ for all $v\in\Tcal_d$ and satisfying the corresponding constraints (see \cref{def:dimension_oblivious_small_tolerance}).
    For any node $v\in\Tcal_d$, letting $B_v:=\set{x:  \frac{d\nu_v}{d\mu_0}(x) \geq \frac{\epsilon}{2}}$, we have
    \begin{align*}
        \epsilon \leq \nu_v (f_{v,0}\neq f_{v,1}) &\leq \frac{\epsilon}{2} + \nu_v(\{x\in B_v:f_{v,0}(x)\neq f_{v,1}(x)\})\\
        &\overset{(i)}{\leq} \frac{\epsilon}{2} +\rho \paren{\mu_0(\{x\in B_v:f_{v,0}(x)\neq f_{v,1}(x)\})} = \frac{\epsilon}{2} + \rho\paren{\mu_0(f_{v,0}\neq f_{v,1};B_v)}.
    \end{align*}
    In $(i)$ we used the fact that since $\nu_v$ is a mixture of distributions $\nu$ satisfying $\nu(B_v)\leq \rho(\mu_0(B_v))$, so does $\nu_v$.
    Because $\rho$ is non-decreasing and $\rho(\delta)<\epsilon/2$ this implies $\mu_0(f_{v,0}\neq f_{v,1};B_v) \geq \delta$. On the other hand, for any left-descendant function $f^0$ of $v$, we have
    \begin{equation*}
        \mu_0(f_{v,0}\neq f^0;B_v) \leq \frac{2}{\epsilon} \Ebb_{Z\sim\mu_0}\sqb{\frac{d\nu_v}{d\mu_0}(Z) \1[Z\in B_v]\1[f_{v,0}(Z)\neq f^0(Z)]} = \frac{2}{\epsilon} \nu_v(f_{v,0}\neq f^0; B_v) \leq \frac{2\eta}{\epsilon} = \frac{\delta}{3}.
    \end{equation*}
    We then consider the tree where at all nodes $v\in\Tcal_d$, we used the region $B_v$, the functions $f_{v,0},f_{v,1}$, and the distribution $\mu_0$.
    The previous equations show that satisfies all the constraints within \cref{def:subregion_oblivious}. This implies $d\leq \tildedim(\delta;\Fcal,\{\mu_0\})$. In summary, we showed
    \begin{equation*}
         \dim(\epsilon,\eta;\Fcal,\Ucal_{\mu_0,\rho}^{\text{pair}}(\Fcal)) \leq \tildedim(\delta;\Fcal,\{\mu_0\}).
    \end{equation*}
    Since $\Fcal$ has finite VC dimension, the distribution class $\{\mu_0\}$ is learnable for oblivious (also adaptive) adversaries, \cref{thm:oblivious_adversaries} implies that for all $\delta>0$, one has $\tildedim(\delta;\Fcal,\{\mu_0\})<\infty$. In summary, we showed that for all $\epsilon>0$, $\dim(\epsilon;\Fcal,\Ucal_{\mu_0,\rho}^{\text{pair}}(\Fcal))<\infty$ (e.g. see \cref{prop:equivalent_oblivious,lemma:comparing_various_dimensions}). Hence, \cref{thm:oblivious_adversaries} implies that $(\Fcal,\Ucal_{\mu_0,\rho}^{\text{pair}}(\Fcal))$ is learnable for oblivious adversaries.
\end{proof}

\subsection{Function Classes of VC Dimension 1}
\label{subsec:vc_1_classes_instantiation_proofs}

Within this section we instantiate the general learnability characterizations when $\Fcal$ has VC dimension 1. We assume without loss of generality that $\Fcal$ has been pruned, as discussed within \cref{subsec:VC1_classes}. Namely, there exists a tree ordering $\preceq$ such that any $f\in\Fcal$ is an initial segment of $\preceq$. Also, for all $x\in\Xcal$ there exists $f,g\in\Fcal$ with $f(x)\neq g(x)$. Last, for any distinct $x_1,x_2\in\Xcal$, there exists $f\in\Fcal$ with $f(x_1)\neq f(x_2)$.

\paragraph{Oblivious Adversaries.} We start with the case of oblivious adversaries, for which we can simplify the characterization from \cref{thm:oblivious_adversaries}. First, we relate $\Tdim(\epsilon;\preceq,\Ucal)$ to the other dimensions defined for oblivious adversaries.

\begin{proposition}\label{prop:oblivious_VC_1}
    Let $\Fcal$ be a function class with VC dimension $1$ with tree ordering $\preceq$ and let $\Ucal$ be a distribution class on $\Xcal$. Then,
    \begin{equation*}
        \floor{\log_2 \Tdim(2\epsilon;\preceq,\Ucal)} \leq \tildedim(\epsilon;\Fcal,\Ucal) \leq \Tdim(\epsilon/8;\preceq,\Ucal).
    \end{equation*}
\end{proposition}

\begin{proof}
    Fix $\epsilon>0$. We prove each bound separately.
    We start by showing that for any $\delta>0$, we have $\tildedim(\epsilon;\Fcal,\Ucal) \geq \floor{\log_2 \Tdim(\epsilon+\delta;\preceq,\Ucal)}=:d_0$. Let $x_1\prec\ldots\prec x_{2^{d_0}}\in\Xcal$ that satisfy the assumptions for $\Tdim(\epsilon+\delta;\preceq,\Ucal)$. For convenience, we introduce a notation $x_0$ such that for all $x\in\Xcal$, $x_0\prec x$. By assumption, for any $l\in[2^{d_0}]$ we have $\sup_{\mu\in\Ucal} \mu(\{x\in\Xcal: x_{l-1}\prec x\preceq x_l\})\geq \epsilon+\delta$. We then fix $\nu_l\in\Ucal$ such that
    \begin{equation*}
        \nu_l(\{x\in\Xcal: x_{l-1}\prec x\preceq x_l\})\geq \epsilon+\frac{\delta}{2},\quad l\in[2^{d_0}].
    \end{equation*}
    For any $l\in[2^{d_0}]$ there exists $y_{l-1}$ such that $x_{l-1}\prec y_{l-1}$ and $\nu_l(\{x\in\Xcal: x_{l-1}\prec x\prec y_{l-1}\})\leq \delta/2$. Then,
    \begin{equation}\label{eq:first_condition_dim_tree}
        \nu_l(\{x\in\Xcal: y_{l-1}\preceq x\preceq x_l\}) = \nu_l(\{x\in\Xcal: x_{l-1}\prec x\preceq x_l\})-\nu_l(\{x\in\Xcal: x_{l-1}\prec x\prec y_{l-1}\}) \geq \epsilon.
    \end{equation}
    By assumption, since $x_{l-1}\neq y_{l-1}$ there exists $g_{l}\in\Fcal$ such that $g_l(x_{l-1})\neq g_l(y_{l-1})$. Because $x_{l-1}\prec y_{l-1}$ this implies $g_{l }(x)=1$ for any $x\preceq x_{l-1}$ while $g_{l }(x) =0$ for any $y_{l-1}\preceq x$. The functions $f_1,\ldots,f_{2^{d_0}}$ effectively form thresholds on $\{x\in\Xcal: x\preceq x_{2^{d_0}}\}$. 
    
    We then construct an $\epsilon$-shattering tree of depth $d_0$ as follows. We enumerate all inner nodes $v\in\Tcal_d$ using an index $i(v)\in[2^{d_0}-1]$ such that for any left-descendant of $w$ of $v$ has $i(w)<i(v)$ and any right-descendant $w$ of $v$ has $i(w)>i(v)$. Then, for any node $v=(v_j)_{j\in[i]}\in\Tcal_{d_0}$ we pose
    \begin{equation*}
        \mu_v := \nu_{i(v)},\quad B_v:=\set{x\in\Xcal:y_{i(v)-1}\preceq x\preceq x_{i(v)}}, \quad \text{and} \quad \begin{cases}
            f_{v,0}:= g_{i(v)},\\
            f_{v,1}:= g_{i(v)+1}.
        \end{cases}
    \end{equation*}
    It remains to check that this satisfies the conditions from \cref{def:subregion_oblivious}. By construction, for any $v\in\Tcal_{d_0}$, on $B_v$, the function $g_{i(v)}$ is equal to $0$ while $g_{i(v)+1}$ is equal to $1$. Then, using \cref{eq:first_condition_dim_tree},
    \begin{equation*}
        \mu_v(f_{v,0}\neq f_{v,1}; B_v)= \mu_{i(v)} (B_v) \geq \epsilon.
    \end{equation*}
    Next, for any left-descendant function $g$ of $v$, by construction of the tree and the index, we must have $g=g_i$ where $i\leq i(v)$. Therefore, $g$ is also equal to $0$ on $B_v$, exactly as $f_{v,0}$ and as a result $\mu_v(f_{v,0}\neq g; B_v)=0$. Similarly, if $g$ is a right-descendant function of $v$ we have $g=g_i$ where $i\geq i(v)+1$ and hence $g$ equals 1 on $B_v$ exactly as $f_{v,1}$. Hence, $\mu_v(f_{v,1}\neq g; B_v)=0$. This ends the proof that $\tildedim(\epsilon;\Fcal,\Ucal)\geq d_0$.

    Next, we prove that $d_1:=\tildedim(\epsilon;\Fcal,\Ucal)\leq \Tdim(\epsilon/6;\preceq,\Ucal)$. We fix a corresponding shattered tree of depth $d_1$: let $\mu_v\in\bar\Ucal$, $B_v\in\Sigma$, and $f_{v,0},f_{v,1}\in\Fcal$ for any $v\in\Tcal_{d_1}$ satisfying the corresponding assumptions from \cref{def:subregion_oblivious}. Fix $v\in\Tcal_{d_1}$. Note that
    \begin{multline*}
        \mu_v(\set{x\in B_v:f_{v,0}(x)=1=1-f_{v,1}(x)}) + \mu_v(\set{x\in B_v:f_{v,1}(x)=1=1-f_{v,0}(x)})\\
        =  \mu_v(f_{v,0}\neq f_{v,1}) \geq \epsilon.
    \end{multline*}
    Hence, we can fix $y(v)\in\{0,1\}$ such that 
    \begin{equation*}
        \mu_v(\set{x\in B_v:f_{v,y(v)}(x)=1=1-f_{v,1-y(v)}(x)})\geq \frac{\epsilon}{2}.
    \end{equation*}
    We follow the path down the tree given by these labels $y_v$. Formally, we construct a sequence $v_1,\ldots,v_{d_1}$ such that $v_l\in\{0,1\}^{l-1}$ for $l\in[d_1]$ and $v_{l+1}=(v_l,y(v_l))$ for $l\in[d_1-1]$. For convenience, we write $\mu_l:=\mu_{v_l}$, $f_l:=f_{v_l,y(v_l)}$ and $g_l:=f_{v_l,1-y(v_l)}$ for $l\in[d_1]$. 
    
    Next, fix any $\delta\in(0,\epsilon/12)$. Since $f_l,g_l\in\Fcal$, we have that $\preceq$ is a total order on $\{x:f_l(x)=1\}$ and on the other hand, if $g_l(x)=0$ then for all $x\preceq y$, we also have $g_l(x)=0$. 
    Together, this shows that there exists $z_l\in\Xcal$ such that $f_l(z_l)=1$, $g_l(z_l)=0$, and
    \begin{equation}\label{eq:distinct_enough}
        \mu_l(B_{v_l}\cap \set{x\in\Xcal:f_l(x)=1,z_l\preceq x})\geq \frac{\epsilon}{2} -\delta.
    \end{equation}
    Next, for  $g\in\{f_s,g_s,s> l\}$, since $g$ is a $y(v_l)$-descendant of $v_l$, we have
    \begin{multline}\label{eq:descendants_close}
        \mu_l(\set{x\in\Xcal:g(x)=f_{l}(x)=1,z_l\preceq x})\\
        \geq \mu_l(B_{v_l}\cap \set{x\in\Xcal:f_l(x)=1,z_l\preceq x}) - \mu_l(\set{x\in B_{v_l}:g(x)\neq f_{l}(x)}) \geq \frac{\epsilon}{2}-\delta -\frac{\epsilon}{3}=\frac{\epsilon}{6}-\delta,
    \end{multline}
    where in the last inequality we used \cref{eq:distinct_enough} and the definition of shattered trees in \cref{def:subregion_oblivious}. Next, recalling that $\{x:f_l(x)=1\}$ is totally ordered by $\preceq$, let $x_l\in\Xcal$ such that $z_l\preceq x_l$, $f_l(x_l)=1$, and
    \begin{equation*}
        \mu_l(\set{x\in\Xcal: z_l\preceq x\prec x_l})< \frac{\epsilon}{6}-\delta-\delta_l \leq  \mu_l(\set{x\in\Xcal: z_l\preceq x\preceq x_l}) ,
    \end{equation*}
    where $\delta_l\in(0,\delta]$.
    Then, \cref{eq:descendants_close} shows that for any $g\in\{f_s,g_s,s> l\}$ there exists $x\in\Xcal$ with $x_l\preceq x$ and $g(x)=f_l(x)=1$. Since $g,f_l\in\Fcal$ this implies that for all $x'\preceq x$, we have $g(x')=f_l(x')=1$. In particular, we showed that
    \begin{equation*}
        \forall l\in[d_1-1], \forall z_l\preceq x \preceq x_l:\quad f_{l+1}(x)=g_{l+1}(x)=f_l(x)=1.
    \end{equation*}
    Now since $g_{l+1}(z_{l+1})=0$ while $f_{l+1}(z_{l+1})=1$, this implies $x_l\prec z_{l+1}$. In summary, we constructed a sequence $z_1\preceq x_1\prec z_2\preceq x_2\prec \ldots \prec z_{d_1}\preceq x_{d_1}$ such that
    \begin{equation*}
        \sup_{\mu\in\Ucal}\mu(\set{x:z_l\preceq x\preceq x_l}) \geq \mu_l(\set{x:z_l\preceq x\preceq x_l}) \geq \frac{\epsilon}{6}-2\delta, \quad l\in[d_1].
    \end{equation*}
    In the first inequality, we used the fact that since $\mu_l\in\bar\Ucal$ is a mixture of distributions on $\Ucal$, there must exist $\tilde\mu_l$ that puts at least as much as mass on the desired region.
    Taking $\delta$ sufficiently small ends the proof that $\Tdim(\epsilon/8;\preceq,\Ucal)\geq d_1$.
\end{proof}

Using \cref{prop:oblivious_VC_1} and \cref{thm:oblivious_adversaries} directly implies \cref{prop:VC_1_oblivious_statement}. Further, we can derive quantitative bounds on the minimax oblivious regret in terms of $\Tdim(\epsilon;\preceq,\Ucal)$ by combining \cref{prop:oblivious_VC_1,thm:combined_oblivious_non_asymptotic}.

\paragraph{Adaptive Adversaries.}
We next turn to the case of adaptive adversaries and prove \cref{thm:characterization_VC1}. To do so, we compare $k(\epsilon)$ and $\tilde k(\epsilon)$ for $\epsilon>0$

\begin{proposition}\label{prop:compare_k_ilde_k}
    Let $\Fcal$ be a function class with VC dimension $1$ and let $\Ucal$ be a distribution class on $\Xcal$. Then, for any $\epsilon>0$, $\tilde k(\epsilon)\in\{k(\epsilon)-1,k(\epsilon)\}$.
\end{proposition}

\begin{proof}
    Fix $\epsilon>0$. It suffices to relate the sets $\Lcal_k(\epsilon)$ and $\tilde\Lcal_k(\epsilon)$ for $k\geq 0$. For convenience, we write $\Dcal_1:=\set{D\subseteq\Xcal\times\{0,1\}:|D|<\infty\text{ and }\exists x\in\Xcal,(x,1)\in D}$, the finite datasets that have at least one sample labeled $1$. Next, for $D\in\Dcal_1$, we denote $x(D):=\max_{\preceq} \set{x\in\Xcal:(x,1)\in D}$ the point that can be used to compress $D$.
    We prove by induction that for any $k\geq 0$,
    \begin{equation}\label{eq:induction_hypothesis_L}
        \Lcal_k(\epsilon)\cap\Dcal_1 = 
        \{D\in\Dcal_1: x(D) \in\tilde\Lcal_k(\epsilon)\}\cap\Lcal_0(\epsilon).
    \end{equation}

    This is immediate for $k=0$ since $\widetilde\Lcal_0(\epsilon)=\Xcal$. We now suppose that \cref{eq:induction_hypothesis_L} holds for $k-1$ for some $k\geq 1$. Then, recalling that $\Lcal_k(\epsilon)\subseteq\Lcal_0(\epsilon)$ (see \cref{lemma:decreasing_level_sets}) we focus on a dataset $D\in\Dcal_1\cap\Lcal_0(\epsilon)$. In particular, for any $(x,y)\in D$ we have $y=\1[x\preceq x(D)]$. Then,
    \begin{align*}
        B_k(D;\epsilon) &= \set{x\in\Xcal: D\cup\{(x,0)\}, D\cup\{(x,1)\}\in\Lcal_{k-1}(\epsilon)}\\
        &= \set{x\in\Xcal: D\cup\{(x,0)\}, D\cup\{(x,1)\}\in\Lcal_{k-1}(\epsilon)\cap\Dcal_1 \text{ and } x(D)\prec x}\\
        &\overset{(i)}{=}\set{x\in\Xcal: x(D)\prec x }\cap\set{x\in\Xcal: D\cup\{(x,0)\},D\cup\{(x,1)\} \in\Lcal_0(\epsilon) } \\
        &\qquad\qquad\cap \{x\in\Xcal,x(D),x\in\widetilde \Lcal_{k-1}(\epsilon) \}\\
        &\overset{(ii)}{=}\set{x\in\Xcal: x(D)\prec x }\cap\set{x\in\Xcal: D\cup\{(x,0)\},D\cup\{(x,1)\} \in\Lcal_0(\epsilon) } \cap\widetilde \Lcal_{k-1}(\epsilon)\\
        &\overset{(iii)}{=}\set{x\in\Xcal: x(D)\prec x }  \cap\widetilde \Lcal_{k-1}(\epsilon)
    \end{align*}
    In $(i)$ we used \cref{eq:induction_hypothesis_L}. In $(ii)$ we used the fact that for any if $x(D)\prec x$ and $x\in\widetilde \Lcal_{k-1}(\epsilon)$, then we also have $x(D)\in\widetilde\Lcal_{k-1}(\epsilon)$ by construction of the sets $\widetilde\Lcal_{k'}(\epsilon)$ for $k'\geq 0$. In $(iii)$ we used the following arguments. If $x(D)\prec x$ then $x(D)\neq x$ and by assumption there exists a function $f\in\Fcal$ such that $f(x)\neq f(x(D))$. Since $x(D)\prec x$ this implies $f(x(D))=1$ and $f(x)=0$. Because $D$ itself was realizable this implies that $f$ realizes $D\cup\{(x,0)\}$. Next, by assumption there exists $f\in\Fcal$ such that $f(x)=1$, which we can check then realizes $D\cup\{(x,1)\}$. In summary, we showed that for any $D\in\Dcal_1\cap\Lcal_0(\epsilon)$ one has
    \begin{equation*}
        D\in \Lcal_k(\epsilon) \Longleftrightarrow \sup_{\mu\in\Ucal} \paren{\set{x\in\Xcal: x(D)\prec x }\cap\ \widetilde \Lcal_{k-1}(\epsilon)} \geq \epsilon \Longleftrightarrow x(D)\in \widetilde\Lcal_k(\epsilon).
    \end{equation*}
    This proves \cref{eq:induction_hypothesis_L}.
    
    Then, for any $k\geq 1$, we obtain
    \begin{align*}
        \Lcal_k(\epsilon)\neq\emptyset \overset{(i)}{\Longleftrightarrow}\emptyset \in \Lcal_k(\epsilon) &\Longleftrightarrow  \sup_{\mu\in\Ucal}\mu\paren{ \{x\in\Xcal: \{(x,0)\},\{(x,1)\}\in\Lcal_{k-1}(\epsilon)\}} \geq\epsilon\\
        &\overset{(ii)}{\Longleftrightarrow}  \sup_{\mu\in\Ucal}\mu\paren{ \{x\in\Xcal: \{(x,0)\}\in\Lcal_{k-1}(\epsilon) \}\cap\widetilde\Lcal_{k-1}(\epsilon)} \geq \epsilon
    \end{align*}
    In $(i)$ we used the fact that by construction of the sets $\Lcal_k(\epsilon)$, if $D\in\Lcal_k(\epsilon)$, any $D'\subseteq D$ also satisfies $D\in\Lcal_k(\epsilon)$. In $(ii)$ we used \cref{eq:induction_hypothesis_L}. In particular, if $\widetilde\Lcal_{k-1}(\epsilon)=\emptyset$ then $\Lcal_{k}(\epsilon)=\emptyset$, which shows that $k(\epsilon) \leq \tilde k(\epsilon) + 1$.

    On the other hand, if $\widetilde\Lcal_{k}(\epsilon)\neq \emptyset$, then $\{(x,1)\}\in\Lcal_{k}(\epsilon)$ for any $x\in\widetilde\Lcal_k(\epsilon)$ from \cref{eq:induction_hypothesis_L}. As a result, we have $\Lcal_k(\epsilon)\neq\emptyset$.  
    This shows that $\tilde k(\epsilon) \leq k(\epsilon)$. This ends the proof.
\end{proof}
Combining \cref{prop:compare_k_ilde_k} with the general characterizations \cref{thm:qualitative_charact,thm:quantitative_charact} proves \cref{thm:characterization_VC1}.

\paragraph{Optimistic Learning.}
We now show that for any function class $\Fcal$ with VC dimension 1, optimistic learning is possible in all considered settings with the same algorithm.
We start by constructing this algorithm then show that it learns under any distribution class for which $(\Fcal,\Ucal)$ is learnable, as characterized in \cref{thm:oblivious_adversaries,thm:qualitative_charact} (faster rates are possible if one considers the realizable setting, but this will be sufficient to show that the algorithm optimistically learns in both realizable and agnostic settings).

The final algorithm will perform the Hedge algorithm over a set of experts, which we now define. Let $S\subseteq [T]$ be a subset of times. We define an expert $E(S)$ which uses these times to update its belief of the current version space. Precisely, we write $S=\{t_1<\ldots<t_k\}$ where $k=|S|$ and define $t_{k+1}:=T+1$ for convenience. At all times $t<t_1$ the expert predicts $\hat y_t(S)=0$ and for any period $l\in[k]$, for $t\in[t_l,t_{l+1})$ the algorithm predicts $\hat y_t(S) = \1[x_t\preceq x_{t_l}]$.
This expert is summarized in \cref{alg:expert_VC_1_optimistic}.

\begin{algorithm}[t]

    \caption{Expert $E(S)$ for optimistically learning function classes of VC dimension 1}\label{alg:expert_VC_1_optimistic}
    
    \LinesNumbered
    \everypar={\nl}
    
    \hrule height\algoheightrule\kern3pt\relax
    \KwIn{Function class $\Fcal$, horizon $T$, set of times $S=\{t_1<\ldots<t_l\}\subseteq [T]$, $t_{k+1}:=T+1$}
    
    \vspace{3mm}

    \lFor{$t\in [t_1-1]$}{Predict $\hat y_t(S)=0$}
    \For{$l\in[k]$}{
        \lFor{$t\in[t_l,t_{l+1})$}{Predict $\hat y_t(S):=\1[x_t\preceq x_{t_l}]$}
    }

    \hrule height\algoheightrule\kern3pt\relax
\end{algorithm}

The final algorithm which we denote by $\Lcal_T^{\text{optimistic}}$ then performs the classical Hedge algorithm detailed in \cref{alg:oblivious_agnostic} on the set of experts $\Scal_E(K_T):=\set{E(S), S\subseteq [T],|S|\leq K_T}$, where $K_T\geq 1$ is a parameter which we choose to be for example $K_T=\sqrt T$ for $T\geq 1$. As can be seen from the proof, in fact, any choice of sequence with $\lim_{T\to\infty} K_T=\infty$ while $\lim_{T\to\infty} \frac{K_T\log T}{T}=0$ would work. 
We prove the following adaptive regret bound on this algorithm.

\begin{theorem}\label{thm:regret_optimistic_algo}
    Let $\Fcal$ be a function class with VC dimension 1 and let $\Ucal$ be a distribution class on $\Xcal$ such that $(\Fcal,\Xcal)$ is learnable for adaptive adversaries (equivalently for the realizable or agnostic settings). Then, for any $T\geq 2$ and adaptive adversary $\Acal$,
    \begin{equation*}
        \AdaptReg_T(\Lcal_T^{\text{optimistic}};\Acal) \lesssim \inf_{\epsilon>0: \tilde k(\epsilon)\leq K_T-1}\set{\epsilon T + \sqrt{(\tilde k(\epsilon)+1) K_T
        \cdot T\log T}}.
    \end{equation*}
\end{theorem}

\begin{proof}
    We fix a learnable distribution class $\Ucal$ for $\Fcal$ on $\Xcal$. From \cref{thm:characterization_VC1}, we have $\tilde k(\epsilon)<\infty$ for all $\epsilon>0$. We now fix $\epsilon>0$ such that $K_T\geq \tilde k(\epsilon)+1$.
    Since $|\Scal_E(K_T)|\leq T^{K_T+1}$, \cref{thm:regret_hedge} implies that
    \begin{equation}\label{eq:event_E_optimistic}
        \Ecal:=\set{ \sum_{t=1}^T \1[\hat y_t\neq y_t] -\min_{S\subseteq [T],|S|\leq k} \sum_{t=1}^T \1[\hat y_t(S)\neq y_t] \leq  c_0\sqrt{K_T\cdot T\log T}}
    \end{equation}
    has probability at least $1-1/T$ for some universal constant $c_0>0$. Next, we introduce the following notation:
    \begin{equation*}
        \tilde k(x;\epsilon):=\sup\{k\geq 0:x\in\widetilde\Lcal_k(\epsilon)\},\quad x\in\Xcal.
    \end{equation*}
    We next define
    \begin{equation*}
        B_l(S):= \{x'\in\Xcal:x_{t_l}\prec x'\} \cap \widetilde\Lcal_{\tilde k(x_{t_l};\epsilon)}.
    \end{equation*}
    By construction we have $\widetilde\Lcal_{\tilde k(x_{t_l};\epsilon)+1}=\emptyset$, which implies $\mu(B_l(S)) < \epsilon$ for all $\mu\in\Ucal$.
    Then, Azuma-Hoeffding's inequality together with the union bound shows that writing $S=\{t_1<\ldots<t_k\}$ and defining $t_{k+1}:=T+1$,
    \begin{equation*}
        \Fcal(S;\epsilon,\delta):=\bigcap_{l\in[k]}\set{ \sum_{t_l<t<t_{l+1}}\1[x_t\in B_l(S)] \leq \epsilon(t_{l+1}-t_l) + 2\sqrt{(t_{l+1}-t_l)\log\frac{T}{\delta}}}
    \end{equation*}
    has probability at least $1-\delta$ since $k<T$. Hence, taking the union bound over all such sets $S$, the event $\Fcal(\epsilon):=\bigcap_{S\subseteq [T],|S|\leq K_T} \Fcal(S;\epsilon,T^{-K_T-2})$ has probability at least $1-1/T$.

    We now reason conditionally on the complete sequence $(x_t,y_t)_{t\in[T]}$. Because it contains only a finite number of samples, we can fix $f^\star\in\Fcal$ such that
    \begin{equation*}
        \sum_{t=1}^T \1[f^\star(x_t)\neq y_t] = \inf_{f\in\Fcal} \sum_{t=1}^T \1[f(x_t)\neq y_t].
    \end{equation*}

    We next construct a set of times $S^\star\subseteq[T]$ iteratively as follows. First, if for all $t\in[T]$, $f^\star(x_t)=0$ we simply pose $S^\star=\emptyset$. We assume this is not the case from now and let $t_1^\star=\min\{t\in [T]: f^\star(x_t)=1\}$. Suppose we already constructed $t_1^\star,t_2^\star,\ldots,t_l^\star$ for $l\in[\tilde k(\epsilon)+1]$. We then denote
    \begin{equation*}
        B_l^\star:=\{x\in\Xcal: x_{t_l^\star}\prec x\} \cap \widetilde\Lcal_{\tilde k(\epsilon)-l+1}.
    \end{equation*}
    If for all $t\in\{t_l^\star+1,\ldots,T\}$ we have $x_{t_l^\star}\nprec x_t$, or $x_t\in B_l^\star$, or $f^\star(x_t)=0$, then we end the recursive construction and let $S^\star=\{t_1^\star,\ldots,t_l^\star\}$. Otherwise, we pose 
    \begin{equation*}
        t_{l+1}^\star = \min\set{t\in[T]: t>t_l^\star,x_{t_l^\star}\prec x_t, x_t\notin B_l^\star,f^\star(x_t)=1},
    \end{equation*}
    and continue the recursion. This ends the construction of $S^\star=\{t_1^\star,\ldots,t_k^\star\}$. Note that by construction, we have $|S|\leq \tilde k(\epsilon)+1$. Indeed, suppose that for $l=\tilde k(\epsilon)+1$, $t_l^\star$ was constructed, then because $\widetilde\Lcal_0(\epsilon)=\Xcal$, we have $B_l^\star=\{x\in\Xcal: x_{t_l^\star}\prec x\}$ and hence any $t\in\{t_l^\star+1,\ldots,T\}$ either satisfies $x_t\in B_l^\star$ or $x_{t_l^\star}\nprec x_t$, and hence the procedure stops at that iteration. 

    In particular, letting $t_{k+1}^\star:=T+1$, since $|S^\star|\leq \tilde k(\epsilon)+1\leq K_T$, under $\Ecal$ we obtained
    \begin{multline}\label{eq:regret_bound_1a}
        \sum_{t=1}^T\1[\hat y_t\neq y_t] - \inf_{f\in\Fcal}\sum_{t=1}^T \1[f(x_t)\neq y_t] \leq \sum_{t=1}^T\1[\hat y_t(S^\star) \neq y_t] - \sum_{t=1}^T \1[f^\star(x_t)\neq y_t] + c_0\sqrt{K_T\cdot T\log T}\\
        \overset{(i)}{\leq}\sum_{l=1}^{k} \sum_{t_l^\star<t<t_{l+1}^\star} \1[\hat y_t(S^\star) \neq f^\star(x_t)] + c_0\sqrt{K_T\cdot T\log T},
    \end{multline}
    In $(i)$ we used the fact that by construction, the expert $E(S^\star)$ and $f^\star$ agree on all times $t<t_1^\star$ and $t\in S^\star$. We now focus on an epoch $l\in[k]$. For any $t\in(t_l^\star,t_{l+1}^\star)$, the expert $E(S^\star)$ uses the prediction function $\1[\cdot\preceq x_{t_l^\star}]$. Note that by construction, $f^\star$ also has value $1$ on $x_{t_l^\star}$. Hence, both functions agree at time $t\in(t_l^\star,t_{l+1}^\star)$ if $x_{t_l^\star}\nprec x_t$. Suppose that we now have $x_{t_l^\star}\prec x_t$, then the prediction of $E(S^\star)$ at that time is 0. In summary, using the construction of the time $t_{l+1}^\star$, we have
    \begin{equation}\label{eq:regret_bound_1b}
        \sum_{t_l^\star<t<t_{l+1}^\star} \1[\hat y_t(S^\star) \neq f^\star(x_t)] =  \sum_{t_l^\star<t<t_{l+1}^\star} \1[f^\star(x_t)=1]\1[x_{t_l^\star}\prec x_t] \leq \sum_{t_l^\star<t<t_{l+1}^\star} \1[x_t\in B_l^\star]. 
    \end{equation}

    To further the bound, we derive some further properties on $S^\star$. We prove that for all $l\in[k]$, $x_{t_l^\star}\notin \widetilde\Lcal_{\tilde k(\epsilon)-l+2}$. This is immediate for $l=1$ since by definition of $\tilde k(\epsilon)$, we have $\widetilde\Lcal_{\tilde k(\epsilon)+1}=\emptyset$. Next, for any $l\in\{2,\ldots,k\}$, by construction we have $x_{t_l^\star}\notin B_{l-1}^\star$ and $x_{t_{l-1}^\star}\prec x_{t_{l}^\star}$. Using the definition of $B_{l-1}^\star$ this precisely gives $x_{t_l^\star}\notin \widetilde \Lcal_{\tilde k(\epsilon)-l+2}$. In particular, for any $l\in[k]$, we have $\tilde k(x_{t_l^\star};\epsilon)\leq \tilde k(\epsilon)-l+1$. In summary,
    \begin{equation*}
        B_l^\star \subseteq B_l(S^\star)\quad l\in[k].
    \end{equation*}
    Therefore, combining \cref{eq:regret_bound_1a,eq:regret_bound_1b} on $\Ecal\cap\Fcal(\epsilon)$  we obtained
    \begin{align*}
        \sum_{t=1}^T\1[\hat y_t\neq y_t] - \inf_{f\in\Fcal}\sum_{t=1}^T \1[f(x_t)\neq y_t] &\leq \sum_{l=1}^k \sum_{t_l^\star<t<t_{l+1}^\star} \1[x_t\in B_l(S^\star)] +c_0\sqrt{K_T\cdot T\log T}\\
        &\leq \epsilon T + \sum_{l=1}^k 2\sqrt{(t_{l+1}^\star-t_l^\star)(K_T+3)\log T} + c_0\sqrt{K_T\cdot T\log T}\\
        &\leq \epsilon T + c_1\sqrt{(\tilde k(\epsilon)+1) K_T\cdot T\log T},
    \end{align*}
    for sime universal constant $c_1>0$. In the last inequality we used Jensen's inequality together with the fact that $k=|S^\star|\leq \tilde k(\epsilon)+1$. Taking the expectation and recalling that $\Ecal\cap\Fcal(\epsilon)$ has probability at least $1-2/T$, we obtained
    \begin{equation*}
        \AdaptReg_T(\Lcal_T^{\text{optimistic}};\Acal)\leq \inf_{\epsilon>0: \tilde k(\epsilon)\leq K_T-1}\set{\epsilon T + c_1\sqrt{(\tilde k(\epsilon)+1) K_T\cdot \log T}} + 2.
    \end{equation*}
    This ends the proof.
\end{proof}

In particular, since the regret bound from \cref{thm:regret_optimistic_algo} is independent of the adaptive adversary and decays to $0$ as $T\to\infty$, this implies that the constructed algorithm $\Lcal_T^{\text{optimistic}}$ is an optimistic learner for $\Fcal$. This proves \cref{thm:VC_1_optimistic_learning} for adaptive adversaries.

We next turn to the case of oblivious adversaries and show that the same algorithm $\Lcal_T^{\text{optimistic}}$ described in the previous paragraph is still an optimistic learner for oblivious adversaries.

\begin{theorem}\label{thm:regret_optimistic_algo_oblivious}
    Let $\Fcal$ be a function class with VC dimension 1 and let $\Ucal$ be a distribution class on $\Xcal$ such that $(\Fcal,\Xcal)$ is learnable for oblivious adversaries (equivalently for the realizable or agnostic settings). Then, for any $T\geq 2$ and oblivious adversary $\Acal$,
    \begin{equation*}
        \OblivReg_T(\Lcal_T^{\text{optimistic}};\Acal) \lesssim \inf_{\epsilon>0: \Tdim(\epsilon;\preceq;\Ucal)\leq K_T-1}\set{\epsilon T + \sqrt{( \Tdim(\epsilon;\preceq;\Ucal)+1) K_T\cdot T\log T}} .
    \end{equation*}
\end{theorem}

\begin{proof}
    We fix an oblivious adversary $\Acal$ and denote the learner strategy by $\Lcal$ for conciseness.
    The proof is similar to that of \cref{thm:regret_optimistic_algo}. The event $\Ecal$ as defined in \cref{eq:event_E_optimistic} still has probability at least $1-1/T$. Next, we fix $f^\star\in\Fcal$ such that
    \begin{equation*}
        \Ebb_{\Lcal,\Acal}\sqb{\sum_{t=1}^T \1[f^\star(x_t)\neq y_t]} \leq \inf_{f\in\Fcal} \Ebb_{\Lcal,\Acal}\sqb{\sum_{t=1}^T \1[f(x_t)\neq y_t]} +1.
    \end{equation*}
    
    Next, we fix $\epsilon>0$ such that $K_T\geq \Tdim(\epsilon;\preceq,\Ucal)+1$. Throughout, for convenience, we introduce a notation $z_0$ for an element such that $x_0\prec x$ for any $x\in\Xcal$ (this will simplify notations). We claim that there exists a sequence of $d_0\leq \Tdim(\epsilon;\preceq,\Ucal)$ elements $x_1\prec \ldots \prec x_{d_0}$, such that $f^\star(x_{d_0})=1$ and for all $\mu\in\Ucal$,
    \begin{equation}\label{eq:properties_increasing_sequence}
        \begin{cases}
            \mu(\set{x:x_{l-1}\prec x\prec x_l})\leq 2\epsilon & l\in[d_0],\\
        \mu(\set{x:x_{d_0}\prec x, f^\star(x)=1}) \leq\epsilon.
        \end{cases}
    \end{equation}
    
    This can be constructed explicitly as follows. Let $l\geq 1$ and suppose we have constructed $x_0\prec\ldots\prec x_{l-1}$ such that for $l'\in[l-1]$ we have $\mu(\set{x:x_{l'-1}\prec x\prec x_{l'}})\leq 2\epsilon $ and $ \mu(\set{x:x_{l'-1}\prec x\preceq x_{l'}})\geq \epsilon$ (these conditions are void if $l=1$).
    Let $S:=\{z:f^\star(z)=1, \mu(\{x: x_{l-1}\prec x\preceq z\})\geq \epsilon\}$. First, since $f^\star\in\Fcal$, $S$ is totally ordered by $\preceq$. If $S=\emptyset$ we end the construction and we note that this implies $\mu(\{x: x_{l-1}\prec x, f^\star(x)=1\})\leq \epsilon$. Otherwise, either $S$ has a minimum in which case we denote it by $x_l$ and it satisfies $\mu(\set{x:x_{l-1}\prec x\prec x_l})\leq\epsilon \leq \mu(\set{x:x_{l-1}\prec x\preceq x_l})$. If $S\neq \emptyset$ does not have a minimum, there exists some element which we denote $x_l\in S$ with $\mu(\{x: x_{l-1}\prec x\preceq z\})\leq 2\epsilon$ (in fact this can be made arbitrarily close to $\epsilon$). 
    Note that by definition of $\Tdim(\epsilon;\preceq,\Ucal)$ this construction must end for $l:=d_0\leq d$. By construction the sequence $x_0\prec\ldots\prec x_{d_0}$ exactly satisfies \cref{eq:properties_increasing_sequence}.

    For notational convenience, we define $I_l:=\set{x:x_{l-1}\prec x\preceq x_{l}}$ for $l\in [d_0]$ and $I_{d_0+1}:=\set{x:x_{d_0}\prec x, f^\star(x)=1}$. Then, for any $S\subseteq[T]$, writing $S=\{t_1,\ldots,t_k\}$, Azuma-Hoeffding's inequality implies that for some universal constant $c_1>0$,
    \begin{equation*}
        \Fcal(S;\epsilon,\delta):=\bigcap_{l\in[k]}\set{\forall s\in[d_0+1], \sum_{t_l<t<t_{l+1}} \1[x_t\in I_s] \leq 2\epsilon(t_{l+1}-t_l) + c_1\sqrt{(t_{l+1}-t_l)\log\frac{T}{\delta}}}
    \end{equation*}
    has probability at least $1-\delta$ since $k<T$ and $d_0\leq \Tdim(\epsilon;\preceq,\Ucal)\leq T$. Hence, taking the union bound over all such sets $S$, the event $\Fcal(\epsilon):=\bigcap_{S\subseteq [T],|S|\leq K_T} \Fcal(S;\epsilon,T^{-K_T-2})$ has probability at least $1-1/T$.

    We now reason conditionally on the complete sequence $(x_t,y_t)_{t\in[T]}$. We construct a set of times $S^\star\subseteq[T]$ as follows: $t_1$ is the first time such that $f^\star(x_t)=1$. Next, having defined $t_1,\ldots,t_l$, let $j_l\in[d_0+1]$ such that $x_{t_l}\in I_{j_l}$. We then let $t_{l+1}>t_l$ be the first time $t\in\{t_l+1,\ldots,T\}$ such that $x_t\in I_j$ for $j\in\{j_l+1,\ldots,d_0+1\}$ and stop whenever there is no such time. The resulting set of times $S^\star$ satisfies $|S^\star|\leq d_0+1\leq \Tdim(\epsilon;\preceq,\Ucal)+1\leq K_T$ since there are only $d_0+1$ sets $I_j$ for $j\in[d_0+1]$. For convenience, we denote $k:=|S^\star|$ and $t_{k+1}:=T+1$. Under $\Ecal$, we obtained
    \begin{align*}
        \sum_{t=1}^T\1[\hat y_t\neq y_t] - \sum_{t=1}^T \1[f^\star(x_t)\neq y_t] &\overset{(i)}{\leq} \sum_{t=1}^T \1[\hat y_t(S^\star)\neq f^\star(x_t)] + c_0\sqrt{K_T\cdot T\log T}\\
        &\overset{(ii)}{\leq} \sum_{l=1}^k \sum_{t_l<t<t_{l+1}} \1[\1[x_t\preceq x_{t_l}]\neq f^\star(x_t)] + c_0\sqrt{K_T\cdot T\log T}.
    \end{align*}
    In $(i)$ we used the event $\Ecal$ from \cref{eq:event_E_optimistic} and in $(ii)$ we used the fact that at all times $t<t_1$ and $t\in S^\star$, both $E(S^\star)$ and $f^\star$ agree. Next, by construction for all $t\in (t_l,t_{l+1})$ for $l\in[k]$, we have $x_t\notin \bigcup_{j_l<j\leq d_0+1} I_j$. Further, note that both $\1[\cdot\preceq x_{t_l}]$ and $f^\star$ agree outside of $\bigcup_{j_l\leq j\leq d_0+1} I_j$. In summary, under $\Ecal\cap\Fcal(\epsilon)$ we have
    \begin{align*}
        \sum_{t=1}^T\1[\hat y_t\neq y_t] - \sum_{t=1}^T \1[f^\star(x_t)\neq y_t] &\leq \sum_{l=1}^k \sum_{t_l<t<t_{l+1}} \1[x_t\in I_{j_l}] + c_0\sqrt{K_T\cdot T\log T}\\
        &\overset{(i)}{\leq} 2\epsilon T + c_1\sum_{l=1}^k \sqrt{(t_{l+1}-t_l)(K_T+3)\log T} + c_0\sqrt{K_T\cdot T\log T}\\
        &\overset{(i)}{\leq} 2\epsilon T +c_2\sqrt{(d_0+1)K_T \cdot T\log T},
    \end{align*}
    for some universal constant $c_2>0$.
    In $(i)$ we used $\Fcal(\epsilon)$ and in $(ii)$ we used $k=|S^\star|\leq d_0+1$. Recalling that $\Ecal\cap\Fcal(\epsilon)$ has probability at least $1-2/T$ and using the definition of $f^\star$, we obtained
    \begin{equation*}
        \OblivReg_T(\Lcal;\Acal) \leq \inf_{\epsilon>0: \Tdim(\epsilon;\preceq,\Ucal)\leq K_T-1} \set{\epsilon T+ \sqrt{(\Tdim(\epsilon;\preceq,\Ucal)+1)K_T\cdot \log T}}+ 3.
    \end{equation*}
    This ends the proof.
\end{proof}

This shows that $\Lcal_T^{\text{optimistic}}$ optimistically learns $\Fcal$ for oblivious adversaries. Combining \cref{thm:regret_optimistic_algo,thm:regret_optimistic_algo_oblivious} proves \cref{thm:VC_1_optimistic_learning}. As a consequence of this result, we can check that the simple realizable \cref{alg:vc_1_realizable} described in \cref{subsec:VC1_classes} optimistically learns $\Fcal$ for both adaptive and oblivious adversaries in the realizable case.

\begin{corollary}\label{cor:simpler_alg_vc_1_realizable}
    Let $\Fcal$ be a function class of VC dimension $1$. Then, \cref{alg:vc_1_realizable} optimistically learns $\Fcal$ in the realizable setting, for both oblivious or adaptive adversaries.
\end{corollary}

\begin{proof}
    We denote by $f^\star$ the optimal function in hindsight which in this realizable case has zero error.
    The proof of both \cref{thm:regret_optimistic_algo,thm:regret_optimistic_algo_oblivious} proceeds by showing that some expert $E(S^\star)$ has the desired regret bound. In both cases, the set of times $S^\star$ is defined recursively as iteratively containing the first time $t$ for which $x_t$ belongs to some specific region and $f^\star(x_t)=1$. As a result, we can always write the prediction of this expert at any time $t\in[T]$ as
    \begin{equation*}
        \hat y_t(S^\star):= \1[x_t\preceq z_t],
    \end{equation*}
    where $z_t\in\{x_s:s<t,\,f^\star(x_t)=1\}\cup\{x_\emptyset\}$. In comparison, letting $x_{\max}(t):=\max_{\preceq}\{x_s:s<t,\,f^\star(x_t)=1\}\cup\{x_\emptyset\}$, at iteration $t$, \cref{alg:vc_1_realizable} predicts $\1[x_t\preceq x_{\max}(t)]$. Note that these two predictions only differ if $z_t\prec x_t\preceq x_{\max}(t)$, in which case $E(S^\star)$ predicts $0$ while \cref{alg:vc_1_realizable} predicts the true value $1=f^\star(x_t)$ (since $f\in\Fcal$). In summary, \cref{alg:vc_1_realizable} makes fewer mistakes than $E(S^\star)$. From the proof of \cref{thm:regret_optimistic_algo,thm:regret_optimistic_algo_oblivious}, this shows that \cref{alg:vc_1_realizable} is an optimistic learner in the realizable setting for both oblivious and adaptive adversaries.
\end{proof}

\subsection{Linear Classifiers}
\label{subsec:linear_classifiers_proof}

In this section, we instantiate our general results for the linear classifier function class. We start with the one-dimensional case which corresponds to thresholds. For this simple class $\Fcal_{\text{threshold}}:=\{\1[\cdot\leq x],\, x\in[0,1]\}$, it turns out that learnable distribution classes $\Ucal$ are identical for both oblivious and adaptive adversaries. We prove \cref{prop:thresholds} which gives a compact characterization of learnable distribution classes for this threshold function class.

\vspace{3mm}

\begin{proof}[of \cref{prop:thresholds}]
    Note that $\leq$ is a tree order for the threshold function class.
    Hence, using \cref{prop:VC_1_oblivious_statement}, we can check that a distribution class $\Ucal$ on $[0,1]$ is learnable for the threshold function class if and only for any $\epsilon>0$ there exists $k_\epsilon\geq 1$ and $0:=x_0(\epsilon) <\ldots < x_{k_\epsilon}(\epsilon):=1$ such that for all $l\in[k_\epsilon]$,
    \begin{equation}\label{eq:condition_thresholds_specialized}
        \sup_{\mu\in\Ucal} \mu((x_{l-1}(\epsilon),x_l(\epsilon)))\leq \epsilon.
    \end{equation}

    Suppose this holds, then we can easily check by induction that for any $l\geq 0$, if $\widetilde\Lcal_l(\epsilon) \subseteq [0,x_k(\epsilon)]$ for $k\in[k_\epsilon]$, then we have $\widetilde\Lcal_{l+1}(\epsilon) \subseteq [0,x_k(\epsilon))$ and $\widetilde\Lcal_{l+2}(\epsilon) \subseteq [0,x_{k-1}(\epsilon)]$. Hence, we have $\widetilde\Lcal_{2k_\epsilon+2}=\emptyset$. From \cref{thm:characterization_VC1} this shows that $\Ucal$ is also learnable for adaptive adversaries. Therefore, learnability for oblivious and adaptive adversaries is equivalent for the threshold function class.

    It now suffices to check that the condition discussed in \cref{eq:condition_thresholds_specialized} is equivalent to $\Ucal\subseteq \Ucal_{\mu_0,\rho}^{pair}(\Fcal_{\text{threshold}})$ for some distribution $\mu_0$ on $[0,1]$ and $\rho:[0,1]\to\Rbb_+$ with $\lim_{\epsilon\to 0}\rho(\epsilon)=0$. The latter condition directly implies \cref{eq:condition_thresholds_specialized} by taking the $\delta$-quantiles of $\mu_0$ where $\delta>0$ is such that $\rho(\delta)< \epsilon$. We now prove the converse and suppose that \cref{eq:condition_thresholds_specialized} holds. Without loss of generality, we assume that $k_\epsilon$ is non-increasing for $\epsilon>0$. We define the following distribution on $[0,1]$, where $\epsilon_n:=2^{-n}$ for $n\in\Zbb$:
    \begin{equation*}
        \mu_0:= \sum_{n\geq 1} \frac{1}{2^n (k_{\epsilon_n}+1)}\sum_{l=0}^{k_{\epsilon_n}} \delta_{x_l(\epsilon_n)}.
    \end{equation*}
    We can check that this distribution is well defined and indeed has $\mu([0,1])=1$. Next, we define a non-decreasing function $\rho:[0,1]\to\Rbb_+$ such that $\rho(0)=0$, and for any $l\geq 1$, and $x\in[\frac{1}{2^{l} (k_{\epsilon_l}+1)},\frac{1}{2^{l-1} (k_{\epsilon_{l-1}}+1)} )$ we pose $\rho(x)=2^{-l+2}$. For any $\mu\in\Ucal$ and interval $[a,b]$ for $0\leq a\leq b\leq 1$, let $n\geq 1$ such that $\mu([a,b])\in [\epsilon_{n-1},\epsilon_{n-2})$ where we pose $\epsilon_0=1$. Then, there exists $l\in\{0,\ldots,k_{\epsilon_n}\}$ such that $x_l(\epsilon_n)\in(a,b)$. Hence,
    \begin{equation*}
        \mu([a,b])\leq \epsilon_{n-2} =4\epsilon_n =\rho\paren{\frac{1}{2^{n}(k_{\epsilon_n}+1)}} \leq \rho(\mu_0(\{x_l(\epsilon_n)\})) \leq \rho(\mu_0((a,b))).
    \end{equation*}
    Therefore $\Ucal\subseteq\Ucal_{\mu,\rho}^{pair}(\Fcal_{\text{threshold}})$, which ends the proof.
\end{proof}

We are now ready to prove the characterization of learnable distribution classes for linear classifiers in general dimension.

\vspace{3mm}

\begin{proof}[of \cref{prop:linear_classifiers}]
    We prove that the proposed condition is necessary and sufficient for learning separately.

    \paragraph{Necessity.}
    Fix a distribution class $\Ucal$ that is learnable for oblivious adversaries (which is also the case if it is learnable for adaptive adversaries).
    Note that we can embed the threshold function class within $\Fcal_{\text{lin}}$ via $\{x\mapsto a^\top x\leq b,b\in\Rbb\}$ for all $a\in S_d:=\{x\in\Rbb^d:\|x\|=1\}$ (the support can be mapped to $[0,1]$). Let $k_\epsilon(a)$ be the largest integer for which there exists $b_0:=-\infty<b_1<\ldots<b_k:=+\infty$ such that $\sup_{\mu\in\Ucal} \mu((b_{l-1},b_l])\geq \epsilon$ for all $l\in[k]$. Then as proved in \cref{prop:oblivious_VC_1} we have $k_\epsilon(a) \leq 2^{\tildedim(\epsilon/2;\Fcal,\Ucal)}$. 
    Hence, for any $\epsilon>0$ there exists $k_\epsilon\geq 1$ such that for any $a\in S_d$ there exists $x_0(\epsilon,a)=-\infty<\ldots<x_{k_\epsilon}(\epsilon,a)=+\infty$ with
    \begin{equation*}
        \sup_{\mu\in\Ucal} \mu(\{x:a^\top x \in (x_{l-1}(\epsilon,a),x_l(\epsilon,a))\}) \leq \epsilon.
    \end{equation*}
    The proof of \cref{prop:thresholds} then exactly constructs a distribution $\mu_a$ on $\Rbb$ (up to a bijection from $\Rbb$ to $(0,1)$) and a tolerance function $\rho:[0,1]\to\Rbb_+$ with $\lim_{\epsilon\to 0}\rho(\epsilon)=0$ which only depends on the sequence $k_\epsilon$ for $\epsilon\in\{2^{-n},n\geq 1\}$, and such that
    \begin{equation}\label{eq:condition_linear}
        \Ucal\subseteq \set{ \mu: \mu(B_a(b_1,b_2)) \leq \rho(\mu_a([b_1,b_2])) ,\,b_1<b_2\in\Rbb}.
    \end{equation}
    We next show that we can merge the distributions $\mu_a$ for $a\in S_d$ into a single distribution. First, applying this condition to $a\in\{e_i,i\in[d]\}$ implies that for any $\epsilon>0$ there exists $M_\epsilon$ such that
    \begin{equation*}
        \sup_{\mu\in\Ucal} \mu(\Rbb^d\setminus B(0,M_\epsilon)) \leq \epsilon,
    \end{equation*}
    where $B(0,M_\epsilon):=\set{x\in\Rbb^d:\|x\|< M_\epsilon}$. 

    Let $\Scal(\epsilon)$ be the set of affine subspaces $E\subseteq \Rbb^d$ such that $\sup_{\mu\in\Ucal} \mu(E) \geq \epsilon$ but for any affine subspace $F\subsetneq E$, we have $\sup_{\mu\in\Ucal} \mu(F) \leq \epsilon/2$. First, note that for any $E\in\Scal(\epsilon)$ we have $\sup_{\mu\in\Ucal} \mu(E\cap B(0,M_{\epsilon/4})) \geq 3\epsilon/4$.
We now show that $\Scal(\epsilon)$ is finite. Otherwise, there exists a dimension $d'\in[d]$ and a sequence of distinct subspaces $E_1,E_2,\ldots \in \Scal(\epsilon)$ of dimension $d'$ that are convergent to some $d'$-dimensional affine subspace $E^\star$ for say the Hausdorff metric restricted to $B(0,M_{\epsilon/4})$. Fix $x^\star\in E^\star$ for convenience. Then, there exists a vector $a\in S_d$ perpendicular to $E^\star$ such that there is a subsequence $F_1,F_2\ldots$ from $E_1,E_2,\ldots$ such that for all $n\geq 1$, we have $F_n\nsubseteq \{x\in\Rbb^d: a^\top (x-x^\star)=0 \}$. We denote $b^\star:=a^\top x^\star$. Next, from the convergence of $F_n$ to $E^\star$, for any $\eta>0$, there exists $n(\eta)$ such that
\begin{equation*}
    B(0,M_{\epsilon/4}) \cap F_{n(\eta)} \subseteq B(0,M_{\epsilon/4}) \cap B_a(b^\star-\eta,b^\star+\eta).
\end{equation*}
Therefore
\begin{multline*}
    \sup_{\mu\in\Ucal}\mu\paren{ B(0,M_{\epsilon/4}) \cap \{x:a^\top x\in [b^\star-\eta,b^\star+\eta]\setminus\{b^\star\} \} } \\
    \geq  \sup_{\mu\in\Ucal}\mu(B(0,M_{\epsilon/4}) \cap F_{n(\eta)}) - \sup_{\mu\in\Ucal} \mu(F_{n(\eta)}\cap\{x:a^\top x=b^\star\}) \geq \frac{3\epsilon}{4}-\frac{\epsilon}{2} = \frac{\epsilon}{4}.
\end{multline*}
In the last inequality we noted that since $F_{n(\eta)}\nsubseteq\{x:a^\top x=b^\star\}$ the affine subspace $F_{n(\eta)} \cap \{x:a^\top x=b^\star\}$ has lower dimension. This contradicts \cref{eq:condition_linear}, which ends the proof that $\Scal(\epsilon)$ is finite. 

We can then define $\mu_\epsilon$ to be the uniform mixture over all uniforms distributions on $F\cap B(0,M_{\epsilon/4})$ for $F\in\Scal(\epsilon)$. Note that if $a\in S_d$ and $b\in\Rbb$ satisfy $\sup_{\mu\in\Ucal}\mu(\{x:a^\top x=b\})\geq \epsilon 2^{d+1}$ then there exists $F\in\Scal(\epsilon)$ such that $F\subseteq \{x:a^\top x=b\}$ (this can be noted by iteratively restricting a linear subset of $\{x:a^\top x=b\}$ until it belongs to $\Scal(\epsilon)$), and as a result,
\begin{equation*}
    \mu_{\epsilon}(\{x:a^\top x=b\}) \geq \frac{1}{|\Scal(\epsilon)|}.
\end{equation*}
Using the same mixture idea as in the proof of \cref{prop:thresholds} we can then construct a distribution $\mu^{(0)}$ on $\Rbb^d$ and a tolerance function $\rho^{(0)}:[0,1]\to\Rbb_+$ with $\lim_{\epsilon\to 0}\rho^{(0)}(\epsilon)=0$ such that for any $a\in S_d$ and $b\in\Rbb$, and $\mu\in\Ucal$,
\begin{equation}\label{eq:first_part_0_measure}
    \mu(\{x:a^\top x=b\}) \leq \rho^{(0)}(\mu^{(0)} (\{x:a^\top x=b\})).
\end{equation}
Last, we define $\mu^{(1)}$ as the equal probability mixture between $\mu^{(0)}$ and a standard Gaussian $\Ncal(0,Id)$ on $\Rbb^d$. We aim to prove that there exists a tolerance function $\rho^{(1)}:[0,1]\to\Rbb_+$ with $\lim_{\epsilon\to 0}\rho^{(1)}(\epsilon)=0$ and satisfying the desired equation
\begin{equation}\label{eq:desired_condition}
    \Ucal\subseteq \{\mu:\mu(B_a(b_1,b_2)) \leq \rho^{(1)}(\mu^{(1)}(B_a(b_1,b_2))),\, a\in S_d, b_1<b_2\in\Rbb\}.
\end{equation}
Suppose by contradiction that this is not the case. Then there exists $\epsilon>0$ and a sequence $(a_n,b_{1,n},b_{2,n},\mu_n)_{n\geq 1}\in S_d\times \Rbb^2\times \Ucal$ such that for all $n\geq 1$,
\begin{equation}\label{eq:contradiction}
    \mu_n(B_{a_n}(b_{1,n},b_{2,n})) \geq \epsilon \quad \text{and} \quad  \mu^{(1)}(B_{a_n}(b_{1,n},b_{2,n})) \leq 2^{-n},\quad n\geq 1.
\end{equation}
Note that for all $n\geq 1$, we have $\mu_n(B_{a_n}(b_{1,n},b_{2,n}) \cap B(0,M_{\epsilon/2})) \geq \epsilon/2$. Hence, up to projecting $b_{1,n},b_{2,n}$ to $[-M_{\epsilon/2},M_{\epsilon/2}]$ by compactness we can assume without loss of generality that the sequences are convergent: $a_n\to a^\star\in S_d$ and $(b_{1,n},b_{2,n})\to (b^\star_1,b_2^\star) \in [-M_{\epsilon/2},M_{\epsilon/2}]^2$ as $n\to\infty$. First, note that necessarily $b_1^\star=b_2^\star$, otherwise because $\mu^{(1)}$ is a mixture with a standard Gaussian, it puts non-zero measure on $B(0,M_{\epsilon/2})\cap \{x:{a^\star}^\top x\in(2 b_1^\star/3+b_2^\star/3,b_1^\star/3+2b_2^\star/3)\}$, contradicting \cref{eq:contradiction} for $n$ sufficiently large.  This shows that $b_1^\star=b_2^\star$.

Suppose that $\sup_{\mu\in\Ucal}\mu(\{x:{a^\star}^\top x=b_1^\star\})>0$. Then, because of \cref{eq:first_part_0_measure}, \cref{eq:contradiction} implies that for $n$ sufficiently large,
\begin{equation}\label{eq:contradiction_step}
    \mu_n\paren{((B(0,M_{\epsilon/2})\cap B_{a_n}(b_{1,n},b_{2,n}) )\setminus \{x:{a^\star}^\top x=b_1^\star\} } \geq\frac{\epsilon}{4}.
\end{equation}
Note that this equation also holds if $\sup_{\mu\in\Ucal}\mu(\{x:{a^\star}^\top x=b_1^\star\})=0$. Using the same arguments as before, this therefore implies that for any $\eta>0$,
\begin{equation*}
    \sup_{\mu\in\Ucal} \mu\paren{B(0,M_{\epsilon/2}) \cap \{x:{a^\star}^\top x\in [b_1^\star-\eta,b_1^\star+\eta]\setminus\{b_1^\star\}\}} \geq \frac{\epsilon}{4},
\end{equation*}
which contradicts \cref{eq:condition_linear}. In summary, we showed that the desired condition \cref{eq:desired_condition} is indeed necessary even for oblivious adversaries.

\paragraph{Preliminaries for Sufficiency.} 
We now show that the desired condition is also sufficient for learning. Hence, we suppose given a distribution $\mu_0$ and a tolerance function $\rho$ satisfying the required properties and such that
\begin{equation}\label{eq:hypothesis_condition_prop}
    \Ucal\subseteq \{\mu:\mu(B_a(b_1,b_2)) \leq \rho(\mu_0(B_a(b_1,b_2))),\, a\in S_d, b_1<b_2\in\Rbb\}.
\end{equation}
Before showing that this is sufficient for $\Ucal$ to be learnable we first derive a consequence of this condition. Precisely, we show that for any $\epsilon>0$ there exists $\delta>0$ such that for any $a\in S_d$ and $b\in\Rbb$, there exists a hyperplane $H$ such that
\begin{equation}\label{eq:useful_condition}
    \sup_{\mu\in\Ucal} \mu(B_a(b,b+\delta)\setminus H)\leq \epsilon.
\end{equation}
Indeed, first note that as above there exists $M_\epsilon>0$ such that $\sup_{\mu\in\Ucal}\mu(\Rbb^d\setminus B(0,M_\epsilon))\leq \epsilon$. Suppose by contradiction that \cref{eq:useful_condition} does not hold. Then, with similar arguments as above, there exists $\epsilon>0$, a convergent sequence $a_n\to a^\star\in S_d$ and $b_n\to b^\star \in [-M_{\epsilon/2},M_{\epsilon/2}]$ such that with $H^\star=\{x:{a^\star}^\top x=b^\star\}$,
\begin{equation*}
    \sup_{\mu\in\Ucal} \mu((B(0,M_{\epsilon/2})\cap B_{a_n}(b_n,b_n+2^{-n}) )\setminus H^\star)\geq \frac{\epsilon}{2},\quad n\geq 1.
\end{equation*}
This contradicts \cref{eq:hypothesis_condition_prop} since it implies that for any $\eta>0$,
\begin{equation*}
    \sup_{\mu\in\Ucal} \mu\paren{B(0,M_{\epsilon/2}) \cap \{x:{a^\star}^\top x\in [b^\star-\eta,b^\star+\eta]\setminus\{b^\star\}\}} \geq \frac{\epsilon}{2}.
\end{equation*}
Hence, \cref{eq:useful_condition} holds.

\paragraph{Sufficiency.}    
To show the sufficiency of \cref{eq:hypothesis_condition_prop}, from \cref{thm:qualitative_charact} it suffices to show that $\Ucal$ is learnable for adaptive adversaries in the realizable setting. To do so, we use \cref{alg:linear_realizable} and show that it successfully learns $(\Fcal,\Ucal)$. We recall its definition briefly. At any iteration $t\geq 1$, we define the current version space $\Fcal_t:=\set{f\in\Fcal_{\text{lin}}:\forall s<t, f(x_s)=y_s}$. We also define the regions that are labeled $0$ or $1$ with certainty, given available information as follows:
    \begin{equation*}
        S_t(y):=\set{ x\in \Xcal: \forall f\in \Fcal_t, f(x)=y},\quad y\in\{0,1\}.
    \end{equation*}
    We then use the following prediction function for time $t\geq 1$:
    \begin{equation*}
        f_t:x\in\Rbb^d \mapsto \argmin_{y\in\{0,1\}} d(x,S_t(y)).
    \end{equation*}The prediction at time $t$ is then $\hat y_t:=f_t(x_t)$.

    We now show that \cref{alg:linear_realizable} learns $(\Fcal_{\text{lin}},\Ucal)$ against realizable adaptive adversaries.
    Fix a realizable adaptive adversary $\Acal$. Since the version spaces are non-increasing, the sets $S_t(0),S_t(1)$ are non-decreasing and disjoint. 
    For any $t\geq 1$, we define $\eta(t):=d(x_t,S_t(\hat y_t)) = \min_{y\in\{0,1\}} d(x_t,S_t(y))$.
    Next, we fix $\epsilon>0$ and let $\delta_\epsilon>0$ for which \cref{eq:useful_condition} holds. We then define
    \begin{equation*}
        C_t(\epsilon):=\{x_s:s<t, \hat y_s\neq y_s, \eta(s)\geq \delta_\epsilon/2\}.
    \end{equation*}
    Consider a time $t\geq 1$ when \cref{alg:adaptive_agnostic} makes a mistake $\hat y_t\neq y_t$ while $\eta(t)\geq \delta_\epsilon/2$. By definition of \cref{alg:linear_realizable}, this implies in particular $d(x_t,C_t(\epsilon)) \geq \delta_\epsilon/2$.
    Hence, by induction we can check that $C_t(\epsilon)$ is an $\delta_\epsilon/2$-packing. In particular, denoting by $N_\epsilon^{(d)}$ the $\delta_\epsilon/2$-packing number of $B(0,M_{\epsilon})$ for any $t\geq 1$ we have $|C_t(\epsilon)\cap B(0,M_{\epsilon})|\leq N_\epsilon^{(d)}$. Hence, most mistakes will happen in the following region:
    \begin{equation*}
        B_\epsilon(t):= \set{x\in B(0,M_{\epsilon}): d(x,S_t(0)), d(x,S_t(1))\leq \delta_\epsilon/2}.
    \end{equation*}
    Formally, as a summary of the previous discussion, we can decompose the loss of \cref{alg:linear_realizable} as follows.
    \begin{equation}\label{eq:decomposition_linear}
        \sum_{t=1}^T \1[\hat y_t\neq y_t] \leq \sum_{t=1}^T \1[\|x_t\|\geq M_{\epsilon}] + N_\epsilon^{(d)} +  \sum_{t=1}^T \1[\hat y_t\neq y_t] \1[x_t\in B_\epsilon(t)] .
    \end{equation}

    In the rest of the proof, we bound the last term. For convenience, we write $\mu_{\Ucal}:=\sup_{\mu\in\Ucal}$, that is, for any measurable set $B\subseteq \Sigma$ we write $\mu_{\Ucal}(B):=\sup_{\mu\in\Ucal}\mu(B)$. In particular, $\mu_{\Ucal}$ satisfies the triangular inequality.
    Fix $t\geq 1$ 
    and note that by construction of $B_\epsilon(t)$, for any function $f_t:=\1[a_t^\top x \leq b_t]$ that realizes the observed data $\{(x_s,y_s),s<t\}$, we must have
    \begin{equation*}
        B_\epsilon(t) \subseteq B_{a_t}(b_t-\delta_\epsilon,b_t+\delta_\epsilon).
    \end{equation*}
    Fix such parameters $a_t,b_t$ depending only on the observed data $\Hcal_t:=\{(x_s,y_s),s<t\}$.
    From \cref{eq:useful_condition} there exists a hyperplane $H_t$ (dependent only on $\Hcal_t$) such that
    \begin{equation*}
         \mu_{\Ucal}(B_\epsilon(t) \setminus H_t) \leq \mu_{\Ucal}(B_{a_t}(b_t-\delta_\epsilon,b_t+\delta_\epsilon) \setminus H_t) \leq \epsilon.
    \end{equation*}
    Then, we construct an affine subspace $F_t$ of $H_t$ (possibly empty) such that $\mu_\Ucal(B_\epsilon(t) \setminus F_t) \leq 2(d+2)\epsilon$ and for any strict affine subspace $F'$ of $F_t$, we have
    \begin{equation*}
        \mu_{\Ucal}(B_\epsilon(t)\setminus F') \geq 2\epsilon + \mu_{\Ucal}(B_\epsilon(t)\setminus F_t).
    \end{equation*}
    The subspace $F_t$ can be constructed by iteratively reducing its dimension until the previous property holds (note that the property is vacuous if $F_t=\emptyset$). Note that this procedure is only dependent on $H_t$, hence without loss of generality, $F_t$ is only dependent on $\Hcal_t$. Finally, let $E_t:=Span(F_s,s\leq t)$, which is also $\Hcal_t$-measurable.
    Note that by construction, the sets $B_\epsilon(t)$ are non-decreasing. Hence, for any strict affine subspace $F'$ of $F_s$ and $s\leq t$, we have
    \begin{equation*}
        \mu_{\Ucal}(B_\epsilon(t)\cap (F_s\setminus F') )\geq \mu_{\Ucal}(B_\epsilon(s)\cap (F_s\setminus F') ) \geq 
        \mu_{\Ucal}(B_\epsilon(s)\setminus F') ) - \mu_{\Ucal}(B_\epsilon(s)\setminus F_s) ) \geq 2\epsilon.
    \end{equation*}
    On the other hand, we recall that $\mu_{\Ucal}(B_\epsilon(t)\setminus H_{t})\leq \epsilon$. This implies that $F_s\subseteq H_{t}$. Therefore, we have $E_t\subseteq H_t$ and hence $\dim(E_t)\leq d-1$.

    Furthering the bound from \cref{eq:decomposition_linear}, we have
    \begin{align}
        \AdaptReg_T(\cref{alg:linear_realizable};\Acal)=\Ebb\sqb{\sum_{t=1}^T \1[\hat y_t\neq y_t] }
        &\leq \Ebb\left[\sum_{t=1}^T \1[\|x_t\|\geq M_{\epsilon}] + N_\epsilon^{(d)} +  \sum_{t=1}^T  \1[x_t\in B_\epsilon(t)\setminus E_t] + \right. \notag\\
        &\qquad \qquad \left.\sum_{t=1}^T  \1[\hat y_t\neq y_t]\1[x_t\in B_\epsilon(t)\cap E_t] \right] \notag\\
        &\overset{(i)}{\leq} \epsilon T+N_\epsilon^{(d)} + 2(d+2)\epsilon T+ \Ebb\sqb{ \sum_{t=1}^T  \1[\hat y_t\neq y_t]\1[x_t\in E_t] }. \label{eq:recursive_bound_linear}
    \end{align}
    In $(i)$ we used the definition of $M_\epsilon$, $F_t\subseteq E_t$, and the fact that both $F_t$ and $E_t$ are $\Hcal_t$-measurable. For convenience, denote $R_T(d')$ the maximum regret of \cref{alg:linear_realizable} for a realizable adaptive adversary restricted to some fixed a priori subspace of $\Rbb^d$ of dimension $d'\in\{0,\ldots,d\}$. We clearly have $R_T(0)\leq 1$. Now recall that $(E_t)_{t\geq 1}$ is a non-decreasing sequence of affine subspaces with dimension at most $d-1$, that is, it is piece-wise constant takes at most $d$ distinct subspace values, all of dimension at most $d-1$. For any $t\geq 1$ for which $\dim(E_t)>\dim(E_{t-1})$ (by convention $E_0=\emptyset$), on the interval of time $[t,t']$ for which $E_s=E_t$ for $s\in[t,t']$, we can replicate the arguments above to bound the regret incurred by \cref{alg:linear_realizable} on the times $\{s\in[t,t']:x_s\in E_t\}$.
    This shows that
    \begin{equation*}
        R_T(d') \leq (2d+5) \epsilon T+N_\epsilon^{(d)} + (d+1)R_T(d'-1),\quad d'\in[d].
    \end{equation*}
    By induction, this shows that there exists $C_d$ depending only on $d$ and $N_\epsilon(d)$ depending only on $d,\epsilon$ and $\Ucal$, such that for any adaptive and realizable adversary,
    \begin{equation*}
        \AdaptReg_T(\cref{alg:linear_realizable};\Acal) \leq C_d\epsilon T+N_\epsilon(d).
    \end{equation*}
    This holds for any $\epsilon>0$ and adversary $\Acal$,
    which ends the proof that \cref{alg:linear_realizable} learns the distribution class $\Ucal$ for realizable adaptive adversaries. Hence the desired condition is also sufficient for learning which ends the proof. In particular, this also shows that learnability against oblivious or adaptive adversaries is equivalent for $\Fcal_{\text{lin}}$.
\end{proof}

\begin{remark}\label{remark:simpler_linear_algo_realizable}
    The proof of \cref{prop:linear_classifiers} shows that the algorithm \cref{alg:linear_realizable} is an optimistic learner for the realizable setting in both oblivious and adaptive settings, since it does not depend on the considered distribution class $\Ucal$.
\end{remark}

It remains to check that optimistic learning is also possible in the agnostic setting to complete the proof of \cref{prop:lin_classifiers_optimistic}.

\vspace{3mm}

\begin{proof}[of \cref{prop:lin_classifiers_optimistic}]
    We start by constructing an algorithm for learning $\Fcal_{\text{lin}}$ with adaptive adversaries, which will perform the Hedge algorithm over a set of experts to be defined.
    For any subset of times $S\subseteq [T]$ let $E(S)$ denote the expert which follows the predictions of \cref{alg:linear_realizable} by using labels corresponding to that expert making mistakes at times in $S$ (in the same spirit as \cref{alg:adaptive_agnostic}). The corresponding expert is summarized in \cref{alg:linear_expert}.

\begin{algorithm}[t]

    \caption{Expert $E(S)$ for $\Ocal(\epsilon)$ average adaptive regret in the agnostic setting}\label{alg:linear_expert}
    
    \LinesNumbered
    \everypar={\nl}
    
    \hrule height\algoheightrule\kern3pt\relax
    \KwIn{Mistake times $S$}
    
    \vspace{3mm}

    \For{$t\geq 1$}{
        Let $\hat y_t(S)$ be the prediction of \cref{alg:linear_realizable} for the history $(x_s,\hat y_s(S)\oplus \1[s\in S])_{s<t}$ and new instance $x_t$. Predict $\hat y_t(S)$
    }
    
    \hrule height\algoheightrule\kern3pt\relax
    \end{algorithm}

    We consider the algorithm $\Lcal(\epsilon)$ which simply performs the Hedge algorithm on the set of experts $\Scal_E(\epsilon):=\{E(S): |S|\leq \epsilon T\}$ as detailed in \cref{alg:oblivious_agnostic}. Intuitively, $\Lcal(\epsilon)$ achieves $\epsilon$ average regret asymptotically. However, the time $t_\epsilon$ starting from which $\Lcal(\epsilon)$ has $\Ocal(\epsilon)$ average regret depends on $\Ucal$ which cannot be used by an optimistic learner. 
    To solve this issue we run a Hedge-variant algorithm on the experts $\{\Lcal(2^{-i}),i\geq 1\}.$ Note that because there is an infinite number of considered experts we cannot directly use the vanilla Hedge algorithm. Instead, we use the restarting strategy from \cite{hanneke2022universally}. Namely, let $t_k:=k(k+1)/2$ for $k\geq 1$. Then, we consider the final algorithm $\Lcal^\star$ which on each epoch $[t_k,t_{k+1})$ performs the classical Hedge algorithm detailed \cref{alg:oblivious_agnostic} with the horizon $t_{k+1}-t_k=k$ on the restricted set of experts $\{\Lcal(2^{-i}),i\in[k]\}$.

    \paragraph{Regret Analysis.} We denote by $\Acal$ the considered adaptive adversary. We now bound the regret of $\Lcal^\star$. We denote by $\hat y_t^\star$ the predictions of $\Lcal^\star$. Similarly, we denote by $\hat y_t^{(i)}$ the predictions of $\Lcal(2^{-i})$. Fix $T\geq 1$ and let $k_{\max}\geq 1$ such that $T\in[t_{k_{\max}},t_{k_{\max}+1})$. We apply \cref{thm:regret_hedge} to each run of the Hedge algorithm for all epochs $[t_k,t_{k+1})$ with $k\in[k_{\max}]$. This shows that for any $i\geq 1$, with probability at least $1-1/T$,
    \begin{align*}
        \sum_{t=1}^T \1[\hat y_t^\star \neq y_t] &\leq t_i + \sum_{k=i}^{k_{\max}} \1[\hat y_t^\star \neq y_t]\\
        &\leq t_i + \sum_{t=t_i}^T \1[\hat y_t^{(i)}\neq y_t] + \sum_{k=i}^{k_{\max}}\paren{\sqrt{\frac{k\log k}{2}} + \sqrt{k\log (Tk_{\max})}},
    \end{align*}
    where noted that $t_{k+1}-t_k=k$ for $k\in[k_{\max}]$.
    Next, note that $k_{\max} \lesssim \sqrt T$ and that $\sum_{k=1}^{k_{\max}}\sqrt{k}\leq k_{\max}^{3/2}$. Hence, there is a constant $c_1>0$ such that with probability at least $1-1/T$,
    \begin{equation*}
        \sum_{t=1}^T \1[\hat y_t^\star \neq y_t] \leq \sum_{t=1}^T \1[\hat y_t^{(i)}\neq y_t] + t_i + c_1 T^{3/4} \log^{1/2} T.
    \end{equation*}
    We denote by $\Ecal_i$ this event. Recall that $\Lcal(2^{-i})$ runs the Hedge algorithm over the set of experts $\Scal_E(2^{-i})$. We denote by $\hat y_t(S)$ the predictions of expert $E(S)$ for $S\subseteq [T]$. Then, \cref{thm:regret_hedge} implies that there is a constant $c_2>0$ such that with probability at least $1-1/T$,
    \begin{equation*}
        \sum_{t=1}^T \1[\hat y_t^{(i)}\neq y_t] \leq \min_{S\subseteq [T],|S|\leq 2^{-i}T} \sum_{t=1}^T \1[\hat y_t(S)\neq y_t] + c_2\sqrt{T\log (|\Scal_E(2^{-i})| T)}.
    \end{equation*}
    We denote by $\Fcal_i$ the above event.
    Note that $|\Scal_E(2^{-i})|\leq (\epsilon_i T+1)\cdot \binom{T}{\floor{\epsilon_i T}}\leq (\epsilon_i T+1)\paren{\frac{e}{\epsilon_i}}^{\epsilon_i T}$, where $\epsilon_i=2^{-i}$.
    Then, on $\Ecal_i\cap\Fcal_i$ we obtained
    \begin{equation*}
        \sum_{t=1}^T \1[\hat y_t^\star \neq y_t] \leq \min_{S\subseteq [T],|S|\leq 2^{-i}T} \sum_{t=1}^T \1[\hat y_t(S)\neq y_t]  + t_i + c_3 \paren{\sqrt{i2^{-i}} \cdot T + T^{3/4 } \log^{1/2} T + \sqrt{iT}},
    \end{equation*}
    for some universal constant $c_3>0$.

    As an important remark, note that the proof for the regret analysis of \cref{alg:linear_realizable} within \cref{prop:linear_classifiers} would also work for the following stronger form of adversaries, which we call prescient adversaries. The adversary first constructs a sequence $(x_t)_{t\in[T]}$ as follows: having constructed $x_1,\ldots,x_{t-1}$ they select $\mu_t\in\Ucal$ adaptively, then sample $x_t\sim \mu_t$. With the knowledge of $x_1,\ldots,x_T$, the adversary then chooses a function $f^\star\in\Fcal$. The value at time $t$ is then $y_t:=f^\star(x_t)$. The main difference with realizable adaptive adversaries is that the values are allowed to depend on future samples $x_t$.
    The reason why the proof of regret for \cref{alg:linear_realizable} holds also for prescient adversaries is that the bounds in the expectation only bound events on $x_t$ (see \cref{eq:recursive_bound_linear}), namely the events $\1[\|x_t\|\geq M_\epsilon]$ and $\1[x_t\in B_\epsilon(t)\setminus E_t]$ which both have probability bounded conditionally on the history $\Hcal_t$.

    With this remark at hand, let $f^\star\in\Fcal_{\text{lin}}$ be a function such that
    \begin{equation*}
        \sum_{t=1}^T \1[f^\star(x_t)\neq y_t] = \min_{f\in\Fcal} \sum_{t=1}^T \1[f(x_t)\neq y_t].
    \end{equation*}
    We consider the prescient adversary $\widetilde\Acal$ that chooses $f^\star$ to label the values $y_t^\star:=f^\star(x_t)$. Let $S^\star$ be the set of times in $[T]$ on which \cref{alg:linear_realizable} makes mistakes for this prescient adversary. Then, the proof of \cref{prop:linear_classifiers} shows that there exists $C_d\geq 1$ depending only on $d$ and $N_i(d)$ depending only on $d$ and $\Ucal$, such that  for any $j\geq 1$,
    \begin{equation*}
        \Ebb[|S^\star|]=\Ebb\sqb{\sum_{t=1}^T\1[\hat y_t(S)\neq y_t^\star]} = \AdaptReg_T(\cref{alg:linear_realizable};\widetilde\Acal) \leq C_d 2^{-j} T + N_j(d).
    \end{equation*}
    Hence, by Markov's inequlity, the event
    \begin{equation*}
        \Gcal_j:= \set{ |S^\star| \leq C_d 2^{-j/2} T + N_j(d)2^{j/2}}
    \end{equation*}
    has probability at least $1-2^{-j/2}$.

    Now fix $i\geq 1$ and let $j\geq i$ be sufficiently large so that $C_d 2^{-j/2}\leq 2^{-i-1}$. Then, for 
    $T\geq N_j(d)2^{j/2} / (C_d 2^{-j/2}):=T_j$, on $\Gcal_j$ we have $|S^\star| \leq 2C_d 2^{-j/2} T \leq 2^{-i}T$. Therefore, on $\Ecal_i\cap\Fcal_i\cap\Gcal_j$ we obtained
    \begin{align*}
        \sum_{t=1}^T \1[\hat y_t^\star \neq y_t] 
        &\leq \sum_{t=1}^T\1[\hat y_t(S^\star) \neq y_t] + t_i + c_3\paren{\sqrt{i2^{-i} }\cdot T + T^{3/4  } \log^{1/2} T +\sqrt{iT}}\\
        &\leq \sum_{t=1}^T\1[f^\star(x_t) \neq y_t] + \sum_{t=1}^T\1[\hat y_t(S^\star)\neq y_t^\star] + t_i + c_3\paren{\sqrt{i2^{-i} }\cdot T + T^{3/4  } \log^{1/2} T +\sqrt{iT}}.
    \end{align*}
    Note that $\Ecal_i\cap\Fcal_i\cap\Gcal_j$ has probability at least $1-2^{-j/2}-2/T$. Taking the expectation over the previous equation together with the definition of $f^\star$ gives for any $T\geq T_j$,
    \begin{align*}
        \AdaptReg_T(\Lcal^\star;\Acal) &\leq ( C_d 2^{-j} + 2^{-j/2} +c_3 \sqrt{i2^{-i}}) T + c_3 T^{3/4 } \log^{1/2} T +c_3\sqrt{iT}+N_j(d)+t_i \\
        &\leq ( 2^{-i+1} +c_3 \sqrt{i2^{-i}}) T + c_3 T^{3/4 } \log^{1/2} T +c_3\sqrt{iT}+t_i.
    \end{align*}
    The above bound holds for any $i\geq 1$ for $T\geq T_j$ and for any adversary $\Acal$, which ends the proof that $\Lcal^\star$ learns $\Ucal$. Hence $\Lcal^\star$ optimistically learns $\Fcal_{\text{lin}}$ in the agnostic setting for adaptive adversaries. This also holds for all considered adaptive/oblivious realizable/agnostic settings since learnable distribution classes coincide for $\Fcal_{\text{lin}}$ as proved in \cref{prop:linear_classifiers}.
\end{proof}

\end{document}